\documentclass[10pt,twocolumn,letterpaper]{article}

\usepackage[pagenumbers]{cvpr} 

\usepackage{amsmath,amsfonts,bm}

\def\eqref#1{equation~\ref{#1}}

\def\1{\bm{1}}

\def\va{{\bm{a}}}
\def\vb{{\bm{b}}}

\def\mA{{\bm{A}}}
\def\mB{{\bm{B}}}

\def\mS{{\bm{S}}}

\def\mU{{\bm{U}}}
\def\mV{{\bm{V}}}
\def\mW{{\bm{W}}}
\def\mX{{\bm{X}}}

\def\mSigma{{\bm{\Sigma}}}

\DeclareMathAlphabet{\mathsfit}{\encodingdefault}{\sfdefault}{m}{sl}
\SetMathAlphabet{\mathsfit}{bold}{\encodingdefault}{\sfdefault}{bx}{n}

\usepackage{url}

\usepackage{booktabs}
\usepackage{multirow}
\usepackage{graphicx}
\usepackage{colortbl}
\usepackage{xcolor}

\usepackage{amsthm}
\usepackage{thm-restate}

\usepackage{adjustbox}
\usepackage{wrapfig}

\usepackage{xspace}
\newcommand{\method}{\textsc{TARA}\xspace}

\usepackage{enumitem}

\usepackage{subcaption}

\newlength{\tightvspace}
\definecolor{cvprblue}{rgb}{0.21,0.49,0.74}
\usepackage[pagebackref,breaklinks,colorlinks,allcolors=cvprblue]{hyperref}

\def\paperID{} 
\def\confName{CVPR}
\def\confYear{2026}

\title{Preference-Aligned LoRA Merging: Preserving Subspace Coverage and Addressing Directional Anisotropy}

\author{
    Wooseong Jeong$^*$ \\
    KAIST \\
    {\tt\small stk14570@kaist.ac.kr}
    \and
    Wonyoung Lee$^*$ \\
    KAIST \\
    {\tt\small wylee@kaist.ac.kr}
    \and
    Kuk-Jin Yoon \\
    KAIST \\
    {\tt\small kjyoon@kaist.ac.kr}
}

\begin{document}
\maketitle

\renewcommand{\thefootnote}{\fnsymbol{footnote}}
\footnotetext[1]{Equal contribution to this work.}

\begin{abstract}
	Merging multiple Low-Rank Adaptation (LoRA) modules  is promising for constructing general-purpose systems, 
    yet challenging because LoRA update directions span different subspaces and contribute unevenly.
    When merged naively, such mismatches can weaken the directions most critical to certain task losses while overemphasizing relatively less important ones, ultimately reducing the model’s ability to represent all tasks faithfully. We revisit this problem through two perspectives: subspace coverage, which captures how broadly LoRA directions cover diverse representational directions, and anisotropy, which reflects the imbalance of influence across those directions. 
    We propose TARA-Merging (\emph{Task-Rank Anisotropy Alignment}), which aligns merging weights using a preference-weighted cross-entropy pseudo-loss while preserving task-relevant LoRA subspaces. This ensures broad subspace coverage and mitigates anisotropy via direction-wise reweighting.
    Across eight vision and six NLI benchmarks, TARA-Merging consistently outperforms vanilla and LoRA-aware baselines, demonstrating strong robustness and generalization, and highlighting the importance of addressing both subspace coverage and anisotropy in LoRA merging. 
Code is available at \url{https://github.com/wooseong97/TARA-Merge}.
\end{abstract}

\section{Introduction}

Low-Rank Adaptation (LoRA)~\cite{hu2022lora} is widely used to adapt large foundation models to downstream tasks, including large language models and vision-language models such as CLIP~\cite{radford2021learning}. By introducing lightweight low-rank adapters, LoRA enables efficient fine-tuning with far fewer trainable parameters and reduced memory overhead, while mitigating the risk of overfitting when only limited data are available. This efficiency has encouraged training separate LoRA adapters on many different datasets, each capturing distinct aspects of vision or language understanding. Combining task-specific LoRA modules into a single network falls under the paradigm of \emph{model merging}~\cite{yang2024model}, which provides an attractive alternative to costly multi-task training~\cite{jeong2024quantifying,jeong2025selective,jeong2025resolving,jeong2025synchronizing} and enables the construction of general-purpose models that flexibly integrate knowledge across tasks and reflect different task preferences.

Recent research has explored various model merging. Early baselines such as task arithmetic~\cite{ilharco2022editing, ortiz2024task, jin2024fine} treat fine-tuned models as linear updates and combine them through a linear combination. Ties-merging~\cite{yadav2023ties} and DARE~\cite{yu2024language} reduces interference by dropping parameters, while AdaMerging~\cite{yang2023adamerging} adapts merging coefficients using gradients derived from an entropy-based surrogate loss. Recent work has targeted model merging specifically for LoRA. KnOTS~\cite{stoica2025knots} aligns adapters using a shared SVD basis, and LoRA-LEGO~\cite{zhao2025loralego} clusters rank-wise units to preserve modularity.
Focusing on LoRA adapters is promising, as LoRA fine-tuning has become standard practice for modern LLMs and VLMs.
Moreover, operating at the adapter level substantially reduces memory and compute overhead compared to merging full model weights, enabling even gradient-based approaches such as AdaMerging to scale to foundation models.
In addition, While several approaches have been proposed, they often overlook at least one of the two fundamental challenges we emphasize in this work, namely preserving subspace coverage and addressing directional anisotropy.

LoRA produces a low-rank update that we decompose into rank-wise components for analysis. We collectively term these directions as \emph{LoRA directions}. With this setup, two aspects are essential for successful merging: \emph{subspace coverage} and \emph{anisotropy}. Subspace coverage captures how widely the LoRA directions collectively span the representational subspace, so preserving coverage ensures access to diverse task-relevant information. Anisotropy refers to directional sensitivity imbalance. Even with sufficient coverage, task losses can be unequally sensitive to different LoRA directions, which can distort trade-offs if ignored. Furthermore, when task preferences are specified, effective merging must preserve coverage while aligning merging weights with those preferences. Most merging methods address only one side of the problem. Task arithmetic and pruning reduce interference but erode subspace coverage and ignore direction-wise sensitivity. AdaMerging tunes global weights yet misses per-direction anisotropy. LoRA-aware methods like KnOTS and LoRA-LEGO improve structure but can shrink effective coverage and still lack sensitivity weighting. In short, they do not jointly preserve coverage and model anisotropy over LoRA directions.

Building on these observations, we propose \method, short for \emph{Task-Rank Anisotropy Alignment}. Our method introduces a structured merging framework that explicitly incorporates task preferences while addressing the two properties above. It preserves LoRA directions to maintain subspace coverage and accounts for anisotropy by direction-wise reweighting according to sensitivity. It then aligns the merging weights with a preference-weighted entropy pseudo loss, yielding merged models that faithfully approximate fine-tuned models across tasks while improving robustness and generalization. We validate our approach across eight vision datasets and six natural language inference tasks, where \method consistently outperforms vanilla baselines and recent LoRA-aware methods in joint-task evaluation and transfer to unseen tasks.

Our main contributions are summarized as follows:
\begin{itemize}[leftmargin=*, itemsep=2pt, topsep=2pt]
	\item Identification of \textbf{subspace coverage} and \textbf{anisotropy} as two key properties for analyzing and guiding LoRA merging, highlighting their roles in subspace preservation and direction-wise sensitivity.
	\item Introduction of \textbf{\method}, a framework that aligns both task-level and rank-level preferences, ensuring coverage of LoRA subspaces while mitigating anisotropy-induced imbalance.
	\item Extensive empirical validation on vision and language benchmarks, demonstrating state-of-the-art performance and strong robustness in various experimental settings.
\end{itemize}

\section{Related work}
\textbf{Model Merging.} Model merging integrates knowledge from separately fine-tuned models without joint multi-task training~\cite{}, reducing computational cost while preserving task performance \cite{yang2024model, gargiulo2025task, marczak2025no}.
Core baselines include \emph{Task Arithmetic}, which adds or subtracts task vectors \cite{ilharco2022editing}, \emph{RegMean}, which requires inner-product matrices of layer inputs for data-free fusion \cite{jin2022dataless}, \emph{TIES}, which resolves sign conflicts and prunes small deltas \cite{yadav2023ties}, and \emph{DARE}, which sparsifies update deltas \cite{yu2024language}.
\emph{AdaMerging} learns layer-wise merge coefficients via entropy minimization to improve multi-task performance \cite{yang2023adamerging}. For models trained on different data or from different initializations, permutation alignment preserves linear mode connectivity \cite{entezari2021role,ainsworth2022git}.
\emph{ZipIt} matches intermediate features to re-basin networks and align neurons across models, enabling training-free layer alignment \cite{stoica2023zipit}.
Uncertainty-aware combinations mitigate mismatch through Fisher-weighted averaging \cite{matena2022merging} and uncertainty or gradient matching \cite{daheim2023model}.
Preference-aware settings motivate multi-objective formulations such as Pareto merging \cite{chen2024pareto} and smooth scalarization \cite{lin2024smooth}.
Overall, recent work advances simple averaging into alignment- and uncertainty-aware schemes with explicit trade-off control.

\vspace{3pt}
\noindent \textbf{LoRA-aware Merging.}
For large networks including foundation models such as LLMs~\cite{dubey2024llama}, Low-Rank Adaptation (LoRA)~\cite{hu2022lora} is a widely used fine-tuning scheme. Conventional parameter-space merging methods~\cite{jin2022dataless,ilharco2022editing,yadav2023ties,yu2024language} transfer poorly to LoRA adapters, motivating LoRA-aware approaches~\cite{stoica2025knots,zhao2025loralego,tang2025lora}. In particular, \citet{stoica2025knots} argue that LoRA-updated models exhibit weaker cross-model representation alignment than full-rank finetunes, and propose \textsc{KnOTS}, which concatenates adapter updates and applies an SVD to align them in a shared subspace before merging principal components. \citet{zhao2025loralego} introduce \emph{minimal semantic units} (MSUs) and perform clustering to assemble a merged adapter with an adjustable effective rank. Related mixtures of LoRAs have also been explored for image generation, including multi-LoRA composition and concept mixing~\cite{zhong2024multi,gandikota2024concept,zhuang2025timestep}. Additional related work and connections to our approach are discussed in \Cref{sec:additional-related-work}.

General model merging ignore the factorized structure of LoRA and often suffer from strong misalignment when adapters are folded into the base model before merging. In contrast, LoRA-aware methods operate on the low-rank factors, which is both efficient and more effective for combining updates. In this work we explicitly account for \emph{subspace coverage} and \emph{anisotropy}, two critical aspects of LoRA merging that prior methods largely overlook.

\section{Subspace Coverage and Anisotropy}
\subsection{Problem Setting and Notation}
\label{subsec:notation}
Modern foundation models in language and vision are increasingly adapted with LoRA rather than full fine-tuning, which motivates principled ways to combine many LoRA adapters into a single model. In practice, directly merging adapters by folding them into the base weights and applying model-merging often underperforms due to strong cross-task misalignment~\cite{stoica2025knots}. These factors call for a merging approaches that is aware of the LoRA structure. We therefore study \emph{LoRA merging}, where multiple LoRA fine-tuned adapters are combined into one model that balances task performance according to a user-specified preference.

\vspace{3pt}
\noindent\textbf{Base model and LoRA adapters.}
Let $\bm W_{0}\in\mathbb{R}^{d\times m}$ be the frozen pretrained weight matrix.
For each task $i \in \{1,\dots,N\}$, we attach a LoRA adapter of rank $r_i$, parameterized as
\begin{equation}
	\Delta\bm W_{i} \;=\; \bm B_{i}\bm A_{i}^{\!\top},
	\qquad
	\bm B_{i}\in\mathbb{R}^{d\times r_i},\;
	\bm A_{i}\in\mathbb{R}^{m\times r_i}.
\end{equation}
Thus, $\operatorname{rank}(\Delta\bm W_{i}) \le r_i \ll \min(d,m)$.
We further decompose each as
$\Delta\bm W_{i} = \sum_{j=1}^{r_i}\vb_{ij}\va_{ij}^{\!\top}$,
where $\vb_{ij}$ and $\va_{ij}$ denote the $j$-th columns of $\bm B_{i}$ and $\bm A_{i}$.
Each outer product $\vb_{ij} \va_{ij}^{\!\top}$ is referred to as a \emph{rank-1 LoRA direction}.

\vspace{3pt}
\noindent\textbf{User preferences.}
We assume that the relative importance of tasks is specified by a vector
$\bm\rho=[\rho_1,\dots,\rho_N]^{\top}\in\Delta_{N-1}$,
where $\Delta_{N-1}$ denotes the probability simplex.
A larger $\rho_i$ indicates higher priority for task $i$.

\vspace{3pt}
\noindent\textbf{Merging as Multi-Objective Optimization.}
Without fixing preferences, LoRA merging can be cast as a multi-objective optimization (MOO) problem \cite{miettinen1999nonlinear} that seeks a set of solutions. We minimize the objective vector $\bm f(\bm W)=[f_1(\bm W),\dots,f_N(\bm W)]^{\top}$ over merged models $\bm W$ obtained by linearly combining available LoRA adapters on top of a fixed base $\bm W_0$. A solution $\bm W^a$ \emph{dominates} $\bm W^b$ if $f_k(\bm W^a)\le f_k(\bm W^b)$ for all $k\in\{1,\dots,N\}$ and $f_j(\bm W^a)<f_j(\bm W^b)$ for some $j$. A solution is \emph{Pareto optimal} if no other feasible solution dominates it. The \emph{Pareto front} is the collection of Pareto-optimal solutions.

\subsection{Preserving LoRA Subspace Coverage}
\begin{figure}[t]
	\centering
	\includegraphics[width=0.85\linewidth]{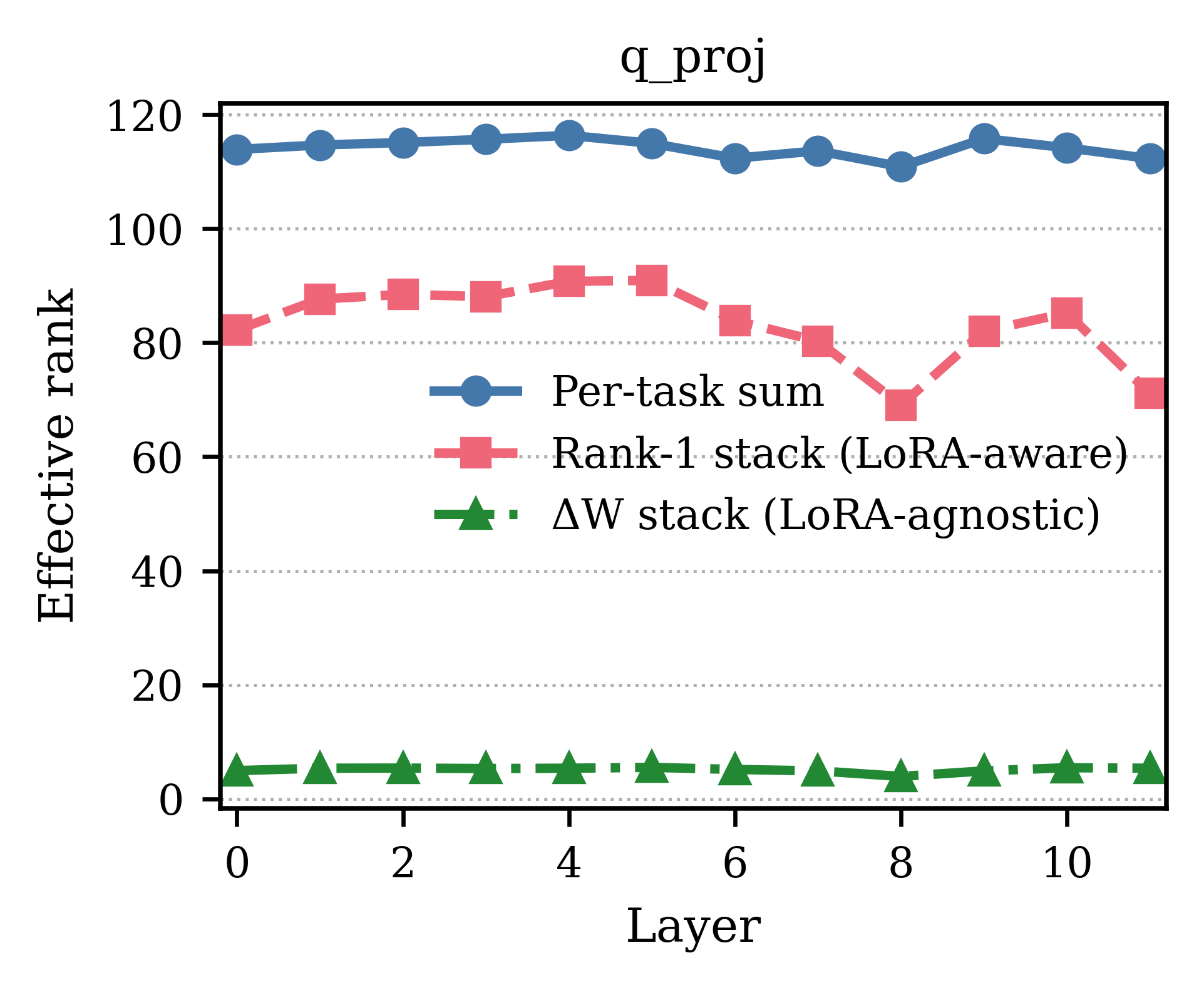}
	\caption{\textbf{Effective rank} across layers for an attention projection layer (e.g., query).
		The gap between \emph{$\Delta \bm X$ stack} and \emph{Rank-1 stack} reflects merge-induced collapse.}
	\label{fig:lines-qproj-raw}
\end{figure}

We quantify how much task-specific representational capacity is retained when LoRA adapters are merged.
We refer to this as \emph{subspace coverage}: it measures how broadly the task-specific LoRA directions span the parameter space and how well a merge preserves that span.
We measure subspace coverage using the entropy-based \emph{effective rank} (erank)~\cite{roy2007effective}.
Let $\bm X\in\mathbb{R}^{d\times m}$ have singular values $\{\sigma_k\}_{k=1}^{r_X}$ where $r_X=\mathrm{rank}(\bm X)$, and define $p_k=\sigma_k^2\big/\sum_{j=1}^{r_X}\sigma_j^2$.
Then the effective rank is
\begin{equation}
	\mathrm{erank}(\bm X)=\exp\!\Big(-\sum_{k=1}^{r_X} p_k \log p_k\Big).
\end{equation}
$\mathrm{erank}(\bm X)$ equals $r_X$ for a flat spectrum and decreases as energy concentrates, providing a basis-invariant measure of effective dimensionality.

We compare three quantities using effective rank, with all stacks formed by vectorizing matrices and placing one vector per row. First, \emph{per-task sum} computes \(\mathrm{erank}\) for each task \(i\) on \(\mX_i=[\operatorname{vec}(\vb_{i1}\va_{i1}^{\top}),\dots,\operatorname{vec}(\vb_{ir_i}\va_{ir_i}^{\top})]^{\top}\) and then sums \(\sum_{i=1}^{N}\mathrm{erank}(\mX_i)\). Second, the $\Delta \bm X$ stack (\emph{LoRA-agnostic}) uses \(\mX_{\text{agnostic}}=[\operatorname{vec}(\Delta \mW_1),\dots,\operatorname{vec}(\Delta \mW_N)]^{\top}\) and computes a single \(\mathrm{erank}(\mX_{\text{agnostic}})\). Third, the Rank-1 stack (\emph{LoRA-aware}) aggregates all rank-1 factors across tasks into \(\mX_{\text{aware}}=[\operatorname{vec}(\vb_{11}\va_{11}^{\top}),\dots,\operatorname{vec}(\vb_{Nr_N}\va_{Nr_N}^{\top})]^{\top}\) and computes \(\mathrm{erank}(\mX_{\text{aware}})\). Here we refer to the Rank-1 stack as \emph{LoRA-aware} because it preserves the adapter factorization and treats each rank-1 direction \(\vb_{ij}\va_{ij}^{\top}\) as a separate basis element, whereas the \(\Delta\bm W\) stack is \emph{LoRA-agnostic} since it collapses each adapter into a single update \(\Delta\bm W_i\) and ignores the internal rank-1 structure.

As shown in \Cref{fig:lines-qproj-raw}, the \emph{LoRA-aware} stack retains about \(70\%\) of the \emph{per-task sum}, indicating that most task-specific directions remain approximately independent after redundancy removal and thus cross-task alignment is weak. This agrees with prior work showing that LoRA fine-tuned models align less across tasks than full-rank fine-tuning~\cite{stoica2025knots}. The gap between the \emph{LoRA-aware} and \emph{LoRA-agnostic} stacks captures subspace collapse under interpolation-based merging~\cite{ilharco2022editing,yadav2023ties,yang2023adamerging}, because unlike \(\bm X_{\text{aware}}\) the interpolated updates in \(\bm X_{\text{agnostic}}\) can interfere destructively across tasks and lower the effective rank. Formal definitions and additional results appear in \Cref{append:subspace_coverage} of supplementary material.

\subsection{Anisotropy of LoRA Directions}
Let $\bm W$ be the current weights and $\{\,\bm S_k\,\}_{k=1}^{K}$ a set of \emph{LoRA directions}, where a LoRA adapter uses factor matrices $\bm B\in\mathbb{R}^{d\times r}$ and $\bm A\in\mathbb{R}^{m\times r}$ (for simplicity, the same rank $r$ for all adapters), with columns $\vb_k$ and $\va_k$, and we set $\bm S_k=\vb_k \va_k^{\top}$ (so $K=r$ for a single adapter, or the sum of ranks if multiple adapters are stacked). Consider a small update restricted to this span with \emph{direction-selection coefficients} $\bm{\phi}=(\phi_1,\dots,\phi_K)^\top\in\mathbb{R}^{K}$ (distinct from the task-preference vector $\bm\rho$), so that $\Delta \bm W=\sum_{k=1}^{K}\phi_k \bm S_k$. Using the Frobenius inner product, the first-order change of each task loss is
\begin{equation}
	\Delta f_i \;\approx\; \big\langle \nabla f_i(\bm W),\,\Delta \bm W\big\rangle_F \;=\; \sum_{k=1}^{K} \phi_k\,\big\langle \nabla f_i(\bm W),\,\bm S_k\big\rangle_F .
\end{equation}
Stacking $\Delta \bm f=[\Delta f_1,\dots,\Delta f_N]^\top$ yields the linearization
\begin{equation}
	\Delta \bm f \;\approx\; \bm J\,\bm{\phi},
	\qquad
	\bm J_{i,k} \;=\; \big\langle \nabla f_i(\bm W),\,\bm S_k\big\rangle_F,
\end{equation}
so $\bm J(\bm W)$ is the task-loss Jacobian restricted to $\operatorname{span}\{\bm S_k\}$. Directional imbalance is captured by the singular values of $\bm J$, which give the following bound from coefficients to loss changes.

\begin{restatable}[Anisotropy Bounds]{proposition}{propone}\label{prop:anisotropy-bounds}
	Let $\sigma_{\max}(\bm J)$ and $\sigma_{\min}(\bm J)$ denote the largest and smallest singular values of $\bm J$. Then, for any coefficient vector $\bm\phi$,
	\begin{equation}
		\sigma_{\min}(\bm J)\,\|\bm{\phi}\|_2
		\;\le\;
		\|\bm J\bm{\phi}\|_2
		\;=\;
		\|\Delta\bm f\|_2
		\;\le\;
		\sigma_{\max}(\bm J)\,\|\bm{\phi}\|_2.
	\end{equation}
	When \(\kappa(\bm J)=\sigma_{\max}(\bm J)/\sigma_{\min}(\bm J)\) is large, the map \(\bm\phi\!\mapsto\!\Delta \bm f\) is \emph{anisotropic}, and equal-norm LoRA-direction updates need not yield proportional task-loss changes.
\end{restatable}

\vspace{2pt}
\noindent\textbf{Empirical evidence of anisotropy.}
To illustrate the qualitative effect anticipated by Proposition~\ref{prop:anisotropy-bounds}, we measure singular-value spectra and condition numbers of $\bm J$ across layers and modules. We observe a pronounced spectral spread and large condition numbers, indicating that a few directions dominate the task-loss response while many are weak. Full plots are deferred to \cref{append:anisotropy} of supplementary material. See the condition-number traces in \cref{fig:kappa-lines-raw,fig:kappa-lines-svd} and the scree curves in \cref{fig:grid-scree-raw-vs-svd} of supplementary material.

\vspace{2pt}
\noindent\textbf{Directional-sensitivity misalignment across preferences.}
\label{para:dsm}
With the scalarized gradient $g(\bm\rho;\bm W)
	=\sum_{i}\rho_i\,\nabla f_i(\bm W)$, define direction-wise \emph{loss sensitivities}
\begin{align}
	h_k(\bm\rho;\bm W)\;=\;\big\langle g(\bm\rho;\bm W),\,\bm S_k\big\rangle_F,
	\\
	\bm h(\bm\rho;\bm W)=\big(h_1,\dots,h_K\big)^{\!\top}.
\end{align}
When the preference $\bm\rho$ changes, the pattern of $\bm h(\bm\rho;\bm W)$ also changes as well. We quantify \emph{Directional Sensitivity Misalignment} by

\begin{equation}
	\xi(\bm\rho_1,\bm\rho_2;\bm W)\;=\;1-\frac{\left|\langle \bm h(\bm\rho_1;\bm W),\,\bm h(\bm\rho_2;\bm W)\rangle\right|}
	{\|\bm h(\bm\rho_1;\bm W)\|_2\,\|\bm h(\bm\rho_2;\bm W)\|_2}\;,
\end{equation}
so that $\xi\in[0,1]$. Here $\langle \cdot,\cdot\rangle$ denotes the dot product.
Smaller $\xi$ indicates weaker $\bm\rho$-dependence of the sensitivity profiles, whereas larger $\xi$ indicates stronger $\bm\rho$-dependence.
When measuring anisotropy, a \(180^\circ\) flip lies in the same one-dimensional span. We therefore treat sign flips as equivalent and focus on alignment along the direction axis.

\begin{figure}
	\centering
	\includegraphics[width=0.85\linewidth]{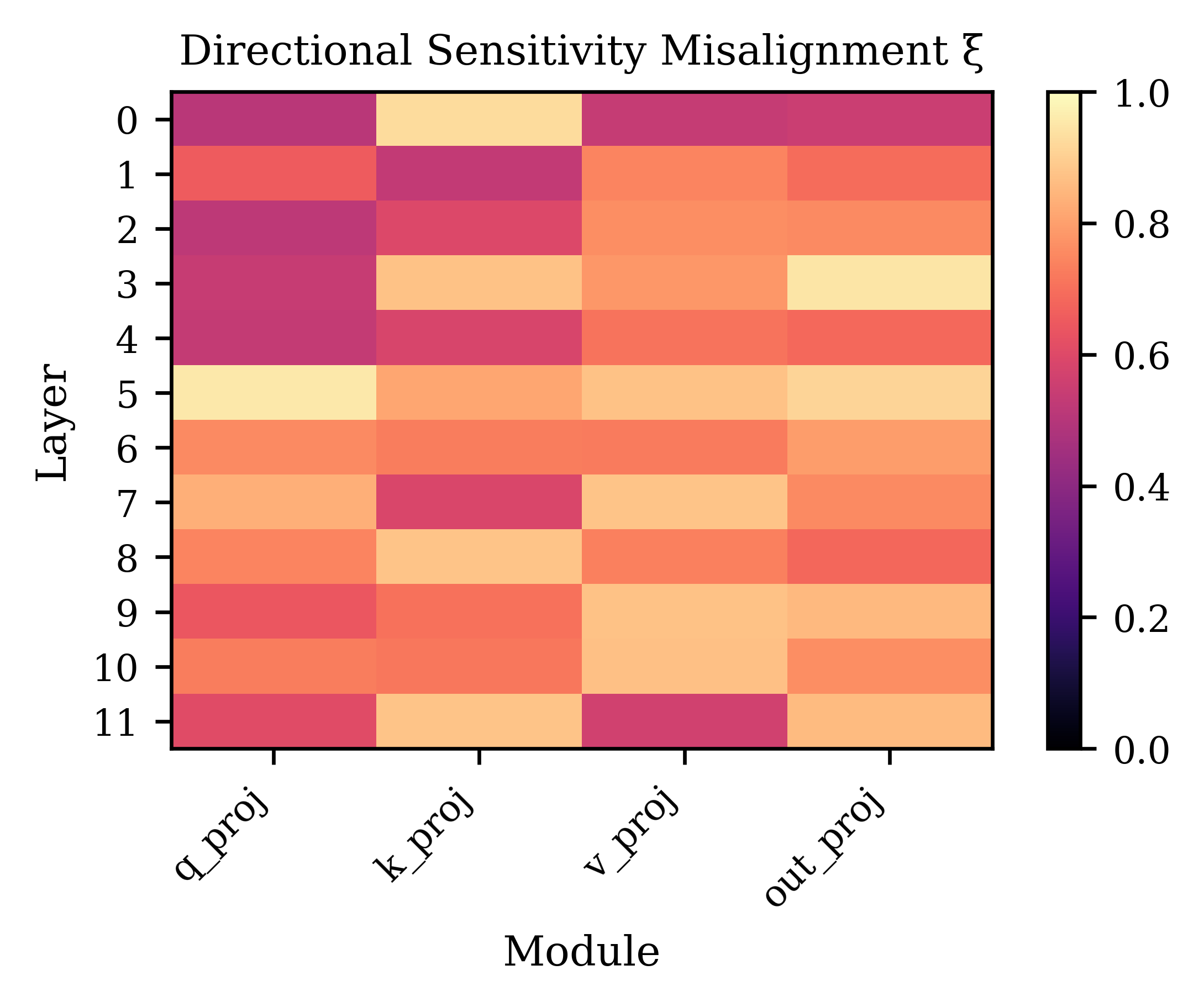}
	\caption{\textbf{Directional-sensitivity misalignment \(\boldsymbol{\xi(\rho_1,\rho_2)}\).}
		Larger $\xi$ indicates stronger change in loss-sensitive directions when switching preferences.
	}
	\label{fig:xi-heatmap-svd}
\end{figure}

We quantify how loss-sensitive directions within the LoRA span change when the task preference switches from $\bm\rho_1$ (uniform over tasks) to $\bm\rho_2$ (one-hot on a single task).
For each preference, we form a sensitivity profile over LoRA directions by projecting the scalarized gradient $g(\bm\rho;\bm W)=\sum_i \rho_i\,\nabla f_i(\mathbf X)$ onto the basis $\{\mS_k\}$.
Throughout this analysis, we set $\bm W$ to the Task Arithmetic merge $\mW_{\text{merge}}=\mW_{0}+\lambda\sum_{i=1}^{N}\Delta \mW_{i}$ with $\lambda=0.3$ following \cite{ilharco2022editing}.
We then summarize the distance between the two preference with a misalignment index $\xi(\bm\rho_1,\bm\rho_2)\in[0,1]$.
Figure~\ref{fig:xi-heatmap-svd} shows substantial misalignment across layers and modules, indicating that directional sensitivities within the LoRA span are highly preference-dependent and motivating our subsequent directional alignment.

\section{Method}
Based on the analysis, we propose \method, \emph{Task-Rank Anisotropy Alignment}, which addresses both \emph{subspace coverage} and \emph{anisotropy} in LoRA merging.
\method weights LoRA directions with an entropy-minimization term, as in AdaMerging~\cite{yang2023adamerging}.
We present two variants that specify how the directions are constructed and how the direction-selection weights are assigned.

\begin{table*}[t]
	\vspace{-10pt}
	\caption{
		Per-task accuracy on eight image-classification benchmarks.
		We merge eight ViT-B/32 checkpoints, each fine-tuned with LoRA.
		The upper panel shows the per-task absolute accuracy of the fine-tuned baselines.
		The lower panel shows the accuracy of the merged models, normalized by their corresponding fine-tuned baseline (\%).
	}
	\vspace{-5pt}
	\label{tab:8vis_knots}
	\centering
	\renewcommand{\arraystretch}{0.75}
	\resizebox{0.80\linewidth}{!}{
		\begin{tabular}{lcccccccc>{\columncolor[gray]{0.9}}c}
			\toprule
			\multirow{2}{*}{Method}                & \multicolumn{9}{c}{Dataset}                                                                                                                                                                                                     \\
			\cmidrule(lr){2-10}
			                                       & Cars                                                                                            & DTD           & EuroSAT       & GTSRB         & MNIST         & RESISC45      & SUN397        & SVHN          & Avg           \\
			\midrule
			                                       & \multicolumn{9}{c}{\textit{Per‑task absolute accuracies (\%)}}                                                                                                                                                                  \\ \cmidrule(lr){2-10}
			Finetuned                              & 74.0                                                                                            & 58.3          & 99.0          & 92.7          & 99.3          & 88.4          & 64.5          & 96.2          & 84.1          \\
			\midrule
			                                       & \multicolumn{9}{c}{\textit{Per‑task accuracies of merged models, normalized to finetuned (\%)}}                                                                                                                                 \\ \cmidrule(lr){2-10}
			\addlinespace[2pt]
			\textit{\textbf{Vanilla Merging}}      & \multicolumn{9}{l}{}                                                                                                                                                                                                            \\
			RegMean~\cite{jin2022dataless}         & 80.2                                                                                            & 71.3          & 37.9          & 47.3          & 43.1          & 70.5          & 93.9          & 43.0          & 60.9          \\
			TA~\cite{ilharco2022editing}           & 82.1                                                                                            & 74.3          & 48.7          & 41.8          & 53.4          & 71.5          & 96.6          & 42.0          & 63.8          \\
			TIES~\cite{yadav2023ties}              & 81.0                                                                                            & 72.5          & 53.8          & 37.4          & 69.0          & 65.3          & 94.8          & 45.3          & 64.9          \\
			DARE-TIES~\cite{yu2024language}        & 81.6                                                                                            & 74.5          & 51.0          & 37.2          & 59.2          & 66.6          & 96.0          & 38.8          & 63.1          \\
			EMR-Merging~\cite{huang2024emr}        & 82.6                                                                                            & 73.9          & 48.2          & 40.2          & 54.6          & 70.6          & 95.7          & 45.9          & 64.0          \\
			FR-Merging~\cite{zheng2025free}        & 79.9                                                                                            & 71.0          & 24.2          & 34.2          & 46.6          & 65.4          & 92.2          & 47.6          & 57.6          \\
			Iso-C~\cite{marczak2025no}             & 82.0                                                                                            & 82.1          & 55.2          & 73.0          & 68.9          & 80.3          & 97.7          & 49.6          & 73.6          \\
			Iso-CTS~\cite{marczak2025no}           & 80.8                                                                                            & \textbf{82.3} & 56.2          & \textbf{74.8} & 68.5          & 79.6          & 97.4          & 48.5          & 73.5          \\
			AdaMerging~\cite{yang2023adamerging}   & 79.5                                                                                            & 73.5          & 70.9          & 39.7          & 63.0          & 69.0          & 97.8          & 66.6          & 70.0          \\ \hline
			\addlinespace[2pt]
			\textit{\textbf{LoRA‑aware Merging}}   & \multicolumn{9}{l}{}                                                                                                                                                                                                            \\
			SVD~\cite{tang2025lora}                & 81.8                                                                                            & 73.5          & 47.7          & 38.3          & 52.8          & 70.2          & 96.4          & 40.4          & 62.6          \\
			Linear~\cite{peft}                     & 81.4                                                                                            & 74.5          & 49.5          & 43.0          & 55.3          & 70.3          & 96.8          & 39.0          & 63.7          \\
			KnOTS‑TIES~\cite{stoica2025knots}      & 82.7                                                                                            & 73.7          & 49.3          & 48.9          & 68.9          & 70.9          & 95.5          & 53.8          & 68.0          \\
			KnOTS‑DARE‑TIES~\cite{stoica2025knots} & 81.8                                                                                            & 75.9          & 50.7          & 40.3          & 53.2          & 70.2          & 97.9          & 41.0          & 63.9          \\
			LoRA-LEGO~\cite{zhao2025loralego}      & 81.1                                                                                            & 73.0          & 54.4          & 40.3          & 48.6          & 71.5          & 97.3          & 37.1          & 62.9          \\
			RobustMerge~\cite{zeng2025parameter}   & 80.1                                                                                            & 76.8          & 54.8          & 38.2          & 55.4          & 69.5          & 96.5          & 35.2          & 63.3          \\ \hline
			\addlinespace[2pt]
			\method-Variant A                      & 82.2                                                                                            & 76.0          & 74.9          & 43.5          & 76.3          & 70.2          & 98.0          & \textbf{70.8} & 74.0          \\
			\method-Variant B                      & \textbf{86.2}                                                                                   & 78.4          & \textbf{76.8} & 42.9          & \textbf{82.7} & \textbf{75.4} & \textbf{98.6} & 69.7          & \textbf{76.3} \\
			\bottomrule
		\end{tabular}
	}
	\vspace{-5pt}
\end{table*}

\vspace{2pt}
\noindent\textbf{Variant A: Per-rank LoRA direction selection.}
This variant reweights individual rank-1 factors and focuses on anisotropy control within each adapter.
Let $\phi_{ij}\in\mathbb{R}$ be a \emph{direction-selection weight} for factor $(i,j)$ with $i\in\{1,\dots,N\}$ and $j\in\{1,\dots,r_i\}$.
The merged weight is
\begin{equation}
	\bm W_{A}(\bm\phi)
	\;=\;
	\bm W_{0}\;+\;\sum_{i=1}^{N}\sum_{j=1}^{r_i}\phi_{ij}\, \vb_{ij}\,\va_{ij}^{\top}.
	\label{eq:variant-a}
\end{equation}
Unless otherwise noted, we allow signed weights and learn their scale directly.

\vspace{2pt}
\noindent\textbf{Variant B: Shared singular-direction selection.}
This variant builds a shared orthonormal basis across tasks to preserve subspace coverage, reduce cross-task interference, and control anisotropy via weights.
Form a matrix by \emph{horizontally} concatenating adapters as follows:
\begin{equation}
	\bm X = \bigl[\Delta \bm W_1,\dots,\Delta \bm W_N\bigr]
	\;\in\;\mathbb{R}^{d\times (mN)},
\end{equation}
and compute the SVD: $\bm X = \bm U\bm \Sigma \bm V^{\top}$.
Select the top $R$ singular values and their associated singular vectors, yielding
\begin{align}
	\vspace{-3pt}
	\bm U      & = [\bm u_1, \dots, \bm u_R] \in \mathbb{R}^{d \times R},    \\
	\bm \Sigma & = \mathrm{diag}(\sigma_1, \dots, \sigma_R),                 \\
	\bm V      & = [\bm v_1, \dots, \bm v_R] \in \mathbb{R}^{(mN) \times R}.
\end{align}

Partition each $\bm v_k$ into $\bm v_k=[\,\bm v_{k1}^{\top},\dots,\bm v_{kN}^{\top}\,]^{\top}$ with each $\bm v_{ki}\in\mathbb{R}^{m}$ and define
\begin{equation}
	\bm S_{ik} \;=\; \bm u_k\,\bm v_{ki}^{\top}\;\in\mathbb{R}^{d\times m}
\end{equation}
Assign weights $\phi_{ik}\in\mathbb{R}$ to these directions and obtain

\begin{equation}
	\bm W_{B}(\bm\phi)
	\;=\;
	\bm W_{0}\;+\;\sum_{i=1}^{N}\sum_{k=1}^{R}\phi_{ik}\,\sigma_k\,\bm u_k\,\bm v_{ki}^{\top}.
	\label{eq:variant-b}
\end{equation}
Together, the shared orthonormal left directions $\{\bm u_k\}_{k=1}^{R}$ and the per-task partitions of the right singular vectors preserve subspace coverage.
The left basis $\{\bm u_k\}$ retains the top-$R$ column space of $\bm X$, and $\bm v_{ki}\in\mathbb{R}^{m}$ yields \emph{rank-1 components} $\bm S_{ik}=\bm u_k \bm v_{ki}^{\top}$.
Direction weights then adjust these rank-1 components preserving subspace coverage.

Variant B incurs extra cost to compute the shared SVD basis, whereas Variant A avoids this step. In return Variant B typically attains higher accuracy, reflecting a practical trade-off between compute and performance. We quantify this trade-off in \Cref{tab:cost} of the supplementary material.

Following \citet{lin2024smooth}, we optimize $\bm\phi$ under a preference vector $\bm\rho$.
Let $\mW(\bm\phi)$ denote either $\mW_{A}(\bm\phi)$ from \cref{eq:variant-a} or $\mW_{B}(\bm\phi)$ from \cref{eq:variant-b}.

\vspace{2pt}
\noindent\textbf{Smooth Tchebycheff Scalarization.}
With anchors $z_i=\min_{\bm W} f_i(\bm W)$, define
\begin{equation}
	\Psi(\bm\phi,\bm\rho)=
	\alpha \log\!\left(\sum_{i=1}^{N}
	\exp\!\left(\!\frac{\rho_i\,\big|f_i\!\big(\bm W(\bm\phi)\big) \! - \! z_i\big|}{\alpha}\!\right)\!\right),
	\label{eq:smooth_tch}
\end{equation}
Here $z_i$ acts as a task-wise anchor that normalizes scale differences across tasks.
The parameter $\alpha > 0$ controls smoothing: \emph{smaller} $\alpha$ concentrates the objective on the worst-performing task, whereas \emph{larger} $\alpha$ yields a smoother aggregation that spreads focus more evenly across tasks.
During merging, we use predictive entropy following AdaMerging~\cite{yang2023adamerging} since $f_i$ requires labels. For $z_i$, we use the entropy loss obtained when only task $i$’s adapter is applied.

Earlier learning-free mergers were attractive for full-parameter models due to lower memory use and simplicity. In LoRA-adapted settings, however, the memory gap between learning-free and learning-based approaches is small even at foundation-model scale, so there is less reason to insist on learning-free methods. We quantify this in \Cref{tab:cost} of the supplementary material.

\section{Experiments}

\textbf{Experimental Setup.}

For vision experiments, we follow \citet{stoica2025knots} and use CLIP~\cite{radford2021learning} with a ViT-B/32 backbone~\cite{dosovitskiy2020image} pre-trained on ImageNet-21k~\cite{deng2009imagenet}. Models are fine-tuned with rank-16 LoRA unless noted otherwise and evaluated on eight image classification benchmarks. We also test our method on six classification datasets with the LLaMA-3 8B model~\cite{dubey2024llama}. Dataset descriptions are in \Cref{append:datasets}, and LoRA fine-tuning details in \Cref{append:finetune} of supplementary material.

\begin{table*}[t]
	\vspace{-10pt}
	\caption{
		Per-task accuracy on six NLI benchmarks. We merge six LLaMA-3 8B checkpoints, each fine-tuned with LoRA. The upper panel shows the absolute per-task accuracy of the fine-tuned baselines. The lower panel shows the merged models’ accuracy, normalized by each task’s corresponding fine-tuned baseline (\%).
	}
	\vspace{-7pt}
	\label{tab:six_nli_benchmarks}
	\centering
	\renewcommand{\arraystretch}{0.85}
	\scriptsize
	\setlength{\tabcolsep}{10pt}
	\resizebox{0.85\linewidth}{!}{
		\begin{tabular}{lcccccc>{\columncolor[gray]{0.9}}c}
			\toprule
			\multirow{2}{*}{Method}                & \multicolumn{7}{c}{Dataset}                                                                                                                                                                       \\
			\cmidrule(lr){2-8}
			                                       & MNLI                                                                                            & QNLI          & SNLI          & RTE            & SICK          & SCITAIL        & Avg           \\
			\midrule
			                                       & \multicolumn{7}{c}{\textit{Per-task absolute accuracies (\%)}}                                                                                                                                    \\ \cmidrule(lr){2-8}
			Finetuned                              & 90.8                                                                                            & 95.3          & 92.1          & 84.8           & 91.3          & 87.3           & 90.2          \\
			\midrule
			                                       & \multicolumn{7}{c}{\textit{Per-task accuracies of merged models, normalized to finetuned (\%)}}                                                                                                   \\ \cmidrule(lr){2-8}
			\addlinespace[2pt]
			\textit{\textbf{Vanilla Merging}}      & \multicolumn{7}{l}{}                                                                                                                                                                              \\
			TA~\cite{ilharco2022editing}           & 67.3                                                                                            & 87.3          & 41.8          & 95.7           & 77.9          & 76.9           & 74.6          \\
			TIES~\cite{yadav2023ties}              & 61.3                                                                                            & 91.5          & 38.2          & 85.5           & 76.8          & 70.5           & 70.6          \\
			DARE-TIES~\cite{yu2024language}        & 42.0                                                                                            & 72.5          & 44.1          & 77.8           & 76.9          & 86.8           & 66.7          \\
			EMR-Merging~\cite{huang2024emr}        & 75.3                                                                                            & 87.5          & 41.2          & 88.0           & 70.3          & 73.9           & 72.7          \\
			Iso-C~\cite{marczak2025no}             & 35.0                                                                                            & 57.8          & 40.5          & 64.1           & 51.1          & \textbf{100.2} & 58.1          \\
			Iso-CTS~\cite{marczak2025no}           & 37.1                                                                                            & 57.1          & 39.8          & 59.4           & 49.8          & 95.3           & 56.4          \\
			FR-Merging~\cite{zheng2025free}        & 65.7                                                                                            & 86.9          & 41.9          & 93.2           & 78.3          & 77.6           & 73.9          \\
			AdaMerging~\cite{yang2023adamerging}   & 47.5                                                                                            & 92.9          & 41.3          & 102.6          & 93.8          & 94.2           & 78.7          \\ \hline
			\addlinespace[2pt]
			\textit{\textbf{LoRA-aware Merging}}   & \multicolumn{7}{l}{}                                                                                                                                                                              \\
			SVD~\cite{tang2025lora}                & 67.7                                                                                            & 87.2          & 41.7          & 95.7           & 77.6          & 76.6           & 74.4          \\
			Linear~\cite{peft}                     & 63.1                                                                                            & 86.2          & 39.5          & 90.6           & 77.5          & 74.5           & 71.9          \\
			KnOTS-TIES~\cite{stoica2025knots}      & 41.1                                                                                            & 83.4          & \textbf{56.6} & 87.2           & 87.9          & 94.8           & 75.2          \\
			KnOTS-DARE-TIES~\cite{stoica2025knots} & \textbf{76.4}                                                                                   & 88.2          & 39.9          & 99.2           & 79.6          & 75.6           & 76.5          \\
			LoRA-LEGO~\cite{zhao2025loralego}      & 62.5                                                                                            & 84.8          & 41.4          & 100.0          & 87.0          & 83.7           & 76.6          \\
			RobustMerge~\cite{zeng2025parameter}   & 66.4                                                                                            & 87.2          & 41.8          & 94.0           & 77.7          & 77.3           & 74.1          \\ \hline
			\addlinespace[2pt]
			\method-Variant A                      & 51.7                                                                                            & 92.6          & 41.4          & 102.6          & 95.3          & 94.4           & 79.7          \\
			\method-Variant B                      & 46.8                                                                                            & \textbf{94.1} & 41.4          & \textbf{103.4} & \textbf{98.1} & 97.8           & \textbf{80.3} \\
			\bottomrule
		\end{tabular}
	}
	\vspace{-5pt}
\end{table*}

\noindent\textbf{Baselines.}
We denote the pretrained weights by $\mW_0$, the task vector of the $i$-th task by $\Delta \mW_i$, and the merged weights by $\mW_\text{merge}$.
For LoRA task vectors, we write $\Delta \mW_i = \mB_i \mA_i^\top$.
The scaling coefficient for merging is denoted by $\lambda$.
Our evaluation spans two distinct groups of baselines.
\emph{Vanilla merging} comprises
(\romannumeral 1) \textbf{RegMean}~\cite{jin2022dataless} aligns the merged model with fine-tuned models by solving a layer-wise least-squares problem using small calibration sets from each task, thereby matching the activations of the merged model to those of the task-specific models.
(\romannumeral 2) \textbf{Task Arithmetic (TA)}~\cite{ilharco2022editing} adds a scaled sum of task vectors to the pretrained weights, $\mW_\text{merge}= \mW_0 + \lambda \sum_{i=1}^{N}\Delta \mW_i$.
(\romannumeral 3) \textbf{TIES}~\cite{yadav2023ties} prunes small-magnitude parameters, resolves sign conflicts, and then averages only parameters with consistent signs with scaling coefficient $\lambda$.
(\romannumeral 4) \textbf{DARE-TIES}~\cite{yu2024language} sparsifies each parameter with Bernoulli probability $p$ and rescales the retained parameters by $1/(1-p)$ to preserve the expectation.
(\romannumeral 5) \textbf{AdaMerging}~\cite{yang2023adamerging} learns merging coefficients by minimizing an output-entropy surrogate in the spirit of test-time adaptation~\cite{wang2020tent}.
\emph{LoRA-aware merging} exploits low-rank adapter structure:
(\romannumeral 6) \textbf{SVD} aggregates LoRA task vectors $\Delta \mW_\text{merge}=\lambda \sum_i \mB_i \mA_i^\top$ and applies truncated SVD $\Delta \mW_{\text{merge}} \approx \mU_r\mSigma_r \mV_r^{\top}$~\cite{tang2025lora} to recover rank of original LoRAs
(\romannumeral 7) \textbf{Linear}~\cite{peft} performs TA on $\left\{ \mB_i \right\}$ and $\left\{ \mA_i \right\}$ rather than whole task vector.
(\romannumeral 8) KnOTS~\cite{stoica2025knots} computes an SVD on concatenated task vectors
$\left[\Delta \mW_1;\dots;\Delta \mW_N \right]=\mU\mSigma \mV^\top$ and merge partitions of $\mV$ associated with each task vector.
TIES or DARE-TIES can be applied to the $\mV_i$, yielding the variants \textbf{KnOTS-TIES} and \textbf{KnOTS-DARE-TIES}.
(\romannumeral 9) \textbf{LoRA-LEGO}~\cite{zhao2025loralego} decomposes each adapter into minimal semantic units, clusters them rank-wise, and assembles a new LoRA from the cluster centroids. We present a reimplementation of LoRA-LEGO.
Finally, we also include the LoRA-aware (\romannumeral 10) \textbf{RobustMerge}~\cite{zeng2025parameter}, and recent vanilla-merging methods: (\romannumeral 11) \textbf{EMR-Merging}~\cite{huang2024emr}, (\romannumeral 12) \textbf{FR-Merging}~\cite{zheng2025free}, and (\romannumeral 13) \textbf{Iso-CTS}~\cite{marczak2025no}.
Implementation details are in \cref{append:imple}.

\vspace{2pt}
\noindent\textbf{Implementation Details.} For our method, we adopt AdamW~\cite{loshchilov2018decoupled} with a learning rate of 0.001. The direction-selection weights $\bm \phi$ are initialized to 0.4 and optimized with a batch size of 16 for 500 iterations similar to \citet{yang2023adamerging} to control training cost. For experiments, we set $\alpha=1$ in \cref{eq:smooth_tch}. The effects of $\alpha$ is analyzed in supple.

\vspace{2pt}
\noindent\textbf{Metrics.} Following prior work~\cite{ilharco2022editing, yadav2023ties, stoica2025knots}, we report the absolute accuracy of individually fine-tuned models on each dataset, and compare different merging methods using a normalized accuracy metric. The normalized accuracy is defined as $\frac{\text{Accuracy of merged model on task-$i$}}{\text{Accuracy of fine-tuned model on task-$i$}}$. This metric indicates how closely a merged model approaches the performance of the corresponding fine-tuned model for each task.

\subsection{Experimental Results}

\noindent\textbf{Per-Task Evaluation across Vision Tasks.}
In this experiments, we adopt the conventional per-task evaluation protocol~\cite{ilharco2022editing, yadav2023ties, stoica2025knots}, where the goal is to test how well different merging strategies preserve task-specific performance. Concretely, a collection of models fine-tuned independently on separate datasets are merged into a unified model, and its performance is measured on each dataset separately, relying solely on the dataset’s own samples and labels.
Table~\ref{tab:8vis_knots} reports normalized per-task accuracies and their averages across eight vision tasks with CLIP, where all models are fine-tuned with LoRA of rank 16. We categorize the baselines into two groups: vanilla merging methods and LoRA-aware methods that explicitly incorporate the LoRA structure.
LoRA-aware methods generally outperform their vanilla counterparts. For instance, KnOTS-TIES surpasses TIES, and KnOTS-DARE-TIES improves upon DARE-TIES. KnOTS-based variants highlight the benefits of accounting for LoRA structure. LoRA-LEGO does not provide significant advantages over vanilla methods, suggesting that merely preserving LoRA rank-1 directions is insufficient to retain task-specific information during merging. By comparison, \method consistently outperforms both vanilla and LoRA-aware baselines.
Variant~A shows steady gains with additional training iterations, leading to higher average accuracy.
Variant~B achieves the best overall trade-off, delivering the highest scores average of 76.3\% with only 250 iterations.
These results demonstrate that exploiting LoRA structure is essential for stable merging, and that \method further effectively capture task-sensitive subspaces. Experiments in \Cref{tab:8vis_knots} use the checkpoints released by KnOTS~\cite{stoica2025knots}. We further confirm the results on independently trained checkpoints, as shown in \Cref{tab:8vis_ours}.

\vspace{2pt}
\noindent\textbf{Per-Task Evaluation on 6 NLI Tasks with LLMs.}
In \Cref{tab:six_nli_benchmarks}, we report per-task normalized accuracies of merged models on six NLI benchmarks.
For this setting, we apply six LoRA adapters with rank 16 to LLaMA-3-8B~\cite{dubey2024llama}.
Each benchmark follows the natural language inference protocol, where a hypothesis is classified against a given premise into entailment, contradiction, or neutral.
Detailed descriptions of each dataset are provided in \Cref{append:datasets} of supplementary material..
The NLI results with LLMs reveal trends consistent with those observed in vision tasks.
Vanilla merging baselines, with the exception of AdaMerging, show limited ability to preserve task-specific information in large-scale models, while LoRA-aware methods still provide only moderate improvements.
In contrast, \method delivers significant and consistent gains. Variant~B achieving the best average normalized accuracy of 80.3\%, highlighting the effectiveness of \method on language tasks.

\begin{table}[t]
	\vspace{-5pt}
	\caption{\textbf{Joint-Task Evaluation Results.} All models share a ViT-B/32 backbone, are fine-tuned with LoRA on separate datasets, and are then merged for joint-task evaluation.}
	\vspace{-5pt}
	\label{tab:union_results}
	\centering
	\tiny
	\setlength{\tabcolsep}{8pt}
	\resizebox{0.99\linewidth}{!}{%
		\begin{tabular}{lccc}
			\toprule
			Method                                 & Hits@1        & Hits@3        & Hits@5        \\
			\midrule
			TA~\cite{ilharco2022editing}           & 43.5          & 65.2          & 74.0          \\
			TIES~\cite{yadav2023ties}              & 43.6          & 65.3          & 73.9          \\
			DARE-TIES~\cite{yu2024language}        & 44.0          & 66.4          & 75.1          \\
			AdaMerging~\cite{yang2023adamerging}   & 48.1          & 73.2          & 83.0          \\
			KnOTS-TIES ~\cite{stoica2025knots}     & 46.8          & 68.1          & 76.3          \\
			KnOTS-DARE-TIES~\cite{stoica2025knots} & 45.2          & 66.9          & 75.3          \\
			LoRA-LEGO~\cite{zhao2025loralego}      & 43.1          & 65.0          & 73.9          \\
			\hline
			\method-Variant A                      & \textbf{51.1} & \textbf{75.9} & 84.9          \\
			\method-Variant B                      & 49.3          & 74.9          & \textbf{85.1} \\
			\bottomrule
		\end{tabular}
	}
	\vspace{-10pt}
\end{table}

\begin{figure}[t]
	\vspace{-10pt}
	\centering
	\includegraphics[width=0.40\textwidth]{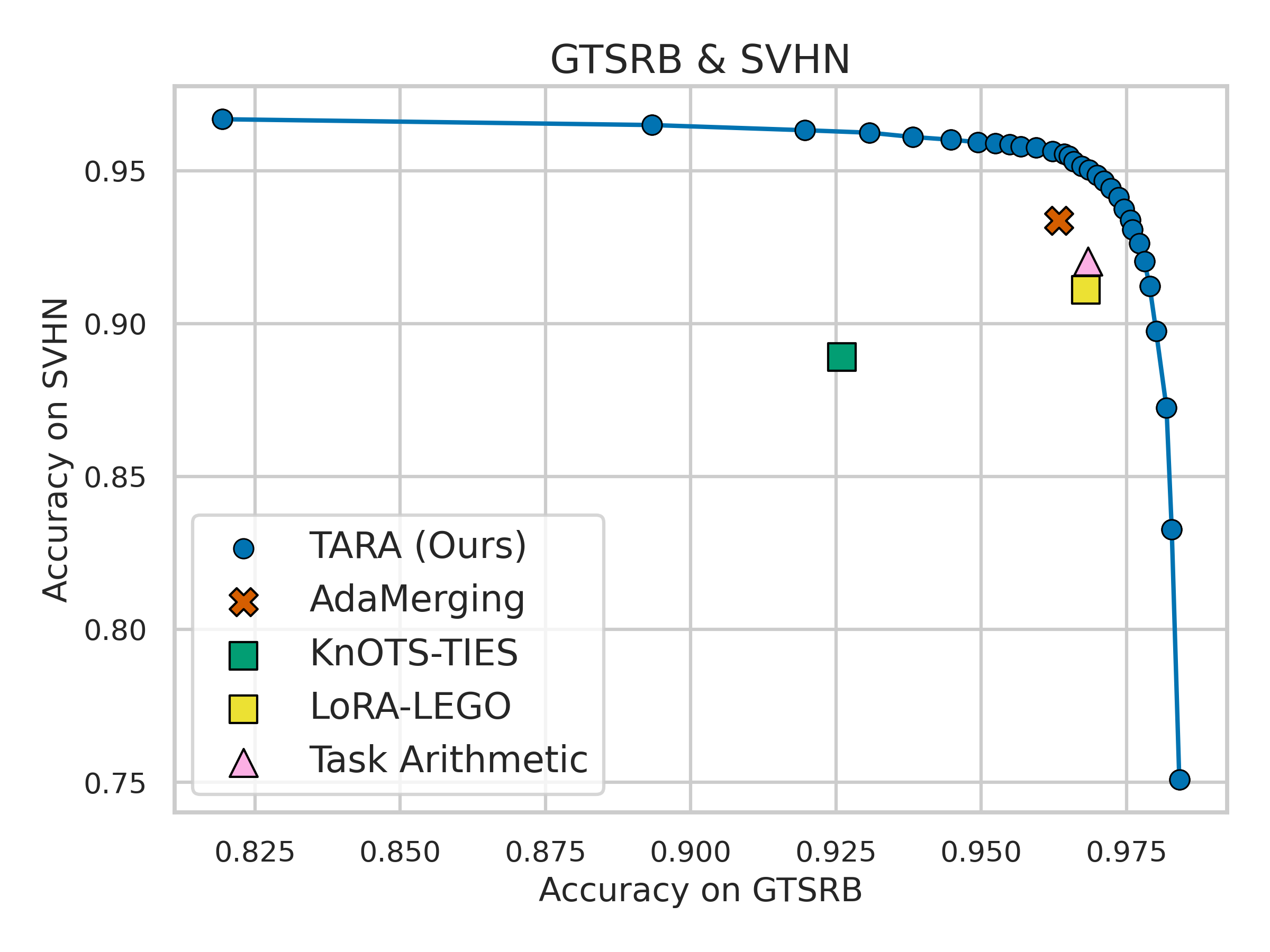}
	\vspace{-3pt}
	\captionof{figure}{\textbf{Two-Task Merging Results.} \label{fig:pareto_GTSRB_SVHN}}
	\vspace{7pt}
	\includegraphics[width=0.40\textwidth]{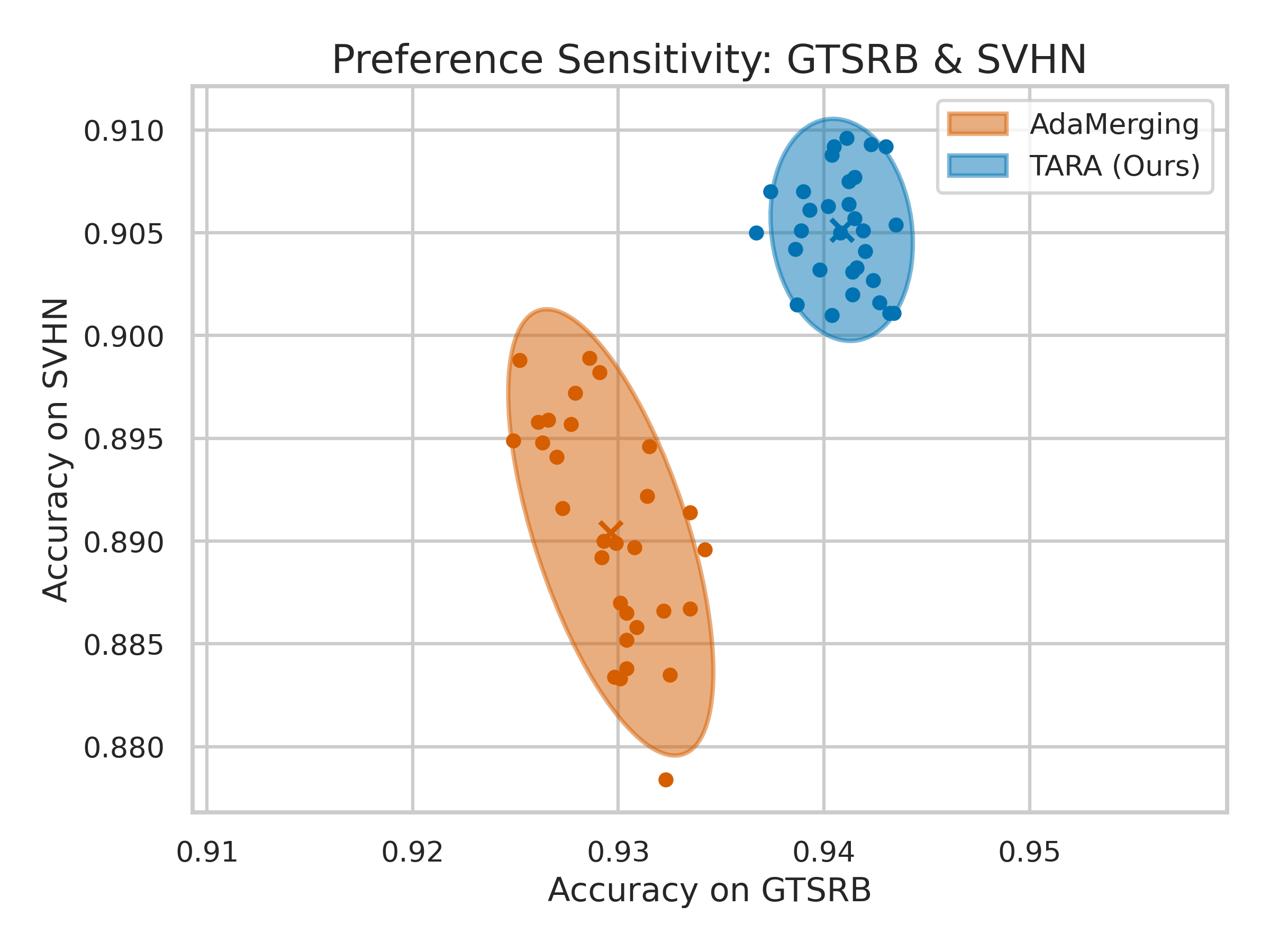}
	\vspace{-5pt}
	\captionof{figure}{\textbf{Preference Sensitivity} of \method\ and AdaMerging. \label{fig:sensitivity}}
	\vspace{-15pt}
\end{figure}

\begin{table*}[t]
	\vspace{-12pt}
	\centering
	\caption{
		Generalization results on two unseen tasks when merging ViT-B/32 models on six tasks.}
	\vspace{-5pt}
	\renewcommand{\arraystretch}{0.85}
	\resizebox{\linewidth}{!}{
		\begin{tabular}{l|ccccccc|ccc|c}
			\toprule
			                                         & \multicolumn{7}{c|}{{Seen Tasks}}                                                                & \multicolumn{3}{c|}{{Unseen Tasks}} & \multicolumn{1}{c}{{All Tasks}}                                                                                                                                          \\
			\midrule \midrule
			{Method}                                 & {Cars}                                                                                           & {DTD}                               & {GTSRB}                         & {RESISC45}    & {SUN397}      & {SVHN}        & \textbf{Avg Acc} & {EuroSAT}     & {MNIST}       & \textbf{Avg Acc} & \textbf{Avg Acc} \\
			\midrule
			                                         & \multicolumn{11}{c}{\textit{Per-task accuracies of merged models, normalized to finetuned (\%)}}                                                                                                                                                                                                                  \\ \cmidrule(lr){2-12}
			{TA}~\cite{ilharco2022editing}           & 82.0                                                                                             & \textbf{74.7}                       & 42.2                            & 71.1          & 96.8          & 42.2          & 68.2             & 37.5          & 48.2          & 42.9             & 61.8             \\
			{TIES}~\cite{yadav2023ties}              & 81.5                                                                                             & 73.2                                & 40.3                            & 64.4          & 95.4          & 45.7          & 66.7             & 37.2          & \textbf{61.8} & 49.5             & 62.4             \\
			{KnOTS-TIES~\cite{stoica2025knots}}      & 82.7                                                                                             & 74.4                                & 48.5                            & 72.5          & 95.3          & 50.2          & 70.6             & 33.2          & 50.5          & 41.8             & 63.4             \\
			{KnOTS-DARE-TIES}~\cite{stoica2025knots} & 81.7                                                                                             & 72.5                                & \textbf{49.4}                   & 71.9          & 94.6          & 52.5          & 70.4             & 31.6          & 51.2          & 41.4             & 63.2             \\
			{LoRA-LEGO}~\cite{zhao2025loralego}      & 82.2                                                                                             & 73.5                                & 46.9                            & 71.2          & 96.7          & 40.8          & 68.5             & 35.5          & 44.7          & 40.1             & 61.4             \\
			{AdaMerging}~\cite{yang2023adamerging}   & 79.7                                                                                             & 73.4                                & 37.7                            & 69.8          & 97.9          & 67.4          & 71.0             & 48.7          & 58.7          & 53.7             & 66.7             \\ \hline
			\rowcolor[gray]{0.9} {\method-Variant A} & 82.8                                                                                             & 73.9                                & 43.8                            & 70.5          & 98.2          & \textbf{73.4} & 73.8             & 48.4          & 61.1          & 54.8             & 69.1             \\
			\rowcolor[gray]{0.9} {\method-Variant B} & \textbf{86.3}                                                                                    & 73.8                                & 47.6                            & \textbf{73.8} & \textbf{99.1} & 71.6          & \textbf{75.4}    & \textbf{53.3} & 61.6          & \textbf{57.5}    & \textbf{70.9}    \\
			\midrule
			\addlinespace[5pt]
			\midrule
			{Method}                                 & {Cars}                                                                                           & {DTD}                               & {EuroSAT}                       & {GTSRB}       & {MNIST}       & {SUN397}      & \textbf{Avg Acc} & {RESISC45}    & {SVHN}        & \textbf{Avg Acc} & \textbf{Avg Acc} \\
			\midrule
			                                         & \multicolumn{11}{c}{\textit{Per-task accuracies of merged models, normalized to finetuned (\%)}}                                                                                                                                                                                                                  \\ \cmidrule(lr){2-12}
			{TA}~\cite{ilharco2022editing}           & 82.3                                                                                             & 75.1                                & 53.0                            & 40.7          & 52.4          & 96.7          & 66.7             & 68.6          & 34.7          & 51.7             & 63.0             \\
			{TIES}~\cite{yadav2023ties}              & 81.9                                                                                             & 74.8                                & 73.7                            & 33.2          & 63.4          & 95.6          & 70.4             & 68.9          & 31.6          & 50.3             & 65.4             \\
			{KnOTS-TIES}~\cite{stoica2025knots}      & 82.5                                                                                             & 72.9                                & 53.8                            & 47.8          & 61.7          & 94.9          & 69.0             & 67.5          & 35.1          & 51.3             & 64.5             \\
			{KnOTS-DARE-TIES}~\cite{stoica2025knots} & 82.7                                                                                             & 72.2                                & 53.1                            & \textbf{48.5} & 61.5          & 94.3          & 68.7             & 66.7          & 35.4          & 51.0             & 64.3             \\
			{LoRA-LEGO}~\cite{zhao2025loralego}      & 81.1                                                                                             & 72.3                                & 55.6                            & 42.8          & 62.6          & 94.7          & 68.2             & 65.2          & 33.9          & 49.5             & 63.5             \\
			{AdaMerging}~\cite{yang2023adamerging}   & 79.6                                                                                             & 74.2                                & 72.5                            & 36.5          & 60.0          & 97.9          & 70.1             & 69.2          & 41.1          & 55.2             & 66.4             \\ \hline
			\rowcolor[gray]{0.9} {\method-Variant A} & 84.9                                                                                             & 75.7                                & 76.3                            & 40.5          & \textbf{77.4} & 97.5          & 75.4             & 69.5          & 42.1          & 55.8             & 70.5             \\
			\rowcolor[gray]{0.9} {\method-Variant B} & \textbf{86.7}                                                                                    & \textbf{77.0}                       & \textbf{79.5}                   & 42.8          & 76.3          & \textbf{98.2} & \textbf{76.8}    & \textbf{70.5} & \textbf{45.6} & \textbf{58.1}    & \textbf{72.1}    \\
			\bottomrule
		\end{tabular}
	}
	\label{tab:general_knots}
	\vspace{-5pt}
\end{table*}

\vspace{2pt}
\noindent\textbf{Joint-Task Evaluation of General Models.}
The joint-task evaluation protocol, originally introduced by \citet{stoica2025knots}, differs from the per-task setup by evaluating merged models over the union of inputs and labels from all eight vision benchmarks. This setting is more suitable for assessing whether a merged model can serve as a general-purpose model across multiple target datasets. In constructing the joint task, labels from all benchmarks are pooled together and duplicates are removed, resulting in a unified label space. A detailed description of this setting is provided in \Cref{append:joint_settings} of supplementary material..
Table~\ref{tab:union_results} presents joint-task evaluation results in terms of top-1/3/5 accuracies. Vanilla merging baselines achieve limited performance with Hits@1 below 45\%. KnOTS-TIES and KnOTS-DARE-TIES obtain small but consistent gains, while AdaMerging achieves further improvements when additional optimization iterations are allowed. \method consistently outperforms all baselines. Both Variant~A and Variant~B benefit from longer training, with Variant~A (500 Iters) reaching the highest Hits@1 (51.1\%) and strong overall balance, while Variant~B delivers the best Hits@5 (85.1\%). These results highlight that our merging strategy effectively integrates task-specific knowledge across datasets and yields more general models under the challenging joint-task setting. The full results are provided in \Cref{tab:8vision_joint} of supplementary material.

\vspace{2pt}
\noindent\textbf{Two-Task Trade-Off Analysis.}
We examine whether preference-aware merging can trace an accuracy trade-off.
Figure~\ref{fig:pareto_GTSRB_SVHN} shows results on the CLIP ViT-B/32 backbone. For the GTSRB-SVHN pair, sweeping 30 preference values and merging the two adapters with \method yields a smooth trade-off curve (line). Baseline mergers (AdaMerging, KnOTS, LoRA-LEGO, TA) appear as isolated operating points. By jointly accounting for \emph{subspace coverage} (retaining effective rank) and \emph{anisotropy} (reweighting directional emphasis), \method produces merged models that reflect diverse preferences and in turn typically lie above the baseline points, offering stronger accuracy trade-offs at comparable performance levels and enabling selection of an operating point by adjusting the preference.

\vspace{2pt}
\noindent\textbf{Preference Sensitivity Analysis.}
To assess robustness to changes in the preference vector, we adopt the following setting. We fix the GTSRB and SVHN preferences to \(0.125\) each and randomly sample the remaining task preferences on the simplex so that all entries sum to \(1\). We sample 30 preference vectors and merge eight CLIP ViT-B/32 adapters with \method\ and with AdaMerging, then record accuracies on GTSRB and SVHN. In \Cref{fig:sensitivity} each point corresponds to one sampled preference and its resulting two-task accuracies. For AdaMerging we use a weighted sum of per-task pseudo-objectives.
Compared to AdaMerging, \method\ produces a tighter point cloud with lower empirical covariance and more stable two-task performance under preference perturbations, consistent with its direction-aware handling during LoRA merging.

\vspace{2pt}
\noindent\textbf{Evaluation of Generalization to Unseen Tasks.}
We assess the generalization performance of \method following the evaluation protocol of \citet{yang2023adamerging}. In this setting, \textit{seen tasks} are those for which LoRA fine-tuned checkpoints are available and included during merging, whereas \textit{unseen tasks} correspond to datasets without any associated LoRA checkpoints. As shown in \Cref{tab:general_knots}, models are merged on six seen tasks and subsequently evaluated on both the seen tasks themselves and two unseen tasks.
The results demonstrate that vanilla baselines and LoRA-aware methods struggle to achieve strong performance on unseen datasets. In contrast, \method consistently outperforms all baselines across both evaluation groups. Variant~B achieves the highest overall accuracy of 70.9\% in the first split, while Variant~B reaches 72.1\% in the second split. These results show that \method preserves in-domain performance and generalizes better to unseen benchmarks.

We provide more detailed analyses, including computational and memory costs, in~\cref{append:additional_exp} of the supplement.

\section{Conclusion}
We revisited LoRA merging through two complementary lenses: \emph{subspace coverage}, where LoRA updates collectively occupy a wide, low-redundancy subspace rather than collapsing onto a few shared directions, and \emph{anisotropy}, the directional imbalance in how changes in the scale of LoRA rank-1 directions translate into loss changes.
Ignoring either can collapse useful subspaces and bias loss-critical directions.
Building on effective-rank and directional-sensitivity analyses, we introduce \method-Merging to align direction weights with task preferences while preserving task-relevant subspaces.
Direction-wise reweighting maintains high subspace coverage and accounts for anisotropy, yielding consistent gains across per-task, joint, and transfer evaluations. This demonstrates that addressing both factors is key to robust, general-purpose LoRA merging.

\vspace{5pt}
\noindent\textbf{Acknowledgements}
This research was supported by the Challengeable Future Defense Technology Research and Development Program through the Agency For Defense Development(ADD) funded by the Defense Acquisition Program Administration(DAPA) in 2026(No.915102201)

{
		\small
		\bibliographystyle{ieeenat_fullname}
		\bibliography{main}
	}


\clearpage
\setcounter{page}{1}
\maketitlesupplementary

\setcounter{section}{0}
\renewcommand{\thesection}{\Alph{section}}

\setcounter{figure}{4}
\setcounter{table}{4}

\section{Proof of \cref{prop:anisotropy-bounds}}

\propone*

\begin{proof}
	Let $\bm J\in\mathbb{R}^{N\times K}$ be the restricted task-loss Jacobian with entries $\bm J_{i,k}=\langle\nabla f_i(\bm W),\,\bm S_k\rangle$, and take its SVD $\bm J=\bm U\bm\Sigma \bm V^{\!\top}$, where $\bm U\in\mathbb{R}^{N\times q}$ and $\bm V\in\mathbb{R}^{K\times q}$ have orthonormal columns ($\bm U^{\!\top}\bm U=\bm V^{\!\top}\bm V=\bm I_q$), $\bm\Sigma=\mathrm{diag}(\sigma_1,\dots,\sigma_q)$ with $\sigma_1\ge\dots\ge\sigma_q>0$, and $q=\mathrm{rank}(\bm J)$. For any coefficient vector $\bm\phi\in\mathbb{R}^{K}$, define $\tilde{\bm\phi}=\bm V^{\!\top}\bm\phi\in\mathbb{R}^{q}$. By orthogonal invariance of the Euclidean norm,
	\begin{align}
		\|\bm J\bm\phi\|_2
		 & = \|\bm U\bm\Sigma \bm V^{\!\top}\bm\phi\|_2
		= \|\bm\Sigma \bm V^{\!\top}\bm\phi\|_2
		= \|\bm\Sigma \tilde{\bm\phi}\|_2.
		\label{eq:orth-invariance-J}
	\end{align}
	The second equality holds since multiplication by $\bm U$ preserves the $\ell_2$ norm: for any $y\in\mathbb{R}^{q}$, $\|\bm U y\|_2=\|y\|_2$. Also, because $\bm V$ has orthonormal columns, $\|\bm V^{\!\top}\bm\phi\|_2\le\|\bm\phi\|_2$.

	Since $\bm\Sigma$ is diagonal,
	\begin{align}
		\|\bm J\bm\phi\|_2^2
		= \|\bm\Sigma \tilde{\bm\phi}\|_2^2
		= \sum_{i=1}^{q} \sigma_i^2\,\tilde\phi_i^{\,2}.
	\end{align}
	For the upper bound, use $\sigma_i\le\sigma_{\max}(\bm J)$ and $\|\tilde{\bm\phi}\|_2\le\|\bm\phi\|_2$ to obtain
	\begin{align}
		\|\bm J\bm\phi\|_2^2
		= \sum_{i=1}^{q} \sigma_i^2\,\tilde\phi_i^{\,2}
		\le \sigma_{\max}(\bm J)^2 \sum_{i=1}^{q} \tilde\phi_i^{\,2}
		\le \sigma_{\max}(\bm J)^2 \|\bm\phi\|_2^2,
	\end{align}
	hence $\|\bm J\bm\phi\|_2 \le \sigma_{\max}(\bm J)\,\|\bm\phi\|_2$.

	For the lower bound, decompose $\bm\phi=\bm\phi_{\mathrm{R}}+\bm\phi_{\mathrm{N}}$ with $\bm\phi_{\mathrm{R}}:=\bm V\bm V^{\!\top}\bm\phi\in\mathrm{range}(\bm V)$ and $\bm\phi_{\mathrm{N}}:=(\bm I-\bm V\bm V^{\!\top})\bm\phi\in\ker(\bm J)$. Equivalently, $\tilde{\bm\phi}=(\tilde\phi_1,\ldots,\tilde\phi_q)$ are the coordinates of $\bm\phi_{\mathrm{R}}$ in the $\bm V$-basis. Then
	\begin{align}
		\|\bm J\bm\phi\|_2^2
		= \sum_{i=1}^{q} \sigma_i^2\,\tilde\phi_i^{\,2}
		\ge \sigma_{\min}(\bm J)^2 \sum_{i=1}^{q}\tilde\phi_i^{\,2}
		= \sigma_{\min}(\bm J)^2 \|\bm\phi_{\mathrm{R}}\|_2^2,
	\end{align}
	which implies $\|\bm J\bm\phi\|_2 \ge \sigma_{\min}(\bm J)\,\|\bm\phi_{\mathrm{R}}\|_2$. If one restricts to the admissible subspace that excludes the nullspace (e.g., the LoRA span intersected with $\ker(\bm J)^\perp$), then $\bm\phi_{\mathrm{N}}=0$ and $\|\bm\phi_{\mathrm{R}}\|_2=\|\bm\phi\|_2$, yielding
	\[
		\sigma_{\min}(\bm J)\,\|\bm\phi\|_2 \;\le\; \|\bm J\bm\phi\|_2 \;\le\; \sigma_{\max}(\bm J)\,\|\bm\phi\|_2.
	\]
	(Equivalently, interpret $\sigma_{\min}(\bm J)$ as the smallest \emph{nonzero} singular value on the feasible subspace.)
\end{proof}

\section{Additional Related Work}
\label{sec:additional-related-work}

\paragraph{Pre-Merging and Additional Model Merging.}
Pre-merging studies~\cite{zhang2025lori,tang2023parameter,zhang2025unraveling,zhuang2025come} clarify and formalize the linear-composition behavior that \emph{Task Arithmetic} relies on: they analyze when parameter-space combinations approximate joint training via partial or tangent linearization \cite{liu2023tangent,tang2023parameter,jin2024fine,ortiz2024task} and connect these to NTK-style local linearization \cite{jacot2018neural}. Building on this foundation, recent formulations explore optimal-transport fusion for Transformers \cite{imfeld2023transformer}, cycle-consistent multi-model merging \cite{crisostomi2024c}, deriving layer-wise coefficients from model-internal statistics \cite{huang2024emr}, permutation with least-squares alignment \cite{nasery2025pleas}, constructing merging recipes from collections of fine-tuned models with diverse hyperparameters \cite{wortsman2022model} with early evidence on CLIP that averaging robustness-oriented fine-tunes can improve zero-shot robustness \cite{wortsman2022robust}, evolutionary search for recipe discovery \cite{akiba2025evolutionary}, sparsity- and magnitude-aware sampling to reduce interference \cite{deep2024della}, and dynamic Fisher weighting guided by Bayesian optimization \cite{lee2025dynamic}. AIM \cite{nobari2025activation} uses the information from the activation space of LLMs for merging. Another line of work leverages intermediate activations during inference or training and uses feature responses in specific layers to guide the merging process~\cite{nobari2025activation,liu2025sens,wu2025unlocking}.

\paragraph{Additional LoRA-targeted Merging.}
Within \emph{LoRA-targeted} methods, a training-free framework decouples direction and scale and orthogonalizes adapter directions before merging \cite{zheng2025decouple}. \citet{zhang2023composing} compose PEFT modules via weight-space arithmetic, enabling flexible transfer without extra training. \citet{zhang2025unraveling} enforce orthogonal LoRA subspaces pre-finetuning, reducing interference and preserving merge performance. \citet{liu2025sens} adjusts task and layer-wise merging coefficients using activation/gradient sensitivity and cross-task transferability. Several methods directly exploit low-rank structure for merging~\cite{prabhakar2024lora,gargiulo2025task,panariello2025accurate}. In particular, \citet{panariello2025accurate} propose a core space that can be efficiently combined with existing merging baselines.

\paragraph{Model Merging in LLMs and VLMs.}
More recent work on model merging directly targets LLMs and VLMs and proposes strategies tailored to these architectures~\cite{chen2025bring,du2025adamms,lu2024twin,wei2025unifying,zhou2024metagpt,zhu2025remedy,zeng2025parameter}. In particular, RobustMerge~\cite{zeng2025parameter} introduces a merging method for MLLMs that exploits directional robustness in a low-rank space. \citet{zhang2025beyond} leverage the rotation symmetry of self-attention layers, which substantially enlarges the equivalence set of Transformer models compared to permutation-based symmetries.

\paragraph{Model Merging for Multi-Task and Continual Learning.}
Another line of work explicitly connects model merging with conventional multi-task learning~\cite{yang2024surgeryv2,yang2024representation,shen2024efficient,wei2025modeling,lee2025interaction}, viewing merging as a mechanism for parameter sharing and representation consolidation across tasks. A complementary set of studies aims to localize task-specific information in the parameters or to quantify interference within linear layers~\cite{davari2024model,tam2023merging,wang2024localizing,cheng2025whoever}, while \citet{marczak2025no} further decompose the parameter space into shared and task-specific subspaces. MuDSC~\cite{xu2024training} also explores heterogeneous settings with diverse dense prediction tasks for evaluating merging performance. However, it is not directly comparable to our setting because it allows branch-like architectures such as ZipIt~\cite{stoica2023zipit}. Model merging has also been studied in continual learning scenarios~\cite{Dziadzio_2025_CVPR,tang2025merging,yang2025continual,wang2024lines,wang2025memoir,qiu2025null}, where merging is used to mitigate catastrophic forgetting and to accumulate knowledge across tasks over time.

\paragraph{LoRA Composition.}
Dynamic routing or mixture-of-experts compositions learn to combine multiple LoRAs at inference time \cite{wu2024mixture,liao2025hmora,li2024mixlora,tang2024smile}, and few-shot dynamic composition improves cross-task transfer \cite{huang2024lorahub,horoi2025less}. In vision, multi-LoRA merging for multi-task recognition demonstrates modularity \cite{kesim2024multi}. In diffusion, multi-LoRA composition and subject-style concept mixing further highlight modularity \cite{zhong2024multi,gandikota2024concept,zhuang2025timestep,zou2025cached,shah2024ziplora,yang2024lora}. These methods dynamically assign or route adapters per input or task in an MoE-like fashion, which alters the inference-time architecture and falls outside our scope.

\section{Experimental Settings}
\subsection{Datasets}
\label{append:datasets}

\subsubsection{Vision Benchmarks}
Following \citet{stoica2025knots}, we adopt a ViT-B/32 backbone pre-trained on ImageNet-21k~\cite{deng2009imagenet}, fine-tuned with LoRA rank 16. Evaluation is conducted on eight standard image-classification datasets. Brief summaries are provided below.

\noindent\textbf{SUN397~\cite{xiao2010sun}.}
Large-scale scene recognition dataset with 397 categories and 108,754 images, each class having at least 100 samples.

\noindent\textbf{Cars~\cite{krause20133d}.}
Fine-grained car recognition covering 196 categories with 16,185 images, evenly split into train and test.

\noindent\textbf{RESISC45~\cite{cheng2017remote}.}
Remote sensing benchmark with 45 scene classes and 31,500 images, about 700 per class.

\noindent\textbf{EuroSAT~\cite{helber2019eurosat}.}
Satellite image classification dataset of 27,000 images in 10 land-use categories across diverse regions.

\noindent\textbf{SVHN~\cite{yuval2011reading}.}
Street View House Numbers dataset with 10 digit classes, 73,257 train images, 26,032 test images, and 500k+ extra samples.

\noindent\textbf{GTSRB~\cite{stallkamp2011german}.}
Traffic sign recognition dataset with 43 categories and over 50,000 labeled images.

\noindent\textbf{MNIST~\cite{lecun1998mnist}.}
Classic handwritten digit dataset of 70,000 grayscale images evenly distributed across 10 classes (60k train / 10k test).

\noindent\textbf{DTD~\cite{cimpoi2014describing}.}
Texture classification dataset with 47 attributes and 5,640 images, about 120 per class.

\subsubsection{Language Benchmarks}
In addition to the main benchmarks used in the paper, we evaluate our method on six classification datasets to assess its applicability to large language models. These datasets are used in conjunction with the LLaMA-3 8B model. Below, we briefly summarize each dataset.

\noindent\textbf{QNLI~\cite{wang2018glue}.}
A question-answering dataset reformulated as sentence-pair classification. Each example is a (question, sentence) pair, and the label indicates whether the sentence contains the answer to the question (binary).

\noindent\textbf{MNLI~\cite{williams2018broad}.}
A large, crowdsourced collection of sentence pairs annotated for natural language inference. Given a premise and a hypothesis, the task is to predict one of three labels: \emph{entailment}, \emph{contradiction}, or \emph{neutral} (three-class).

\noindent\textbf{SNLI~\cite{bowman2015large}.}
A corpus of \(\sim\)570k human-written sentence pairs labeled for \emph{entailment}, \emph{contradiction}, or \emph{neutral}, supporting the standard NLI task (three-class).

\noindent\textbf{RTE~\cite{dagan2005pascal, haim2006second, giampiccolo2007third, bentivogli2009fifth}.}
A suite of textual entailment datasets compiled from the RTE1, RTE2, RTE3, and RTE5 challenges. Examples are drawn from news and Wikipedia. In the GLUE formulation, all RTE sets are converted to binary classification by collapsing \emph{neutral} and \emph{contradiction} into \emph{not-entailment} for consistency.

\noindent\textbf{SICK~\cite{marelli2014semeval}.}
A dataset of around 10k sentence pairs built from image and video captions. Each pair is labeled for semantic relatedness (1-5 scale) and for textual entailment with three classes: \emph{entailment}, \emph{contradiction}, and \emph{neutral}. It is widely used for evaluating compositional semantics, entailment, and similarity in a unified setting.

\noindent\textbf{SciTail~\cite{khot2018scitail}.}
An entailment dataset derived from multiple-choice science exams and web sentences. Each (premise, hypothesis) pair is labeled as \emph{entails} or \emph{neutral} (binary). The dataset contains 27,026 examples (10,101 \emph{entails} and 16,925 \emph{neutral}).

\begin{figure*}[t]
	\centering
	\begin{minipage}{0.80\linewidth}
		\centering
		\begin{subfigure}[t]{0.49\linewidth}
			\includegraphics[width=\linewidth]{figure/subspace_coverage_new2/fig_lines_q_proj_raw.png}
			\subcaption{Query}\label{fig:spread_q_norm}
		\end{subfigure}\hfill
		\begin{subfigure}[t]{0.49\linewidth}
			\includegraphics[width=\linewidth]{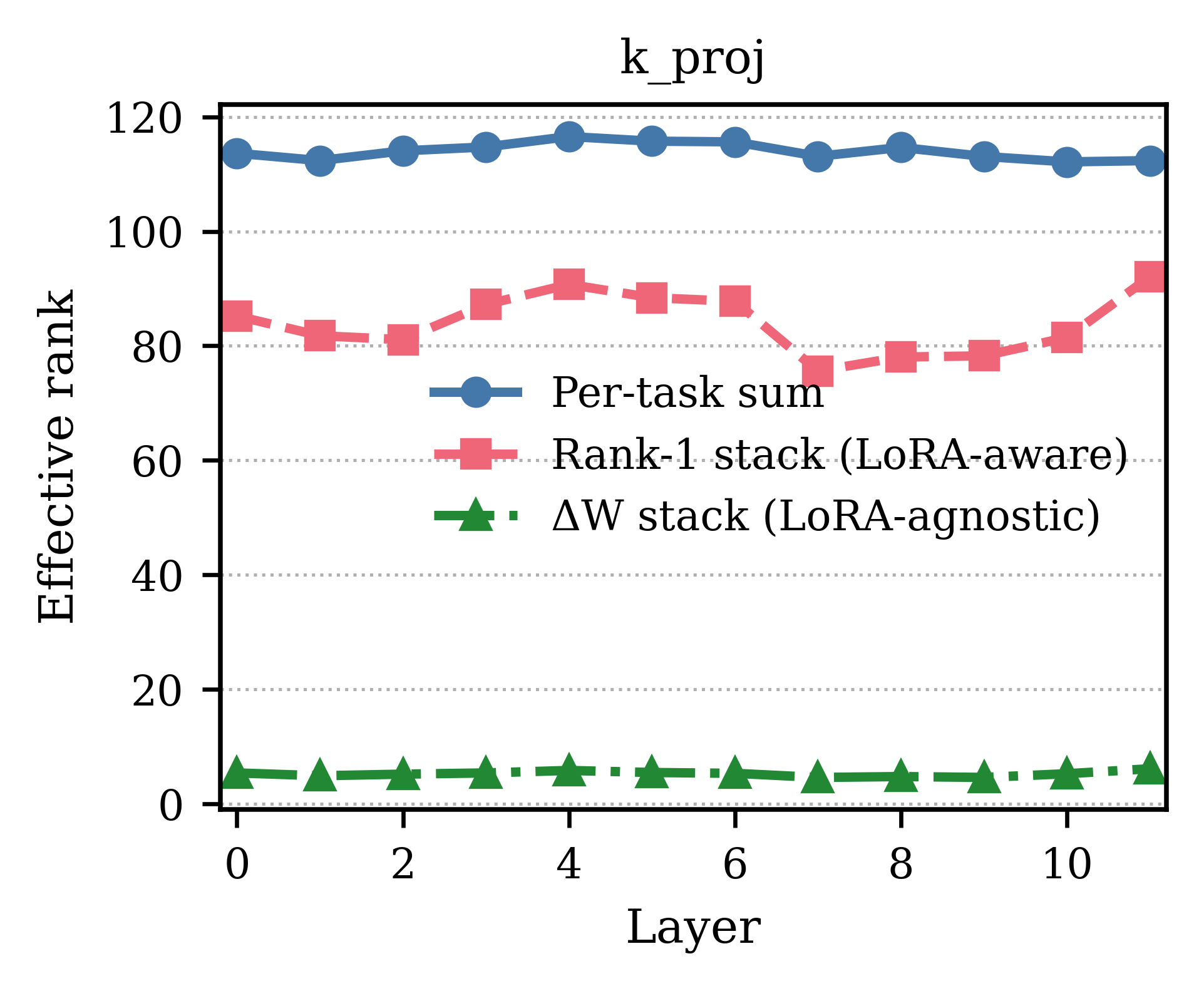}
			\subcaption{Key}\label{fig:spread_k_norm}
		\end{subfigure}

		\vspace{2pt}

		\begin{subfigure}[t]{0.49\linewidth}
			\includegraphics[width=\linewidth]{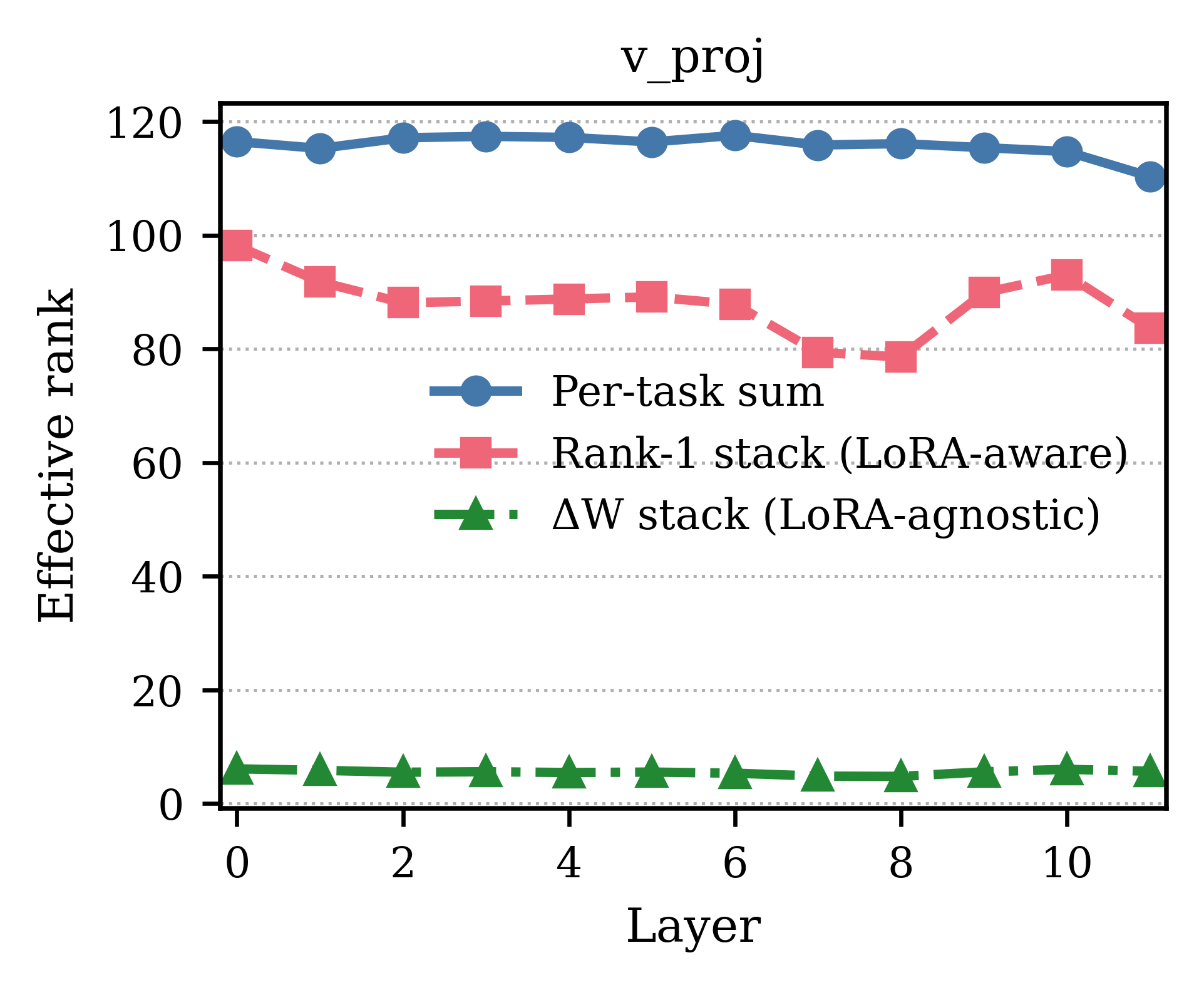}
			\subcaption{Value}\label{fig:spread_v_norm}
		\end{subfigure}\hfill
		\begin{subfigure}[t]{0.49\linewidth}
			\includegraphics[width=\linewidth]{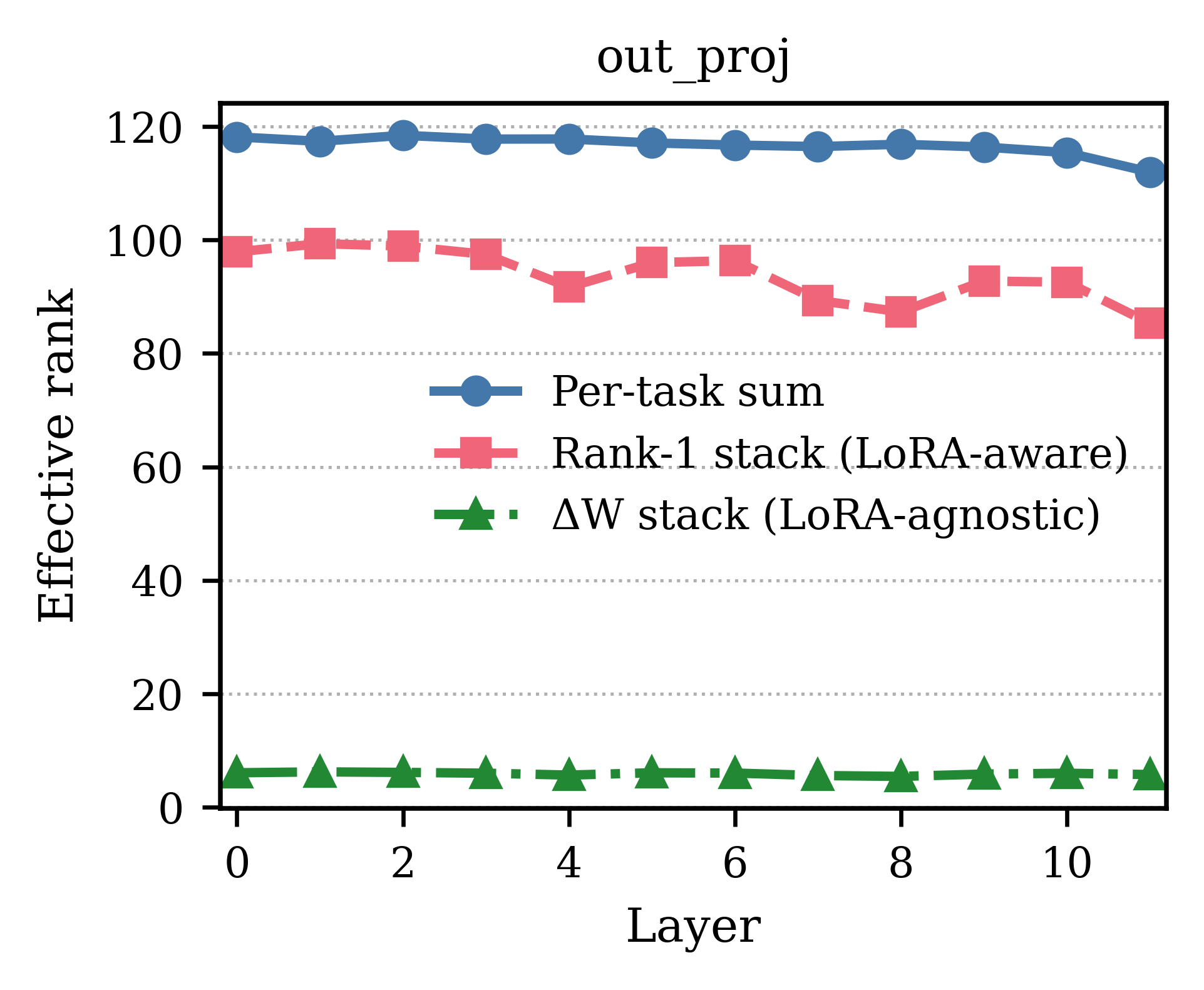}
			\subcaption{Output}\label{fig:spread_o_norm}
		\end{subfigure}
	\end{minipage}

	\caption{\textbf{Effective Rank} across layers for each module. Higher curves indicate broader subspace coverage. The gap between \emph{$\Delta \bm W$ stack} and \emph{Rank-1 stack} reflects merge-induced collapse.}
	\label{fig:spread_raw}
	\vspace{-10pt}
\end{figure*}

\subsection{Training Details}
\label{append:finetune}

For the vision tasks, we fine-tuned CLIP~\cite{radford2021learning} with a ViT-B/32 backbone.
All models were optimized using AdamW~\cite{loshchilov2018decoupled} with a learning rate of 3e-4 and a cosine learning rate scheduler~\cite{loshchilov2017sgdr} with a weight decay of 1e-4.

For sequence classification, we fine-tuned LLaMA-3~\cite{dubey2024llama} with 8B parameters, following the setup of \cite{stoica2025knots}.
Optimization used AdamW with a learning rate of 3e-5 and a linear scheduler.
A linear classification head with three outputs was added for natural language inference (entailment, contradiction, neutral), and in binary settings predictions for the unused class were disregarded.

To enable efficient adaptation, we applied LoRA~\cite{hu2022lora} to both CLIP and LLaMA.
Specifically, low-rank adapters were attached to the query, key, value, and output projection matrices of each transformer block.
Unless otherwise stated, the default LoRA hyperparameters were: rank = 16, $\alpha$ = 16, and dropout = 0.1.

\begin{figure*}[t]
	\centering
	\begin{subfigure}[t]{0.24\textwidth}
		\centering
		\includegraphics[width=\linewidth]{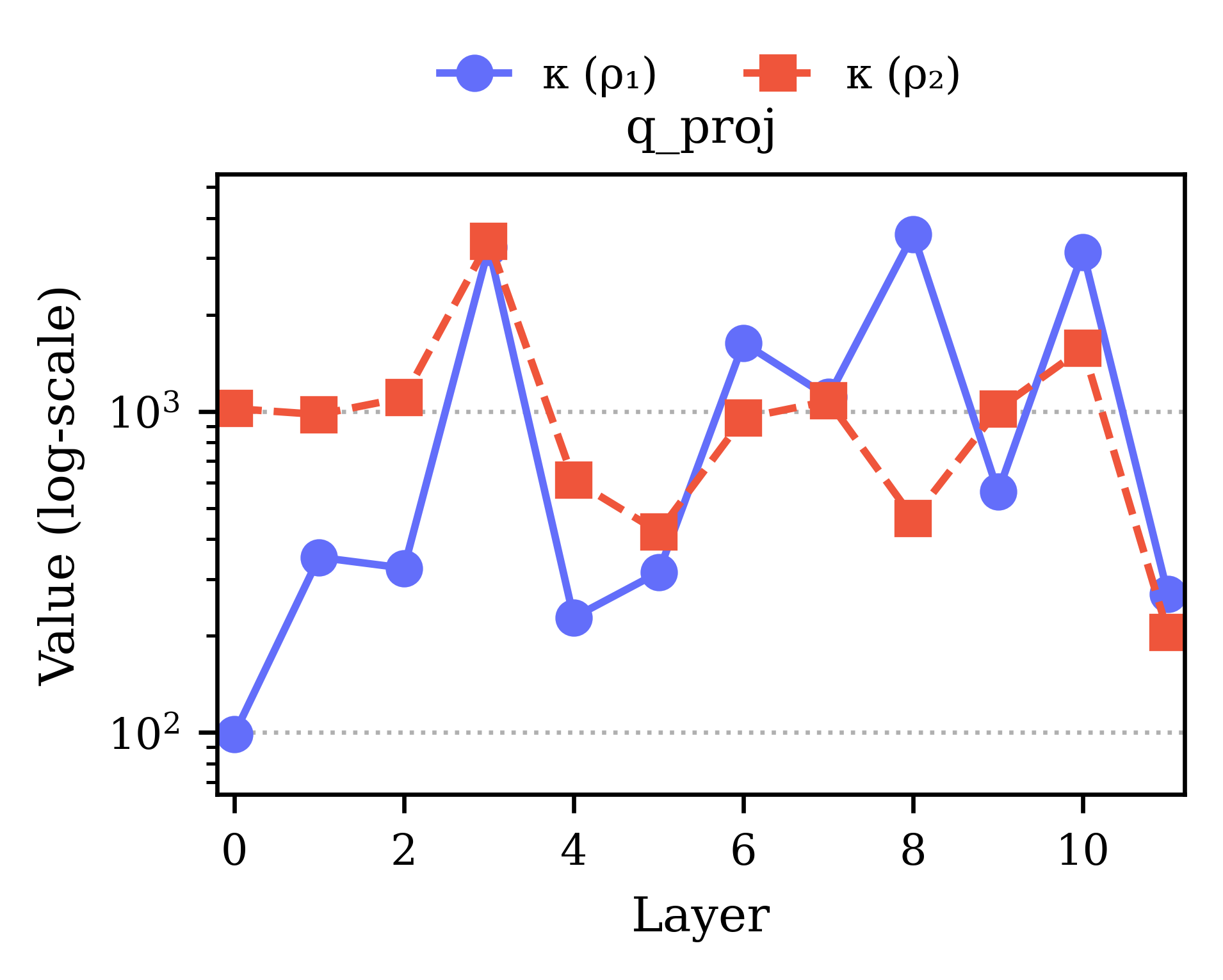}
		\caption{\texttt{q\_proj}}
	\end{subfigure}\hfill
	\begin{subfigure}[t]{0.24\textwidth}
		\centering
		\includegraphics[width=\linewidth]{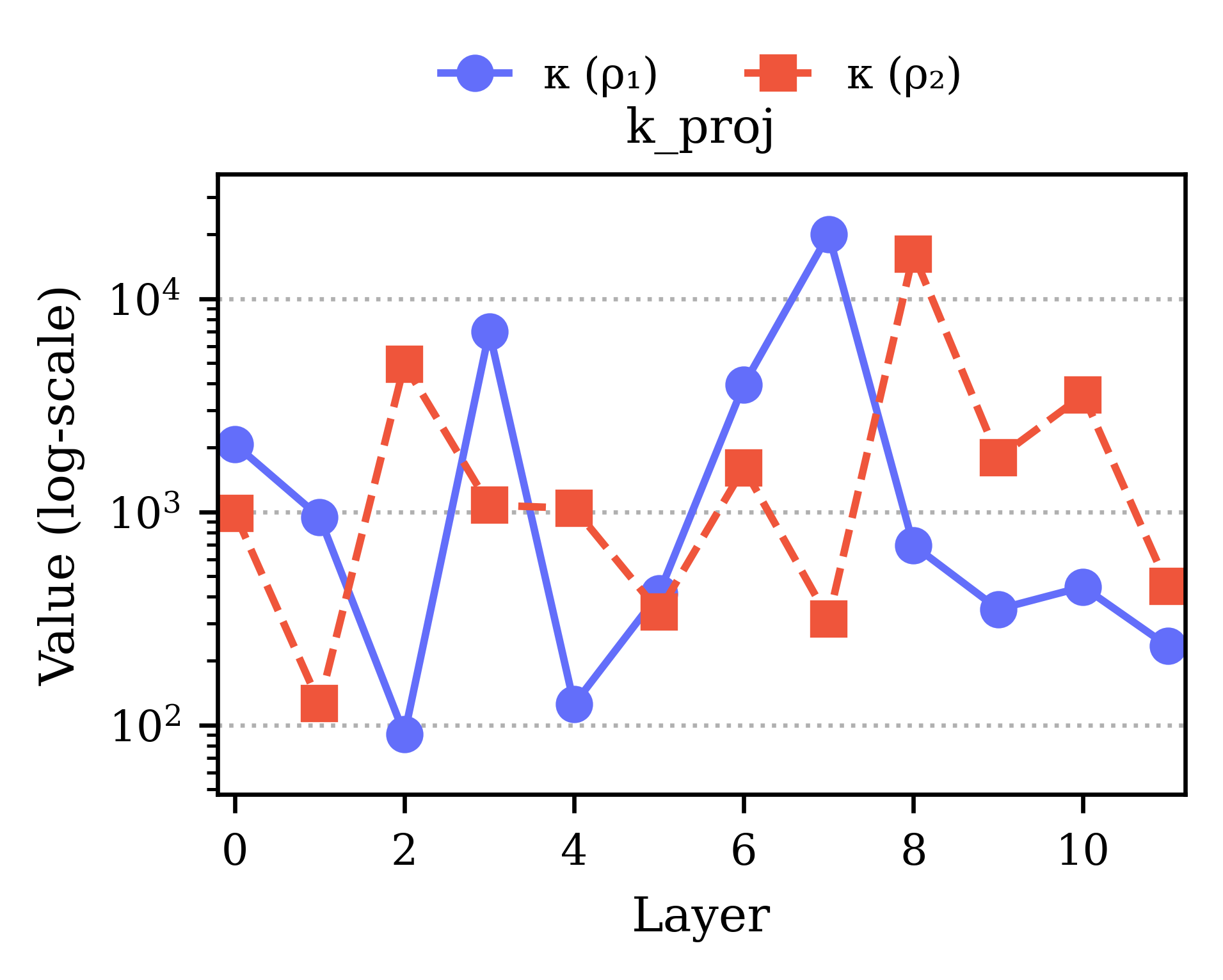}
		\caption{\texttt{k\_proj}}
	\end{subfigure}\hfill
	\begin{subfigure}[t]{0.24\textwidth}
		\centering
		\includegraphics[width=\linewidth]{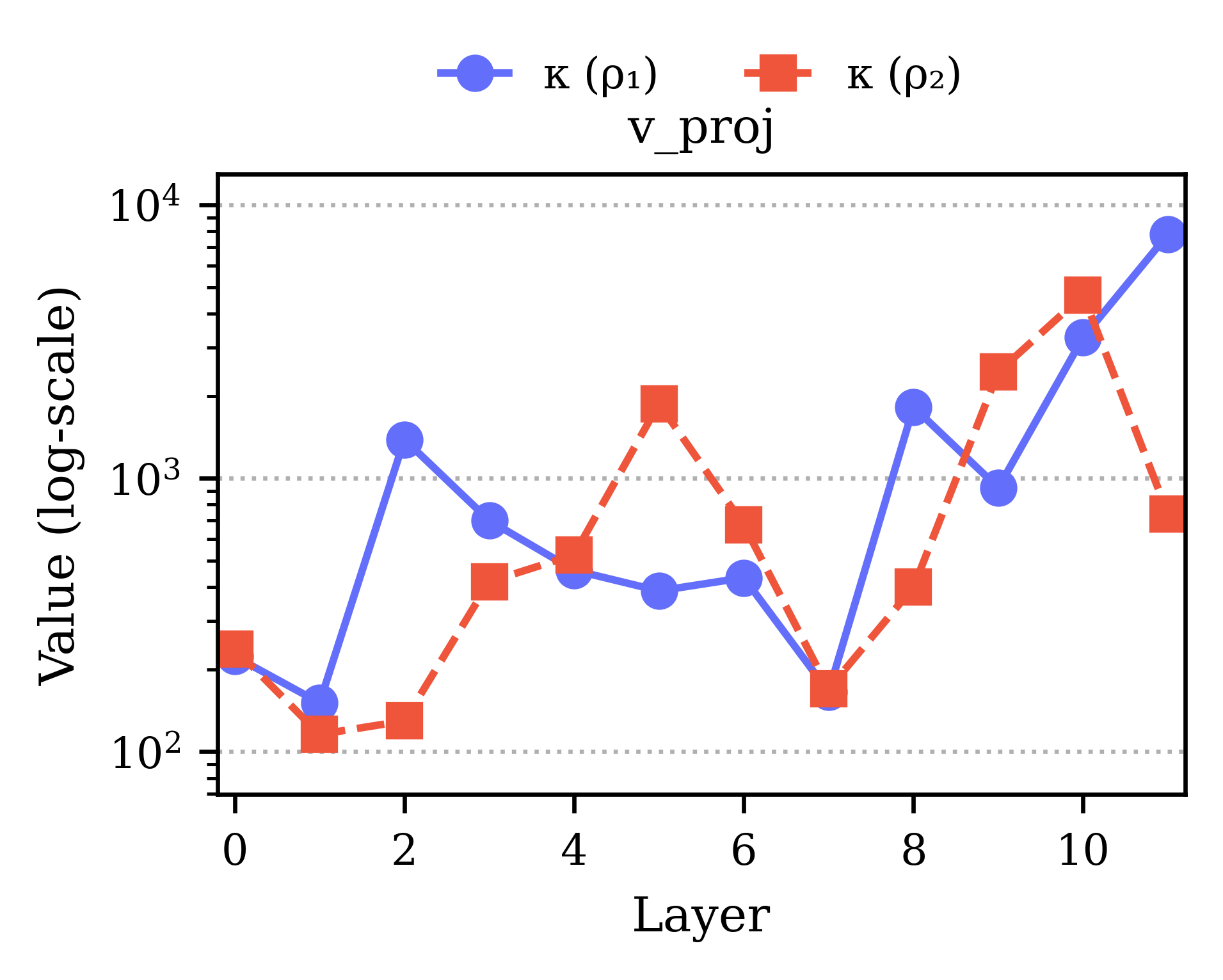}
		\caption{\texttt{v\_proj}}
	\end{subfigure}\hfill
	\begin{subfigure}[t]{0.24\textwidth}
		\centering
		\includegraphics[width=\linewidth]{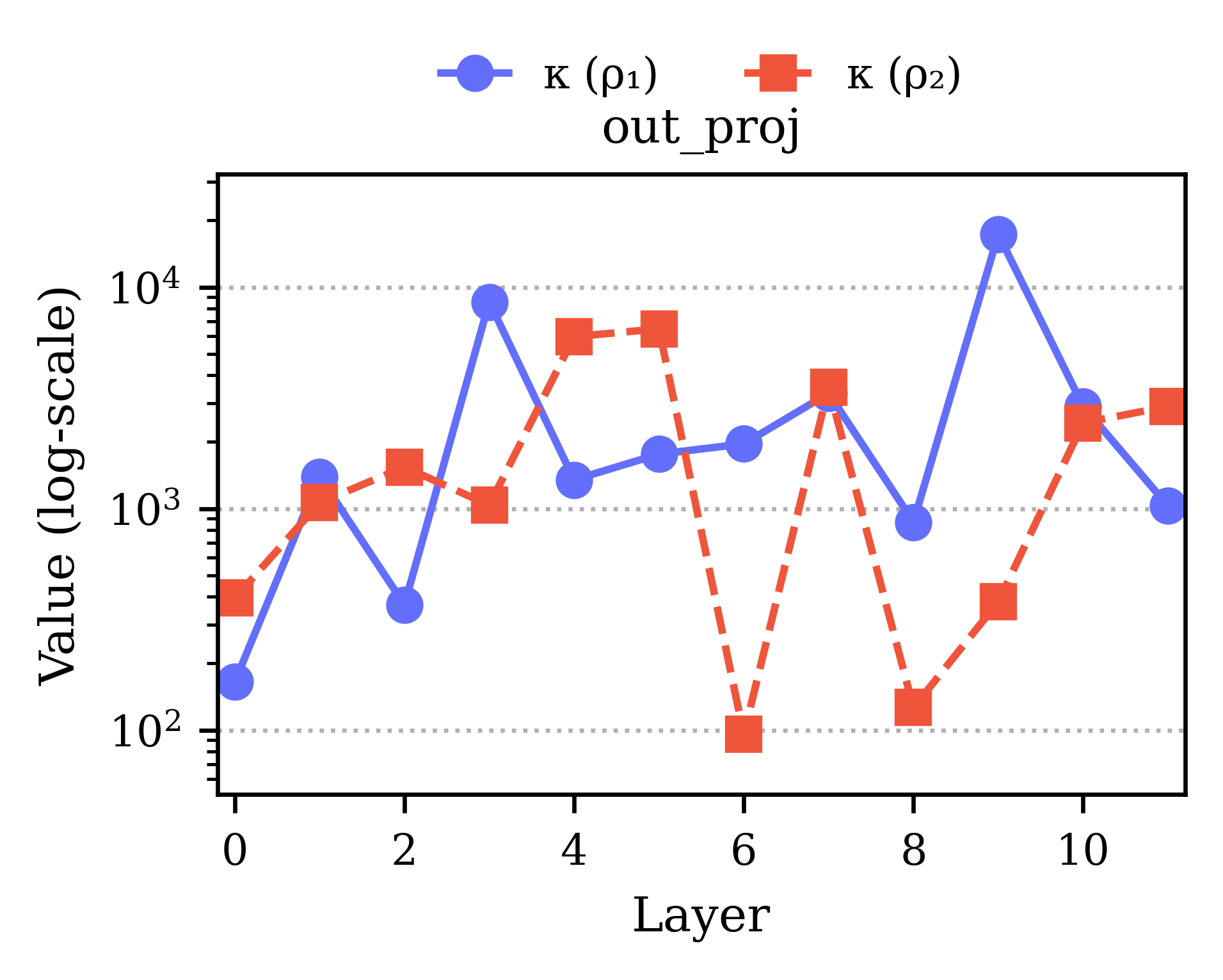}
		\caption{\texttt{out\_proj}}
	\end{subfigure}
	\caption{\textbf{Condition Number Anisotropy \(\boldsymbol{\kappa}\) (RAW Basis).}
		Layer-wise condition number \(\kappa(\bm\rho)\) per module under the non-orthogonal LoRA basis (RAW).
		Larger \(\kappa\) indicates stronger \emph{within-preference} directional concentration of loss sensitivity.
		(Here, \(\bm\rho_1\) uniformly weights all tasks, whereas \(\bm\rho_2\) assigns all weight to a single task (one-hot).)}
	\label{fig:kappa-lines-raw}
	\vspace{-10pt}
\end{figure*}
\begin{figure*}[t]
	\centering
	\begin{subfigure}[t]{0.24\textwidth}
		\centering
		\includegraphics[width=\linewidth]{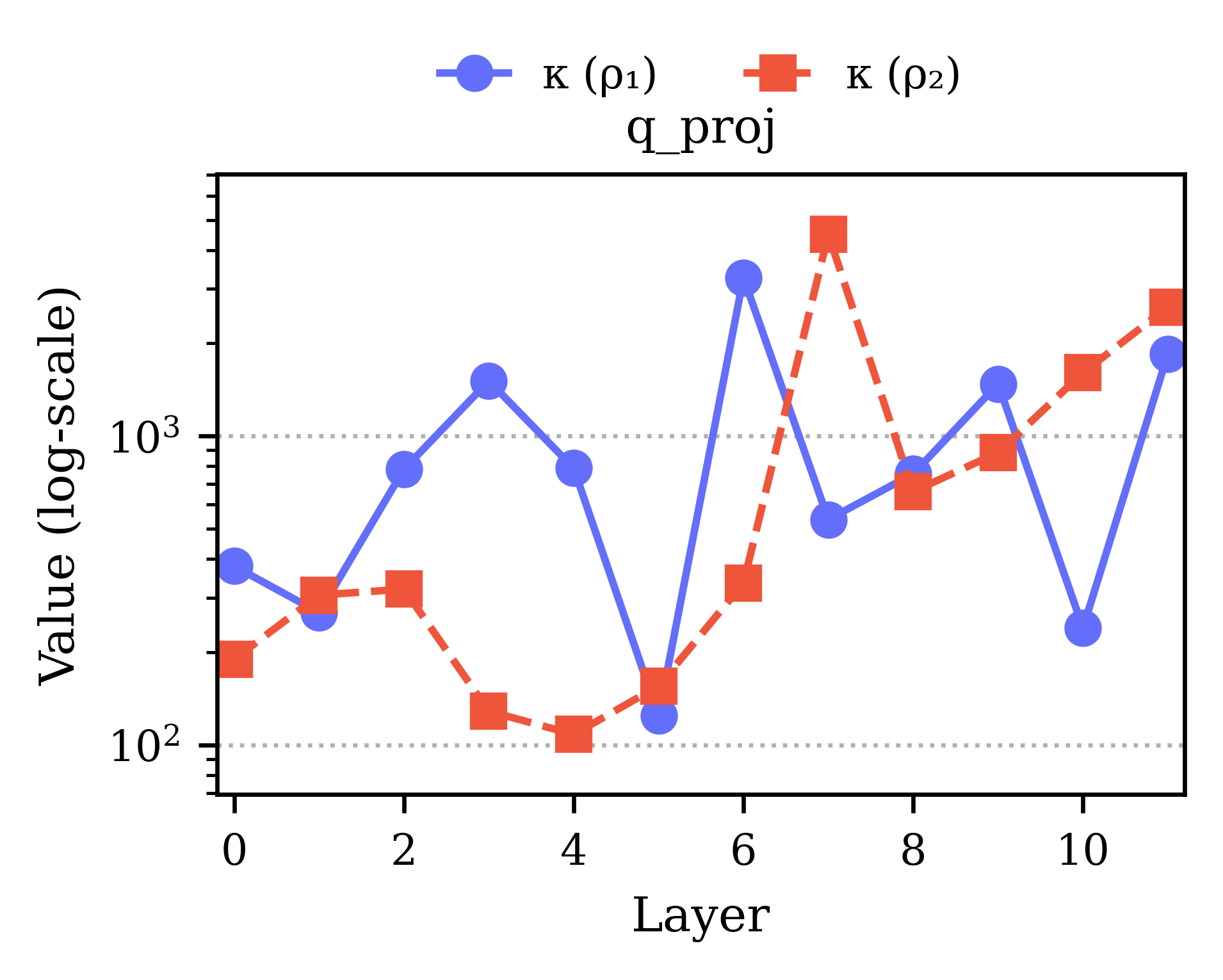}
		\caption{\texttt{q\_proj}}
	\end{subfigure}\hfill
	\begin{subfigure}[t]{0.24\textwidth}
		\centering
		\includegraphics[width=\linewidth]{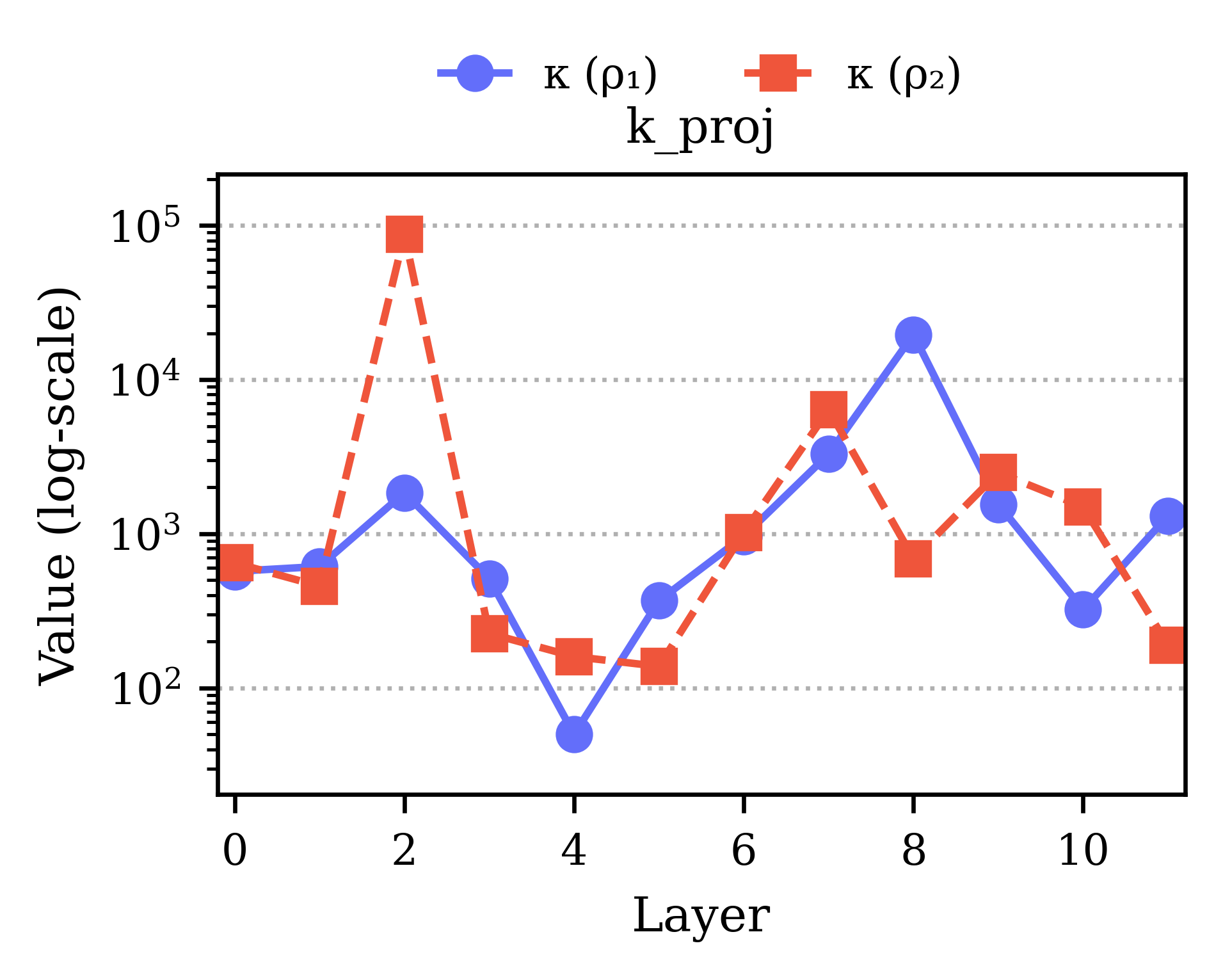}
		\caption{\texttt{k\_proj}}
	\end{subfigure}\hfill
	\begin{subfigure}[t]{0.24\textwidth}
		\centering
		\includegraphics[width=\linewidth]{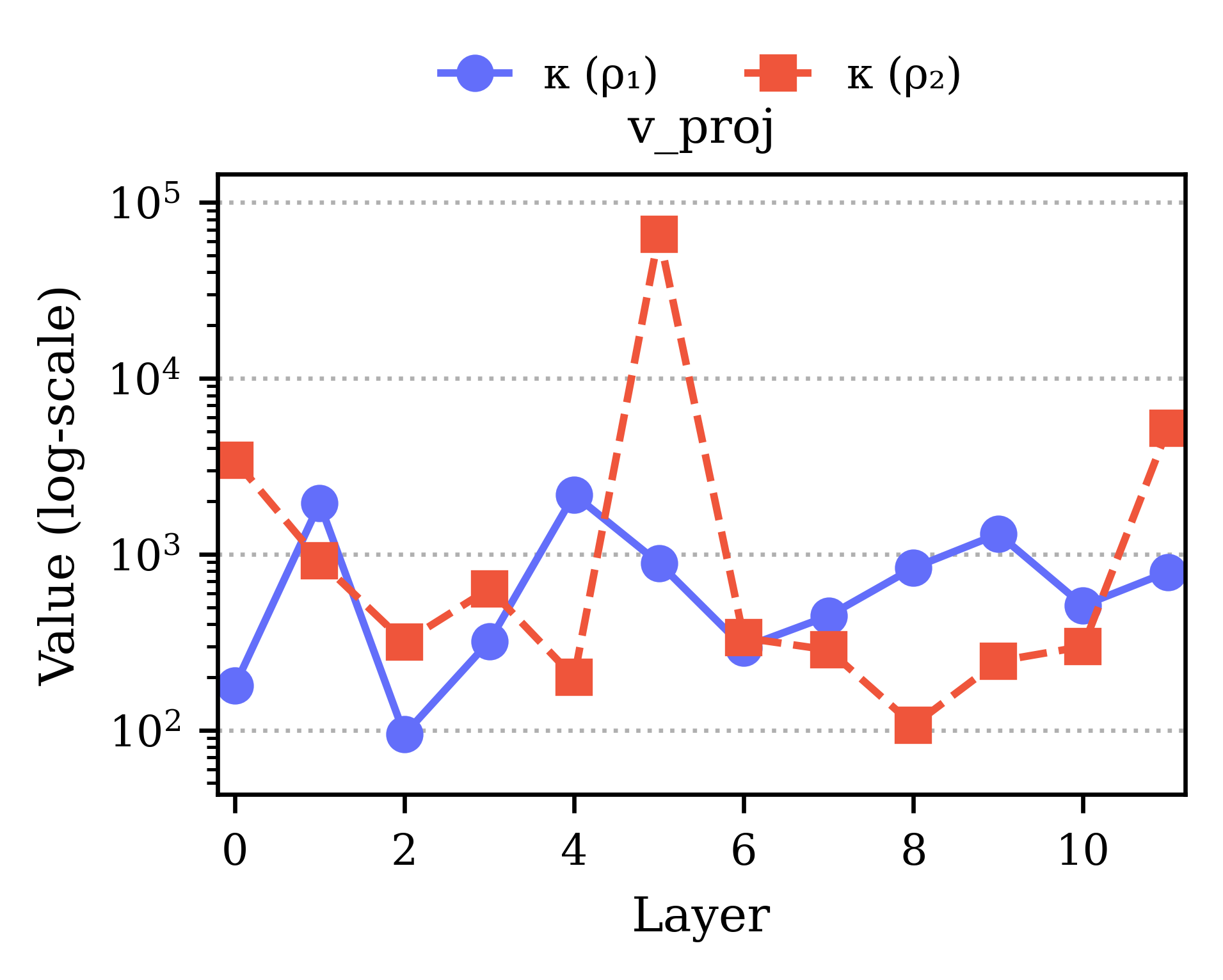}
		\caption{\texttt{v\_proj}}
	\end{subfigure}\hfill
	\begin{subfigure}[t]{0.24\textwidth}
		\centering
		\includegraphics[width=\linewidth]{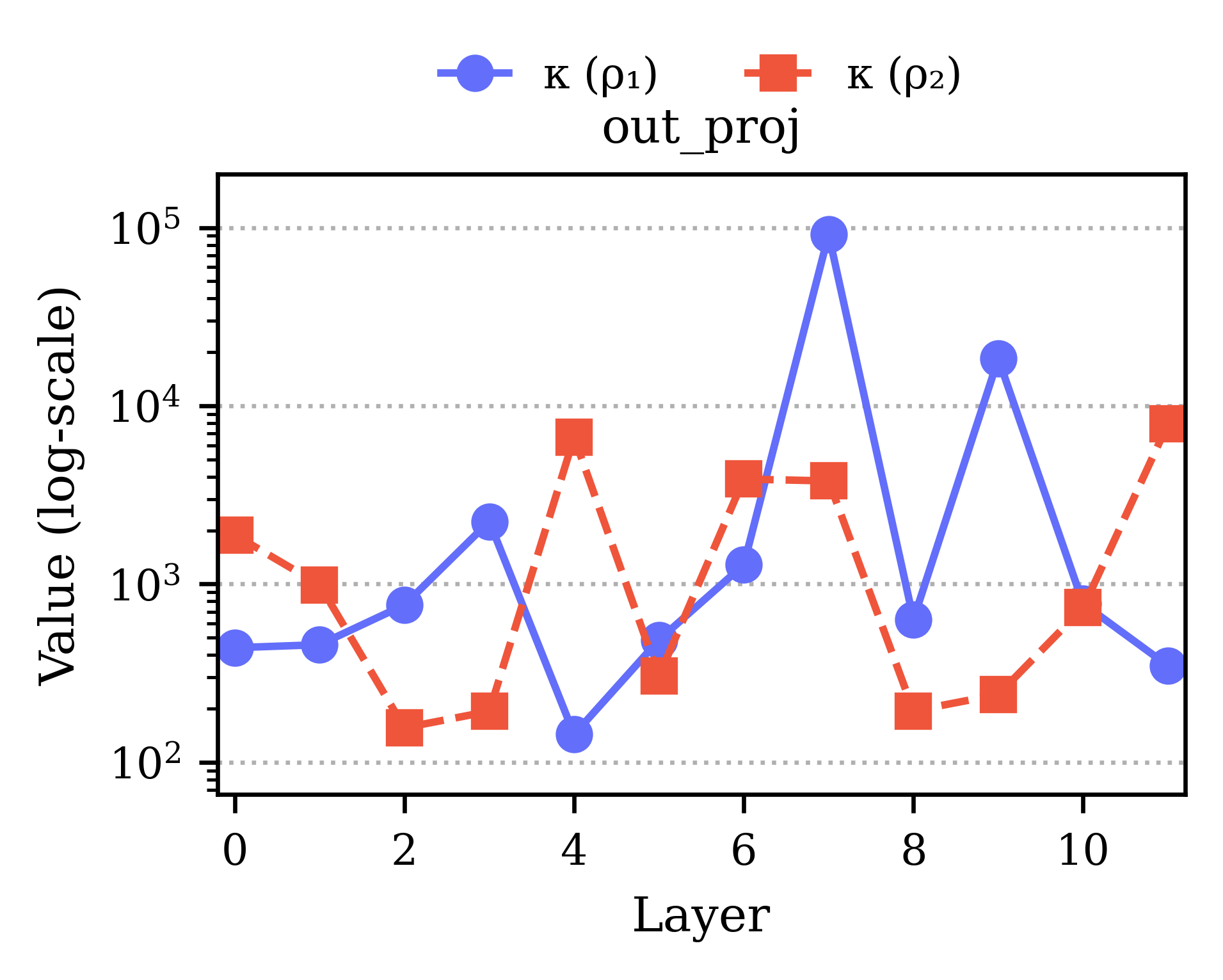}
		\caption{\texttt{out\_proj}}
	\end{subfigure}
	\caption{\textbf{Condition Number Anisotropy \(\boldsymbol{\kappa}\) (SVD Basis).}
		Layer-wise condition number \(\kappa(\bm\rho)\) per module under the orthogonal LoRA basis (SVD).
		Larger \(\kappa\) indicates stronger \emph{within-preference} directional concentration of loss sensitivity.
		(Here, \(\bm\rho_1\) uniformly weights all tasks, whereas \(\bm\rho_2\) assigns all weight to a single task (one-hot).)}
	\label{fig:kappa-lines-svd}
\end{figure*}


\subsection{Implementation Details}
\label{append:imple}

In this paper, we compare our approach against ten baselines, following the experimental protocols reported in their original works.
For TA~\cite{ilharco2022editing}, we tune the scaling coefficient of task vectors across the range of $\left[ 0.1, 0.2, \dots 1.0 \right]$.
For TIES~\cite{yadav2023ties} and DARE-TIES~\cite{yu2024language}, both the scaling coefficient $\lambda$ and pruning coefficient $\eta$ were selected via grid search. Specifically, $\lambda$ was varied over $\left[0.8, 0.9, \dots, 1.8\right]$, while $\eta$ was explored over $\left[0.1, 0.2, \dots, 0.9\right]$.
For AdaMerging~\cite{yang2023adamerging}, all coefficients $\{\lambda_k^l\}_{n=1, l=1}^{N, L}$ (where $N$ is the number of tasks and $L$ the number of layers) are initialized to 0.3 before optimization with the entropy surrogate.
For KnOTS~\cite{stoica2025knots}, we evaluate two variants: KnOTS-TIES and KnOTS-DARE-TIES, which combine KnOTS with TIES~\cite{yadav2023ties} and DARE~\cite{yu2024language}, respectively.
For LoRA-LEGO~\cite{zhao2025loralego}, we reimplement the algorithm following the original paper.
We apply $k$-means clustering to LoRA Minimal Semantic Units (MSUs), where each MSU $u=[a, b]$ consists of a row $a$ from the LoRA down-projection matrix $A$ and the corresponding column $b$ from the up-projection matrix $B$.
We experiment with cluster sizes $k \in \{8, 16, 32, 64, 128\}$, and reproduce the dual scaling strategies: parameter reweighting and output reweighting.
Parameter reweighting rescales each cluster centroid $\mu$ to match the average norm of its members, $\mu^{\prime} = \tfrac{\tfrac{1}{p} \sum_{i=1}^p \|s_i\|}{\|\mu\|}\mu$, compensating for the reduced norm after merging, while output reweighting scales the merged LoRA by $\tfrac{\sqrt{r}}{\sqrt{k}}$ (with $r$ the original LoRA rank and $k$ the merged rank) to stabilize variance. We evaluate every $k$ and both scaling strategies, reporting the best performance.
We further include the LoRA-aware RobustMerge~\cite{zeng2025parameter} and strong vanilla-merging methods EMR-Merging~\cite{huang2024emr}, FR-Merging~\cite{zheng2025free}, and Iso-CTS~\cite{marczak2025no}, and report their results in \cref{append:additional_exp}.

\begin{figure*}[t]
	\vspace{-10pt}
	\centering
	\begin{subfigure}[t]{0.47\textwidth}
		\centering
		\includegraphics[width=\linewidth]{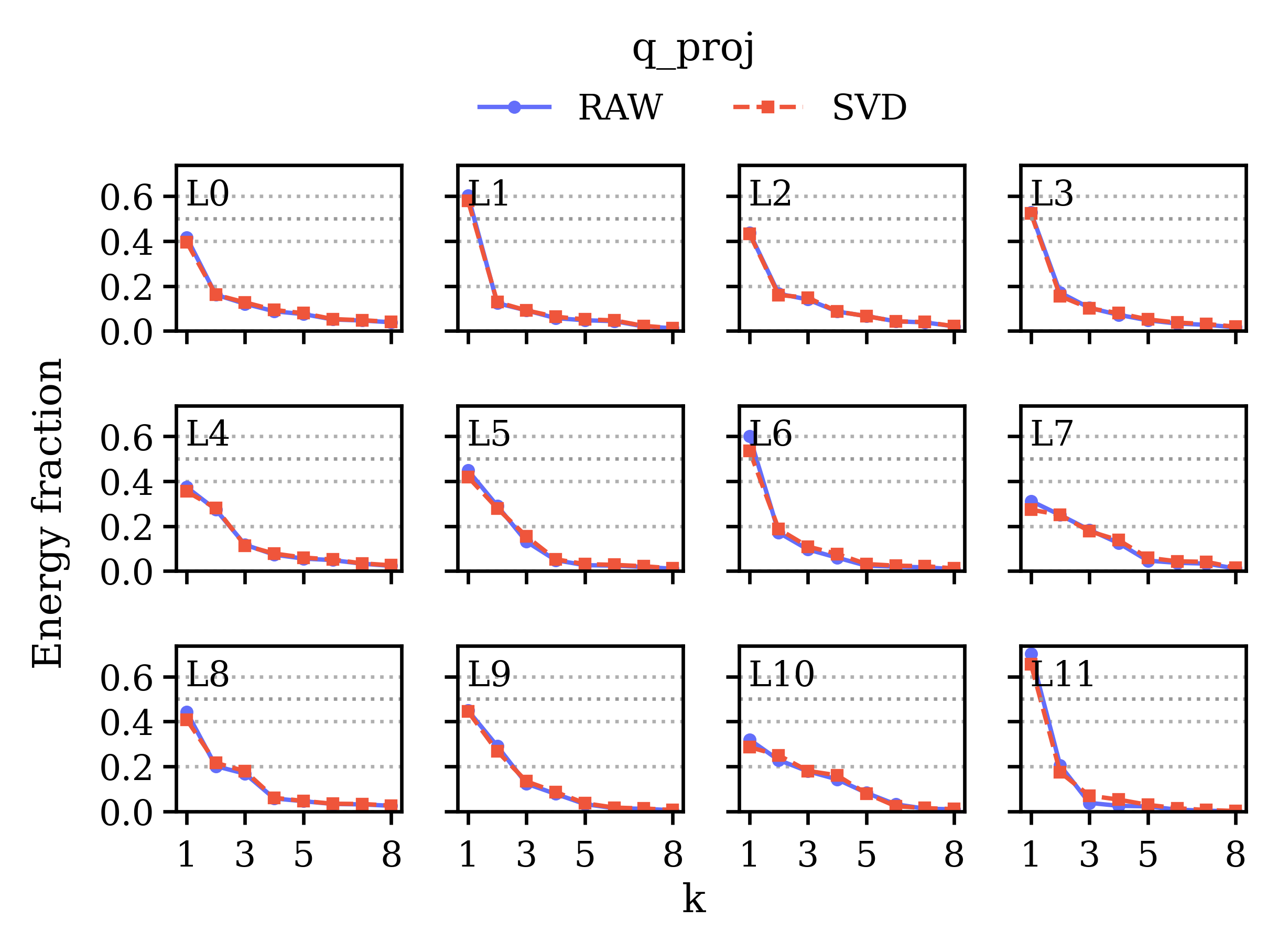}
		\caption{\texttt{q\_proj}}
	\end{subfigure}\hfill
	\begin{subfigure}[t]{0.47\textwidth}
		\centering
		\includegraphics[width=\linewidth]{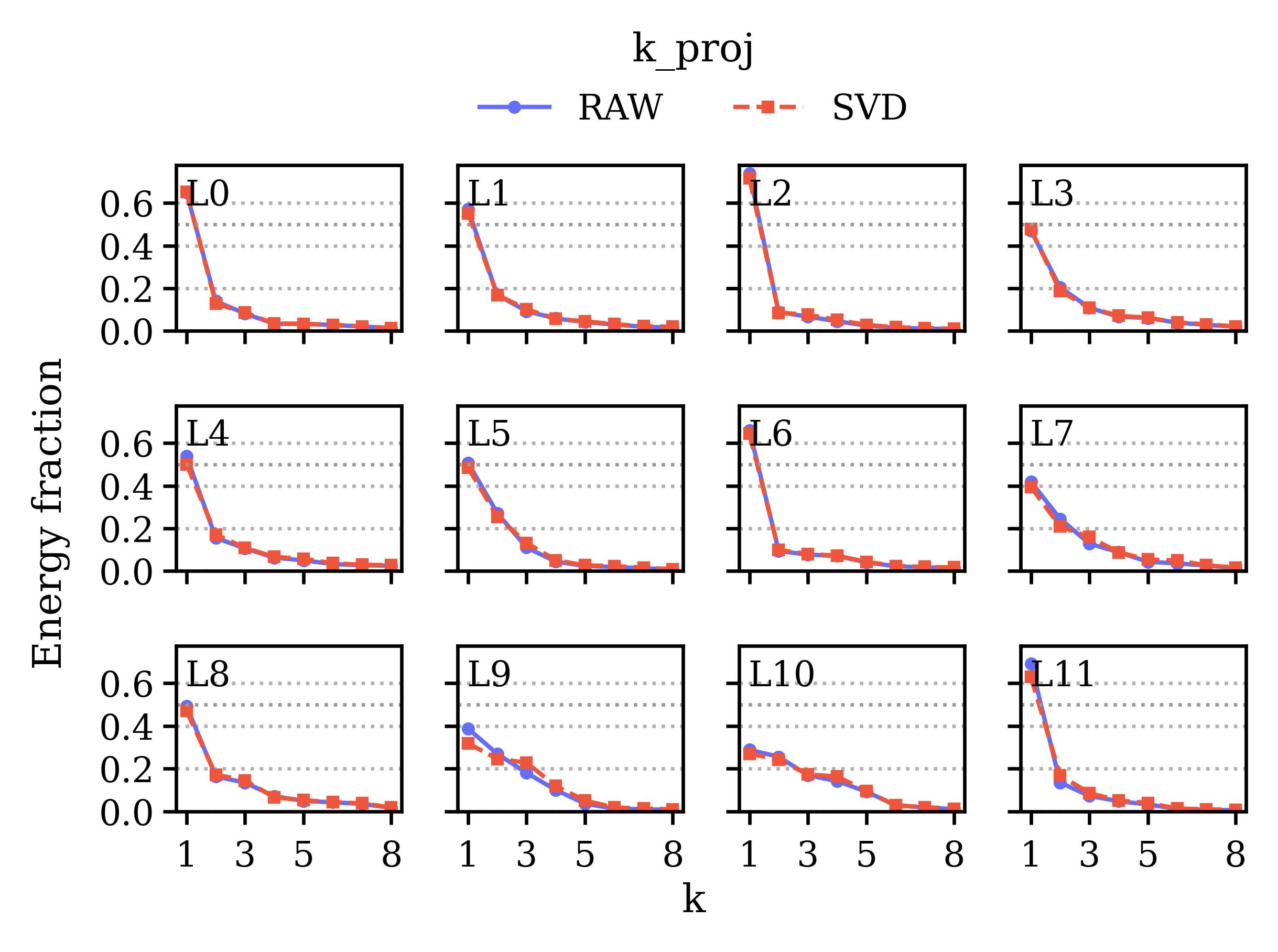}
		\caption{\texttt{k\_proj}}
	\end{subfigure}
	\vspace{4pt}
	\begin{subfigure}[t]{0.47\textwidth}
		\centering
		\includegraphics[width=\linewidth]{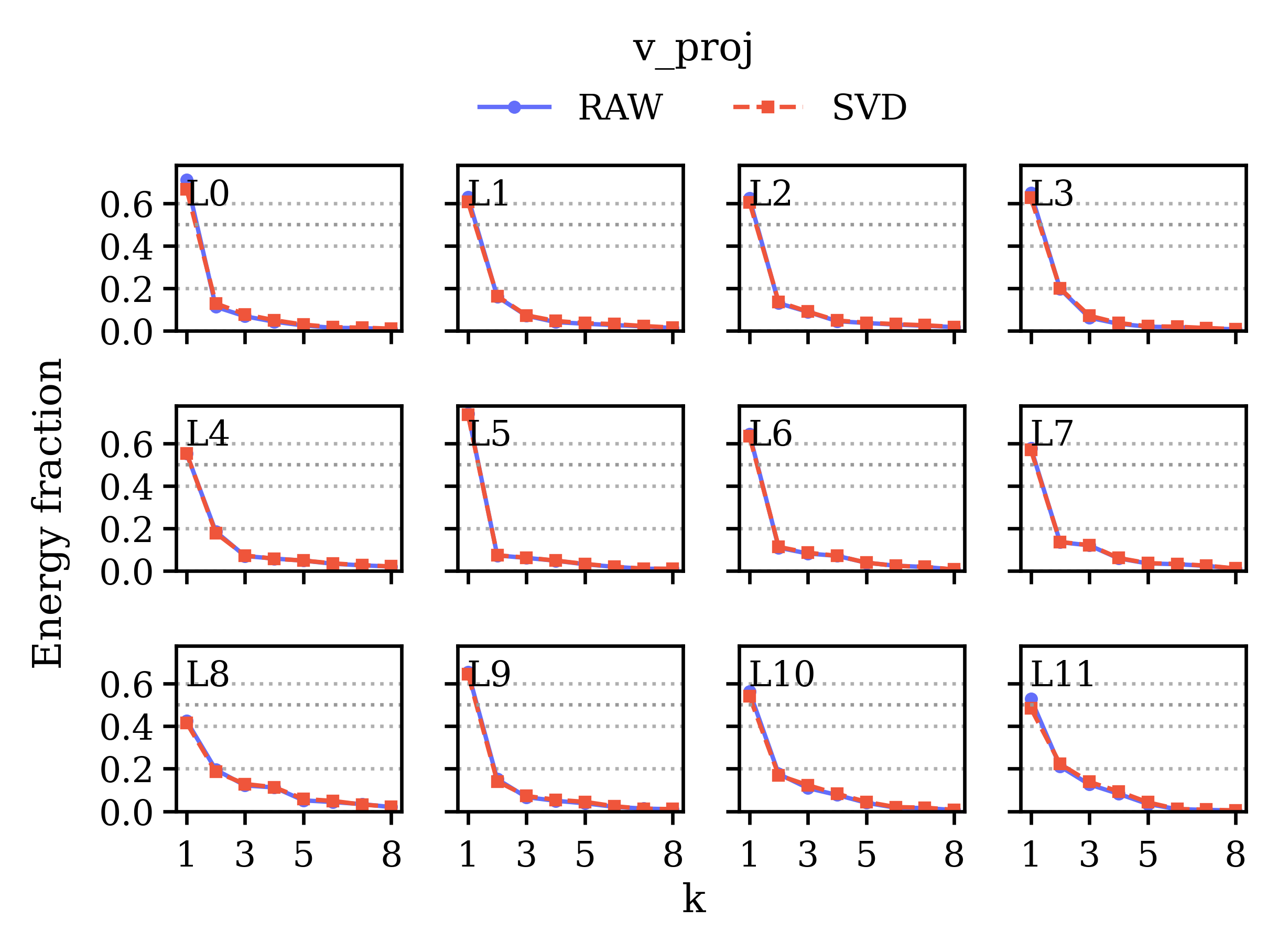}
		\caption{\texttt{v\_proj}}
	\end{subfigure}\hfill
	\begin{subfigure}[t]{0.47\textwidth}
		\centering
		\includegraphics[width=\linewidth]{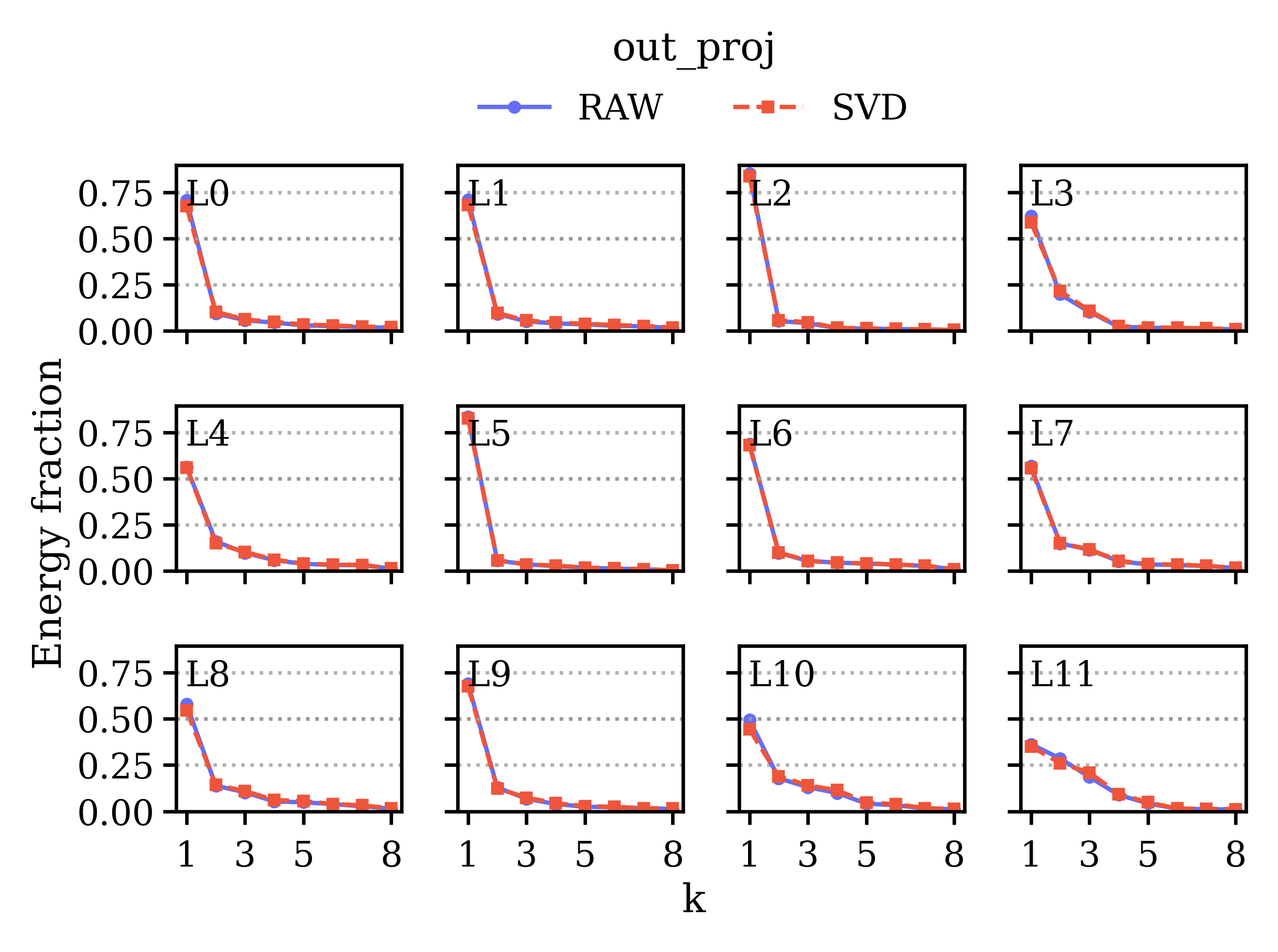}
		\caption{\texttt{out\_proj}}
	\end{subfigure}
	\vspace{-5pt}
	\caption{\textbf{Layer-wise Directional-Sensitivity Spectra (RAW vs.\ SVD).}
		For each module and layer, we form the task-direction sensitivity matrix with entries
		\(J_{i,k}=\langle\nabla f_i(W),\,S_k\rangle\) and plot the scree curve of its singular values as
		energy fractions \(\sigma_k^2/\!\sum_j \sigma_j^2\) versus component index \(k\).
		“RAW” stacks rank-1 LoRA factors directly, while “SVD” uses a shared orthonormal basis. The two spectra are overlaid in each panel.
		A larger leading mass indicates stronger concentration of sensitivity into a few modes, whereas flatter spectra indicate a more distributed use of LoRA directions.}
	\label{fig:grid-scree-raw-vs-svd}
	\vspace{-5pt}
\end{figure*}

\subsection{Per-Task and Joint-Task Evaluation}
\label{append:joint_settings}
In addition to the standard per-task benchmark, we also adopt the “joint-task” evaluation introduced by \citet{stoica2025knots}. Unlike the per-task regime, where each dataset is treated independently, the joint-task aggregates the inputs and labels from all eight vision benchmarks into a single evaluation pool. After combining all label sets, duplicate classes are removed (for example, MNIST~\cite{lecun1998mnist} and SVHN~\cite{yuval2011reading} share the same digit categories), resulting in 748 unique labels across the combined benchmark. This setup is particularly demanding because models must not only generalize within a task, but also discriminate among labels that appear across different datasets. In some cases, labels are semantically close or hierarchical, such as “islet” in SUN397~\cite{xiao2010sun} and “island” in RESISC45~\cite{cheng2017remote}, which increases the difficulty of classification. To account for such ambiguity, evaluation is reported using Hits@$k$ metrics, where Hits@1 corresponds to top-1 accuracy and higher $k$ values allow partial credit when the correct label appears among the top-$k$ predictions. Joint-task protocol serves as an important benchmark for assessing the generality of merged models, as it directly tests whether a single model can operate across the diverse label space of multiple datasets.

\section{Additional Analysis: Subspace Coverage}
\label{append:subspace_coverage}

This appendix presents the exact formulations used in the main paper’s subspace-coverage analysis.
We form all stacks by vectorizing matrices and placing one vector per row.
Let \(\bm W_0\in\mathbb{R}^{d\times m}\) be the pre-trained weight.
For task \(i\in\{1,\dots,N\}\) with LoRA rank \(r_i\), write \(\Delta \bm W_i=\sum_{j=1}^{r_i} \vb_{ij}\va_{ij}^{\top}\) with \(\vb_{ij}\in\mathbb{R}^{d}\) and \(\va_{ij}\in\mathbb{R}^{m}\).
Let \(\operatorname{vec}(\cdot):\mathbb{R}^{d\times m}\!\to\!\mathbb{R}^{dm}\) denote matrix vectorization.
For any matrix \(\bm X\), let \(\{\sigma_k\}_{k=1}^{R_{\bm X}}\) be its nonzero singular values with \(R_{\bm X}=\mathrm{rank}(\bm X)\) and define
\[
	p_k=\frac{\sigma_k^2}{\sum_{j=1}^{R_{\bm X}}\sigma_j^2},
	\qquad
	\mathrm{erank}(\bm X)=\exp\!\Big(-\sum_{k=1}^{R_{\bm X}} p_k \log p_k\Big).
\]

The following entries are used in our analysis. For each entity, we define the stacked matrix and report its subspace coverage via effective rank.
\begin{itemize}[leftmargin=*, itemsep=0.25em]
	\item \textbf{Per-task sum} — compute coverage per task and then aggregate across tasks:
	      \[
		      \bm X_i=
		      \begin{bmatrix}
			      \operatorname{vec}\!\big(\vb_{i1}\va_{i1}^{\top}\big) \\
			      \vdots                                                \\
			      \operatorname{vec}\!\big(\vb_{ir_i}\va_{ir_i}^{\top}\big)
		      \end{bmatrix}
		      \in\mathbb{R}^{r_i\times dm},
	      \]
	      \[
		      \mathrm{PerTaskSum}=\sum_{i=1}^{N}\mathrm{erank}(\bm X_i).
	      \]

	\item \textbf{\(\Delta \bm W\) stack (LoRA-agnostic)} — coverage of task updates when the LoRA rank-1 structure is ignored. We stack the task updates and use $\mathrm{erank}(\bm X_{\text{agnostic}})$ with
	      \[
		      \bm X_{\text{agnostic}}=
		      \begin{bmatrix}
			      \operatorname{vec}\!\big(\Delta \bm W_1\big) \\
			      \vdots                                       \\
			      \operatorname{vec}\!\big(\Delta \bm W_N\big)
		      \end{bmatrix}
		      \in\mathbb{R}^{N\times dm},
	      \]

	\item \textbf{Rank-1 stack (LoRA-aware)} — coverage when rank-1 directions from all adapters are retained. We stack all rank-1 directions and use $\mathrm{erank}(\bm X_{\text{aware}})$ with
	      \[
		      \bm X_{\text{aware}}=
		      \begin{bmatrix}
			      \operatorname{vec}\!\big(\vb_{11}\va_{11}^{\top}\big) \\
			      \vdots                                                \\
			      \operatorname{vec}\!\big(\vb_{N r_N}\va_{N r_N}^{\top}\big)
		      \end{bmatrix}
		      \in\mathbb{R}^{(\sum_i r_i)\times dm},
	      \]
\end{itemize}

\paragraph{Other modules.}
Supplementary results in \cref{fig:spread_raw} show the same qualitative pattern across additional modules (e.g., Attention and MLP layers).
The ordering \(\textit{Per-task sum} \ge \textit{Rank-1 stack (LoRA-aware)} \ge \Delta \bm W\ \textit{stack (LoRA-agnostic)}\) consistently holds, and the LoRA-aware stack retains roughly \(60\%\!\sim\!70\%\) of the per-task sum.
The gap between LoRA-aware and LoRA-agnostic reflects subspace collapse under interpolation-based merging.

\section{Additional Analysis: Anisotropy}
\label{append:anisotropy}

\paragraph{Within-Preference Anisotropy.}
We measure how unevenly a single preference $\bm\rho\!\in\!\Delta_{N-1}$ concentrates loss sensitivity onto LoRA directions.
Let $g(\bm\rho;\bm W)=\sum_{i=1}^{N}\rho_i\,\nabla f_i(\bm W)$ be the scalarized gradient at parameters $\bm W$ obtained by
\emph{averaging} task LoRA updates and then scaling by $0.3$ (as Task Arithmetic~\cite{ilharco2022editing}).
Given a basis of LoRA directions $\{\bm S_k\}_{k=1}^{K}$ (either the RAW stack of rank-1 factors or the shared SVD-orthonormal basis), define the directional sensitivities with \emph{condition-number}.
Larger $\kappa$ indicates stronger \emph{within-preference} concentration of sensitivity onto a few directions (greater anisotropy), while smaller $\kappa$ indicates a more balanced use of the LoRA subspace. We report layer-wise $\kappa(\bm\rho;\bm W)$ for each module (\texttt{q\_proj}, \texttt{k\_proj}, \texttt{v\_proj}, \texttt{out\_proj}) under both RAW and SVD bases in \cref{fig:kappa-lines-raw} and \cref{fig:kappa-lines-svd}. Here, $\bm\rho_1$ uniformly weights all tasks, whereas $\bm\rho_2$ assigns all weight to a single task (one-hot).

\Cref{fig:grid-scree-raw-vs-svd} compares, for each module and layer, the singular-value energy distribution of the task-direction sensitivity matrix under the RAW (rank-1 factor stack) and SVD (shared orthonormal) bases. Each curve shows the energy fraction \(\sigma_k^2/\sum_j\sigma_j^2\) versus component index \(k\). A larger leading mass indicates stronger anisotropy with sensitivity concentrated in a few directions, while flatter curves indicate a more balanced use of the LoRA subspace.

\paragraph{Directional-Sensitivity Misalignment in the SVD Basis.}
Following Sec.~\textcolor{cvprblue}{3.3},
we conduct the same analysis with directional sensitivities projected onto the rank-1 LoRA directions obtained by SVD orthogonalization.
We then compute a misalignment index $\xi(\bm\rho_1,\bm\rho_2)$ between the uniform and one-hot preferences across the LoRA layers.
As shown in \cref{fig:xi-heatmap-raw}, the resulting heatmap closely mirrors the raw-basis result in
Fig.~\textcolor{cvprblue}{2}: substantial layer- and module-wise misalignment persists, indicating preference-dependent sensitive directions that remain even with SVD basis.

\begin{figure}[ht]
	\centering
	\vspace{-5pt}
	\includegraphics[width=0.90\linewidth]{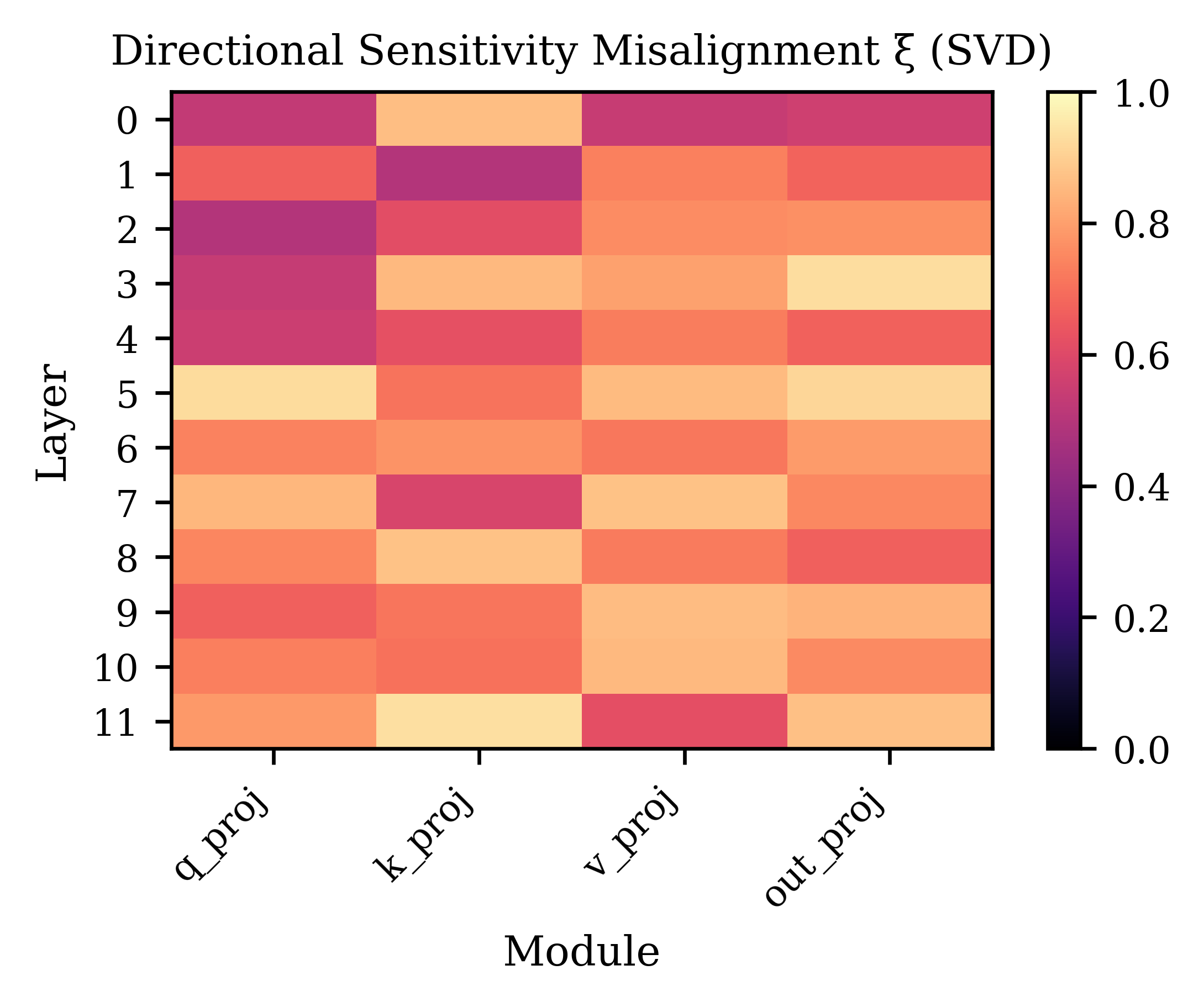}
	\vspace{-10pt}
	\caption{\textbf{Directional-sensitivity misalignment \(\boldsymbol{\xi(\rho_1,\rho_2)}\)} in the SVD basis.}
	\vspace{-5pt}
	\label{fig:xi-heatmap-raw}
\end{figure}


\section{Additional Results}
\label{append:additional_exp}

\subsection{Extended Analysis of Baseline Comparisons}
In this section, we provide a more detailed analysis of the baseline comparisons presented in the main paper. We elaborate on the per-task performance metrics across both vision and natural language inference (NLI) domains, offering an extended examination of the behavior of recent vanilla and LoRA-aware merging methods.

\begin{table*}[t]
	\caption{
		To further address concerns regarding latency coming from gradient-based optimization, we experimented our method with LLaVA~\cite{liu2023improvedllava, liu2023llava} on merging six VQA experts.
	}
	\vspace{-5pt}
	\label{tab:six_vqa_benchmarks}
	\centering
	\renewcommand{\arraystretch}{0.85}
	\scriptsize
	\setlength{\tabcolsep}{10pt}
	\resizebox{0.85\linewidth}{!}{
		\begin{tabular}{lcccccc>{\columncolor[gray]{0.9}}c}
			\toprule
			\multirow{2}{*}{Method}              & \multicolumn{7}{c}{Dataset}                                                                                                                            \\
			\cmidrule(lr){2-8}
			                                     & ScienceQA                                                                                       & VizWiz & IconQA & ImageNet & ChartQA & DocVQA & Avg  \\
			\midrule
			                                     & \multicolumn{7}{c}{\textit{Per-task absolute accuracies (\%)}}                                                                                         \\ \cmidrule(lr){2-8}
			Finetuned                            & 86.6                                                                                            & 69.7   & 65.2   & 96.0     & 39.3    & 41.2   & 66.4 \\
			\midrule
			                                     & \multicolumn{7}{c}{\textit{Per-task accuracies of merged models, normalized to finetuned (\%)}}                                                        \\ \cmidrule(lr){2-8}
			\addlinespace[2pt]
			TA~\cite{ilharco2022editing}         & 81.6                                                                                            & 71.3   & 57.2   & 43.6     & 66.6    & 79.5   & 66.6 \\
			TIES~\cite{yadav2023ties}            & 90.9                                                                                            & 69.8   & 66.5   & 54.8     & 79.1    & 81.9   & 73.8 \\
			KnOTS-TIES~\cite{stoica2025knots}    & 87.4                                                                                            & 77.0   & 60.5   & 44.8     & 76.5    & 85.6   & 72.0 \\ \hline
			\addlinespace[2pt]
			AdaMerging~\cite{yang2023adamerging} & 87.9                                                                                            & 81.0   & 66.3   & 47.0     & 74.5    & 76.1   & 72.1 \\
			\method-Variant A                    & 88.3                                                                                            & 78.7   & 69.2   & 60.5     & 78.4    & 80.4   & 75.9 \\
			\method-Variant B                    & 89.1                                                                                            & 79.2   & 69.5   & 59.6     & 78.3    & 80.0   & 76.0 \\
			\bottomrule
		\end{tabular}
	}
\end{table*}

\begin{table*}[t]
	\centering
	\caption{
		Time and memory cost analysis.
		We group baselines into gradient-free and gradient-based methods and compare their time and memory cost on LLaMA and LLaVA. For gradient-free methods, we exclude the time spent on grid search for tuning the scaling factor on the validation set. We report only \emph{time per iteration} for gradient-free methods and all times measured in minutes (min). TIES, KnOTS-TIES, and \method-Variant B require additional computation to resolve parameter conflicts or perform SVD, so we report both their total time and, in parentheses, the separate contributions from preprocessing and either one validation run (for gradient-free methods) or one optimization run (for gradient-based methods). Memory usage is reported as VRAM in MiB.
	}
	\label{tab:cost}
	\resizebox{0.60\linewidth}{!}{
		\begin{tabular}{lcccc}
			\toprule[1.2pt]
			                                       & \multicolumn{2}{c}{LLaMA-3-8B} & \multicolumn{2}{c}{LLaVA-1.5-7B}                                 \\
			\cmidrule(lr){2-3}\cmidrule(lr){4-5}
			Method                                 & Time (min)                     & VRAM (MiB)                       & Time (min)       & VRAM (MiB) \\
			\midrule
			TA~\cite{ilharco2022editing}           & 3.6                            & 16{,}150                         & 10.6             & 14{,}814   \\
			TIES~\cite{yadav2023ties}              & 8.1 (4.5+3.6)                  & 16{,}150                         & 18.3 (7.7+10.6)  & 14{,}814   \\
			KnOTS-TIES~\cite{stoica2025knots}      & 10.8 (7.2+3.6)                 & 16{,}150                         & 24.3 (13.7+10.6) & 14{,}814   \\ 
			\midrule
			AdaMerging ~\cite{yang2023adamerging}  & 8.0                            & 19{,}918                         & 13.9             & 19{,}788   \\
			\rowcolor[gray]{0.9} \method-Variant A & 8.2                            & 19{,}922                         & 14.1             & 19{,}788   \\
			\rowcolor[gray]{0.9} \method-Variant B & 15.2 (6.6+8.6)                 & 20{,}040                         & 27.7 (14.6+13.1) & 20{,}662   \\ 
			\bottomrule[1.2pt]
		\end{tabular}
	}
\end{table*}

\vspace{2pt}
\noindent\textbf{Per-Task Evaluation across Vision Tasks.} For the eight image-classification benchmarks, \cref{tab:8vis_knots} reports both the absolute per-task accuracies of the fine-tuned LoRA checkpoints and the normalized accuracies of the merged models. Strong vanilla merging methods such as Iso-CTS~\cite{marczak2025no} reach up to 73.5\% of the fine-tuned performance on average, but still lag behind the fine-tuned baselines and show large variance across datasets. Among LoRA-aware approaches, KnOTS-TIES~\cite{stoica2025knots} provides the strongest baseline, whereas RobustMerge~\cite{zeng2025parameter} yields noticeably weaker average performance and smaller gains, indicating that its robustness benefits are limited in this setting. On top of these baselines, \method-Variant A and \method-Variant B further improve the average normalized accuracy to 74.0\% and 76.3\%, respectively, with consistent gains, showing that our method better preserves per-task performance when merging multiple checkpoints than more recent merging baselines.

\vspace{2pt}
\noindent\textbf{Per-Task Evaluation on 6 NLI Tasks with LLMs.} For the six NLI benchmarks, \cref{tab:six_nli_benchmarks} shows a different picture from CLIP-based vision settings. Although Iso-C~\cite{marczak2025no} performs strongly on vision benchmarks, in the NLI setting it produces highly unbalanced results: it attains more than 100\% of the fine-tuned accuracy on SCITAIL, but suffers large drops on the other tasks, leading to the lowest average among recent vanilla baselines. EMR-Merging~\cite{huang2024emr} and FR-Merging~\cite{zheng2025free}, which are effective in conventional CLIP-based multi-task setups, also fail to transfer cleanly to LLMs and do not achieve competitive averages in this regime. LoRA-aware methods generally improve over vanilla merging, but RobustMerge offers only modest gains over simpler LoRA-aware baselines, with limited improvement in average normalized accuracy. In contrast, our \method-Variant A and \method-Variant B reach 79.7\% and 80.3\% average normalized accuracy, respectively, outperforming all vanilla and LoRA-aware baselines. These results indicate that methods tailored to CLIP-based models do not automatically carry over to LLMs, while our approach scales robustly from vision to language tasks.

\subsection{Evaluation on 6 VLM Tasks}
We further evaluate the scalability of our method on vision–language models (VLMs). Specifically, we use LLaVA~\cite{liu2023llava, liu2023improvedllava} on six benchmarks: ScienceQA~\cite{lu2022learn}, VizWiz~\cite{gurari2018vizwiz}, IconQA~\cite{lu2021iconqa}, ImageNet~\cite{ILSVRC15}, ChartQA~\cite{masry2022chartqa}, and DocVQA~\cite{mathew2021docvqa}. We fine-tune each LoRA adapter with rank 16 applied to the query, key, value and output projection modules of the attention layers, and train them for 500 iterations using AdamW~\cite{loshchilov2018decoupled} with a learning rate of \(3\times 10^{-5}\) and batch size \(1\). We compare against vanilla model merging baselines, Task Arithmetic (TA)~\cite{ilharco2022editing} and TIES~\cite{yadav2023ties}, as well as the LoRA-aware methods KnOTS-TIES~\cite{stoica2025knots} and AdaMerging~\cite{yang2023adamerging}. On LLaVA, AdaMerging underperforms LoRA-targeted merging methods such as KnOTS-TIES, indicating that its gains do not transfer well to the VLM regime. In contrast, both \method-Variant A and \method-Variant B achieve clear and stable improvements over all baselines across the six benchmarks, confirming that our approach also works reliably in multimodal VLM scenarios.

\subsection{Ablations and Time/Memory Analysis}
\Cref{tab:cost} summarizes the empirical efficiency of each method in terms of time cost and peak VRAM. We apply rank-16 LoRA and merge six LLaMA and six LLaVA models on a single NVIDIA GeForce RTX 3090. The upper panel reports gradient-free methods, and the lower panel reports gradient-based methods. A key observation is that, even when merging six fine-tuned LoRA checkpoints at foundation-model scale, the peak memory usage only increases by about 4–5 GB, so gradient-based merging does not incur prohibitive VRAM overhead.

For time, we report the cost of a single validation evaluation for gradient-free methods and the total optimization time for gradient-based methods. At first glance, gradient-free and gradient-based approaches have similar per-run time, but gradient-free methods usually require multiple validation runs to tune the scaling factor $\lambda$, which substantially increases their effective cost.
For example, tuning the Task Arithmetic scaling factor in $\mW_{\text{merge}}=\mW_{0}+\lambda\sum_{i=1}^{N}\Delta \mW_{i}$ with a grid from $0.3$ to $0.7$ in steps of $0.1$ entails five trials. If each evaluation takes $10.6$ minutes, the total time rises to about $53$ minutes.

Gradient-based methods do not require such grid search, so their total time can be more favorable.
\method has runtime comparable to AdaMerging but achieves higher accuracy.
Methods such as TIES, KnOTS-TIES, and \method-Variant B incur additional preprocessing to compute an SVD over all task vectors.
For these methods, we therefore report both their total time and, in parentheses, a breakdown into preprocessing time and either one validation run for the gradient-free setting or one optimization run for the gradient-based setting.
For six LLaMA-3-8B models, for instance, \method-Variant B spends $6.6$ of its $15.2$ minutes on SVD preprocessing and the remaining $8.6$ minutes on optimization, which is comparable to Variant~A and AdaMerging.
Even for large models such as LLaMA-3-8B and LLaVA-1.5-7B, a merge completes within 30 minutes, underscoring the practical applicability of \method.

To complement the analysis on foundation models, we also evaluate computational efficiency on relatively lightweight vision targets. \Cref{tab:efficiency} reports the time and memory costs using the ViT-B/32 backbone. Consistent with the trends observed in larger models, the runtimes for learning-free methods (e.g., KnOTS-TIES, LoRA-LEGO) represent only a single validation pass. Because practical deployment requires grid searches to find optimal scaling factors, their total operational time typically exceeds the single-run optimization time of \method. Regarding memory consumption, while a distinct gap in VRAM usage exists between learning-free and gradient-based methods on the smaller ViT-B/32 architecture, this discrepancy diminishes at the foundation-model scale (as seen with LLaMA-3). Since frozen base weights dominate the overall memory footprint of large models, the relative overhead introduced by gradient computation becomes less pronounced.

\begin{table}[h]
	\centering
	\caption{Computational Cost (ViT-B/32). Times denote a single validation pass (learning-free) vs. optimization (learning-based).}
	\vspace{-5pt}
	\label{tab:efficiency}
	\renewcommand{\arraystretch}{1.0}
	\resizebox{0.70\linewidth}{!}{
		\setlength{\tabcolsep}{3pt}
		\begin{tabular}{l|cc}
			\toprule
			\textbf{Method}  & \textbf{Time (min)} & \textbf{VRAM (MiB)} \\
			\midrule
			KnOTS-TIES~\cite{stoica2025knots}       & 2.8                 & 2262                \\
			LoRA-LEGO~\cite{zhao2025loralego}        & 4.0                 & 1416                \\
			\midrule
			AdaMerging~\cite{yang2023adamerging}       & 4.1                 & 4344                \\
			TARA (Variant A) & 4.1                 & 4344                \\
			TARA (Variant B) & 5.1                 & 5922                \\
			\bottomrule
		\end{tabular}}
\end{table}

\subsection{Additional Ablation Studies}

\paragraph{Ablation on Hyperparameter $\alpha$.}
\Cref{fig:stch_alpha_ablation} reports the effect of the STCH scaling hyperparameter $\alpha$ on normalized accuracy for both Variant~A and Variant~B.
Performance is slightly degraded at very small scaling ($\alpha = 0.1$), but quickly stabilizes once $\alpha$ reaches $0.5$.
For $\alpha \ge 0.5$, both variants show only minor fluctuations and maintain consistently strong performance across all benchmarks, indicating that our method is not sensitive to the exact choice of $\alpha$ in this range. Unless otherwise noted, we therefore fix $\alpha = 1$ in all main experiments.

\begin{figure}[ht]
	\centering
	\includegraphics[width=0.90\linewidth]{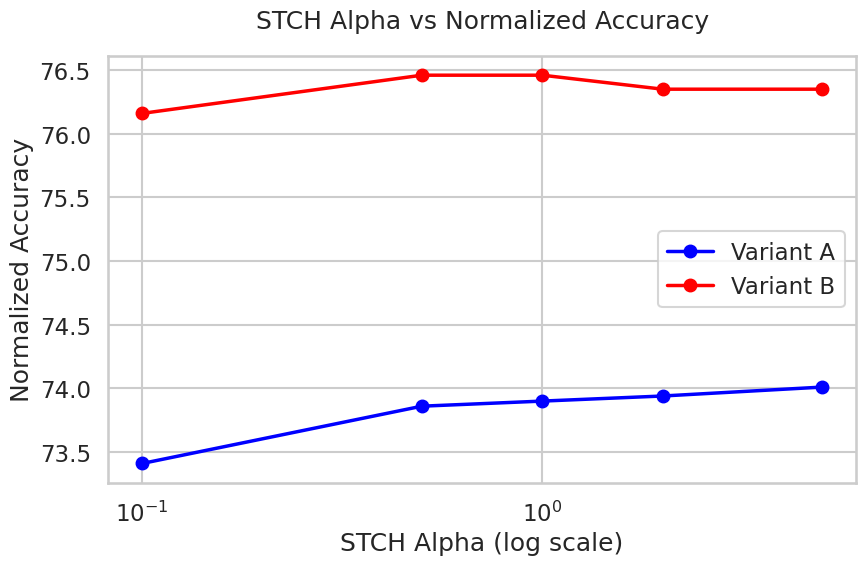}
	\caption{
		Ablation study of the hyperparameter $\alpha$.
		Normalized accuracy is reported for two model variants across
		$\alpha \in \{0.1, 0.5, 1.0, 2.0, 5.0\}$.
		Both Variant~A and Variant~B exhibit stable performance over a wide range of $\alpha$ values,
		indicating that the method is not sensitive to this hyperparameter.
		Unless otherwise stated, we use $\alpha = 1$ for all main experiments.
	}
	\label{fig:stch_alpha_ablation}
\end{figure}

\paragraph{Two-Task Pareto Fronts across Pairs.}
Figure~\ref{fig:pareto_front_full} plots accuracy trade-offs for four task pairs on CLIP ViT-B/32. Sweeping the preference with \method yields smooth Pareto fronts that are consistently dominant over AdaMerging, that is, \method attains higher accuracy for the same preference across most of the preference range, especially in the balanced region where both tasks must be retained. Baseline mergers appear as isolated points and frequently lie below or off our front. We attribute this dominance to \method’s LoRA-aware design: it preserves subspace coverage, retaining useful low-rank directions across tasks, while reweighting anisotropic directions to mitigate interference, keeping the merged model competitive near each task-optimal end without collapsing in the middle.

\begin{figure*}[t]
	\centering
	\includegraphics[width=\linewidth]{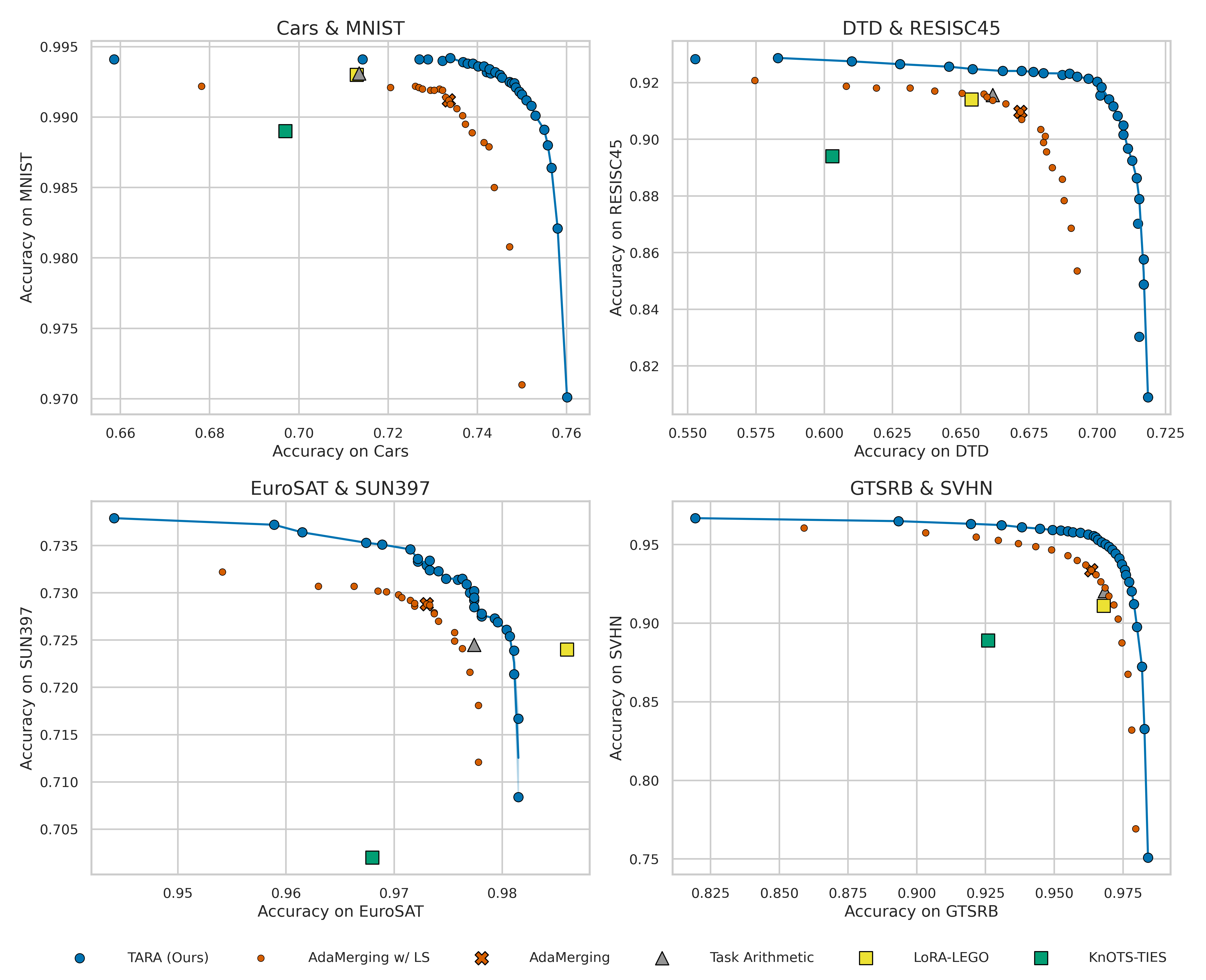}
	\caption{\textbf{Two-Task Trade-Offs on CLIP ViT-B/32.}
		Pairs shown from top-left to bottom-right: Cars-MNIST, DTD-RESISC45, EuroSAT-SUN397, GTSRB-SVHN.}
	\label{fig:pareto_front_full}
    \vspace{-5pt}
\end{figure*}

\paragraph{Robustness under Preference Perturbations.}
Figure~\ref{fig:preference_sensitivity_full} tests robustness with $8$ tasks by fixing the two focal-task weights to $0.125$ each (total $0.25$) and randomly choosing the other six weights so they are nonnegative and sum to $0.75$. Each scatter cloud shows $30$ such samplings and the ellipses summarize their empirical covariance, with the mean marked by “$\times$”. Across all pairs, \method produces tighter ellipses, indicating lower variance under preference perturbations, while AdaMerging shows larger spread. We also applied both a linear weighted-sum objective and a smooth Tchebycheff objective to \method and to AdaMerging; the choice of objective yielded nearly identical ellipses and means.

\begin{figure*}[t]
	\centering
	\includegraphics[width=\linewidth]{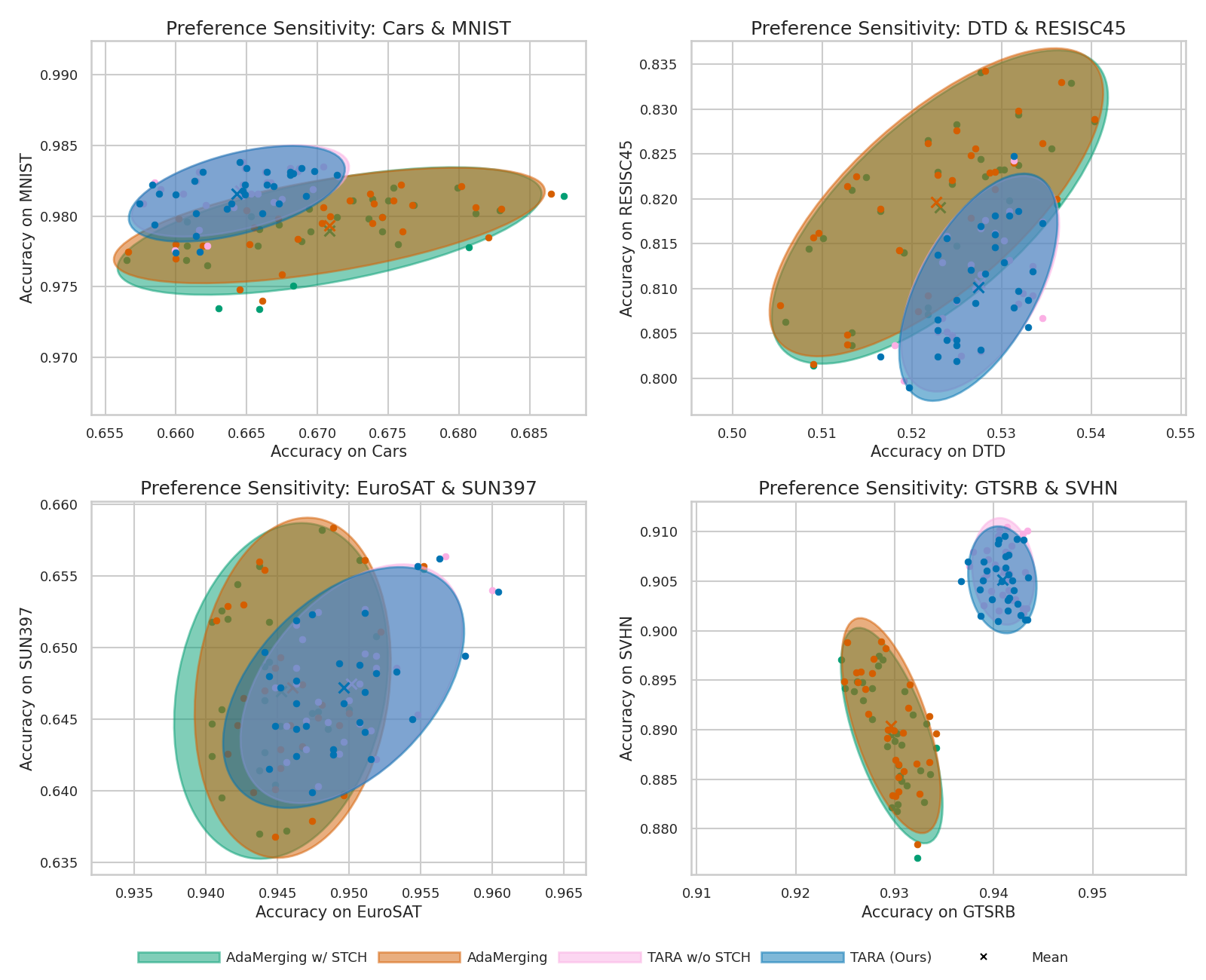}
	\caption{\textbf{Preference Sensitivity under Random Global Preferences.}
		Pairs shown from top-left to bottom-right: Cars-MNIST, DTD-RESISC45, EuroSAT-SUN397, GTSRB-SVHN.}
	\label{fig:preference_sensitivity_full}
\end{figure*}

\newcolumntype{x}[1]{>{\centering\arraybackslash}p{#1pt}}
\newcolumntype{y}[1]{>{\raggedright\arraybackslash}p{#1pt}}
\newcolumntype{z}[1]{>{\raggedleft\arraybackslash}p{#1pt}}

\begin{table*}
	\centering
	\caption{
		\textbf{Eight CLIP/ViT-B-32 Models Joint-Task Results. 
		}}
	\begin{adjustbox}{width=0.99\textwidth}
		\begin{tabular}{y{85}y{40}x{35}x{35}x{40}x{40}x{40}x{45}x{35}x{35}>{\columncolor[gray]{.9}}x{40}}
			\toprule[1.2pt]
			\multirow{2}{*}{{Method}} & \multirow{2}{*}{{Metric}} & \multicolumn{8}{c}{{Joint-Task Performances (\%)}}                                                                                  \\
			\cmidrule{3-11}
			                          &                           & {Cars}                                             & {DTD} & {EuroSAT} & {GTSRB} & {MNIST} & {RESISC45} & {SUN397} & {SVHN} & {Avg} \\
			\midrule
			\multirow{3}{*}{TA~\cite{ilharco2022editing} }
			                          & \texttt{Hits@1}           & 60.7                                               & 40.7  & 15.3      & 38.8    & 31.8    & 59.7       & 61.9     & 29.2   & 43.5  \\
			                          & \texttt{Hits@3}           & 84.9                                               & 63.7  & 23.0      & 66.1    & 48.4    & 83.6       & 83.9     & 50.1   & 65.2  \\
			                          & \texttt{Hits@5}           & 92.0                                               & 74.0  & 31.0      & 77.9    & 55.7    & 90.2       & 89.9     & 61.6   & 74.0  \\
			\midrule
			\multirow{3}{*}{TIES~\cite{yadav2023ties} }
			                          & \texttt{Hits@1}           & 60.4                                               & 39.7  & 13.0      & 35.0    & 33.4    & 58.6       & 61.3     & 32.9   & 43.6  \\
			                          & \texttt{Hits@3}           & 84.9                                               & 61.9  & 21.9      & 63.3    & 48.2    & 82.8       & 83.7     & 53.6   & 65.3  \\
			                          & \texttt{Hits@5}           & 92.1                                               & 72.4  & 29.0      & 75.3    & 54.0    & 89.4       & 89.9     & 64.1   & 73.9  \\
			\midrule
			\multirow{3}{*}{DARE-TIES~\cite{yu2024language} }
			                          & \texttt{Hits@1}           & 60.8                                               & 39.3  & 12.8      & 33.7    & 34.3    & 57.5       & 60.4     & 35.5   & 44.0  \\
			                          & \texttt{Hits@3}           & 85.3                                               & 61.4  & 18.1      & 63.6    & 50.2    & 82.3       & 82.6     & 57.8   & 66.4  \\
			                          & \texttt{Hits@5}           & 92.6                                               & 73.0  & 20.9      & 74.6    & 55.7    & 89.0       & 89.1     & 69.5   & 75.1  \\
			\midrule
			\multirow{3}{*}{\shortstack{AdaMerging~\cite{yang2023adamerging}                                                                                                                            }}
			                          & \texttt{Hits@1}           & 58.7                                               & 37.9  & 18.1      & 36.7    & 51.6    & 57.4       & 63.1     & 61.6   & 48.1  \\
			                          & \texttt{Hits@3}           & 83.1                                               & 59.5  & 56.9      & 63.6    & 71.0    & 81.3       & 84.8     & 85.1   & 73.2  \\
			                          & \texttt{Hits@5}           & 91.0                                               & 70.5  & 72.7      & 75.9    & 80.7    & 89.1       & 90.8     & 92.8   & 83.0  \\
			\midrule
			\multirow{3}{*}{KnOTS-TIES~\cite{stoica2025knots}}
			                          & \texttt{Hits@1}           & 61.7                                               & 40.5  & 16.2      & 44.2    & 39.1    & 59.0       & 60.6     & 36.7   & 46.8  \\
			                          & \texttt{Hits@3}           & 85.8                                               & 63.8  & 22.3      & 69.0    & 52.6    & 83.9       & 82.8     & 58.4   & 68.1  \\
			                          & \texttt{Hits@5}           & 92.6                                               & 74.5  & 31.1      & 79.8    & 58.4    & 90.4       & 89.1     & 68.5   & 76.3  \\
			\midrule
			\multirow{3}{*}{\shortstack{KnOTS                                                                                                                                                           \\-DARE-TIES~\cite{stoica2025knots}}}
			                          & \texttt{Hits@1}           & 60.4                                               & 40.3  & 15.9      & 41.9    & 34.6    & 58.4       & 60.4     & 34.8   & 45.2  \\
			                          & \texttt{Hits@3}           & 85.0                                               & 63.5  & 21.1      & 68.4    & 50.0    & 83.8       & 82.4     & 56.5   & 66.9  \\
			                          & \texttt{Hits@5}           & 92.2                                               & 74.5  & 26.6      & 78.9    & 55.9    & 90.4       & 88.9     & 67.4   & 75.3  \\
			\midrule
			\multirow{3}{*}{LoRA-LEGO~\cite{zhao2025loralego} }
			                          & \texttt{Hits@1}           & 60.3                                               & 39.9  & 15.9      & 41.6    & 29.8    & 57.6       & 60.1     & 30.2   & 43.1  \\
			                          & \texttt{Hits@3}           & 84.5                                               & 61.8  & 19.5      & 68.8    & 47.8    & 82.8       & 82.3     & 51.2   & 65.0  \\
			                          & \texttt{Hits@5}           & 91.8                                               & 74.0  & 21.1      & 78.3    & 56.9    & 89.7       & 88.8     & 63.5   & 73.9  \\
			\midrule
			\multirow{3}{*}{\shortstack{\method-Variant A                                                                                                                                               }}
			                          & \texttt{Hits@1}           & 60.8                                               & 39.5  & 15.2      & 39.9    & 65.2    & 59.0       & 63.2     & 65.7   & 51.1  \\
			                          & \texttt{Hits@3}           & 84.8                                               & 60.8  & 58.0      & 67.0    & 80.7    & 82.6       & 85.2     & 87.8   & 75.9  \\
			                          & \texttt{Hits@5}           & 92.2                                               & 71.8  & 74.9      & 78.2    & 86.4    & 89.9       & 91.0     & 95.0   & 84.9  \\
			\midrule
			\multirow{3}{*}{\shortstack{\method-Variant B                                                                                                                                               }}
			                          & \texttt{Hits@1}           & 63.2                                               & 40.1  & 12.1      & 41.9    & 64.6    & 62.1       & 63.6     & 46.4   & 49.3  \\
			                          & \texttt{Hits@3}           & 87.0                                               & 61.3  & 56.3      & 67.0    & 83.5    & 84.7       & 85.5     & 74.2   & 74.9  \\
			                          & \texttt{Hits@5}           & 93.4                                               & 72.5  & 76.9      & 77.2    & 88.9    & 90.9       & 91.3     & 89.9   & 85.1  \\
			\bottomrule[1.2pt]
		\end{tabular}
	\end{adjustbox}
	\label{tab:8vision_joint}
\end{table*}

\paragraph{Merging Models with LoRA Rank 16 on our Checkpoints.}

We complement the main-paper results (which use the KnOTS-released checkpoints~\cite{stoica2025knots}) with additional experiments on our own LoRA checkpoints to gauge robustness across sources.
Specifically, we fine-tune CLIP ViT-B/32 adapters with a learning rate different from that used by KnOTS, and then evaluate the same merging protocols. \Cref{tab:8vis_ours} reports per-task absolute accuracies for the fine-tuned adapters (upper panel) and, for each merger, taskwise accuracies normalized to the corresponding fine-tuned baseline (lower panel). Table~\ref{tab:general_ours} reports generalization to unseen tasks on our checkpoints. The results are consistent with the main paper, indicating that the advantages of preserving subspace coverage and addressing anisotropy hold across different fine-tuning learning rates. A comparison of checkpoints is provided in \cref{tab:pertask_eight_comp}.

\begin{table*}[ht]
	\vspace{-5pt}
	\caption{Using our checkpoints. Per‑task accuracy on eight image‑classification benchmarks.
		We merge eight ViT‑B/32 checkpoints, each finetuned with LoRA.
		The upper panel shows the per‑task absolute accuracy of the finetuned baselines; the lower panel reports accuracy of merged models, normalized by their corresponding finetuned baseline (\%).}
	\vspace{-5pt}
	\centering
	\renewcommand{\arraystretch}{0.85}
	\resizebox{0.85\linewidth}{!}{
		\begin{tabular}{lcccccccc>{\columncolor[gray]{0.9}}c}
			\toprule
			\multirow{2}{*}{Method}                & \multicolumn{9}{c}{Dataset}                                                                                                                                                                                                     \\
			\cmidrule(lr){2-10}
			                                       & Cars                                                                                            & DTD           & EuroSAT       & GTSRB         & MNIST         & RESISC45      & SUN397        & SVHN          & Avg           \\
			\midrule
			                                       & \multicolumn{9}{c}{\textit{Per‑task absolute accuracies (\%)}}                                                                                                                                                                  \\ \cmidrule(lr){2-10}
			Finetuned                              & 76.2                                                                                            & 72.5          & 98.6          & 98.3          & 99.2          & 93.1          & 73.7          & 96.7          & 88.5          \\
			\midrule
			                                       & \multicolumn{9}{c}{\textit{Per-task accuracies of merged models, normalized to finetuned (\%)}}                                                                                                                                 \\ \cmidrule(lr){2-10}
			\addlinespace[2pt]
			\textit{\textbf{Vanilla Merging}}      & \multicolumn{9}{l}{}                                                                                                                                                                                                            \\
			TA~\cite{ilharco2022editing}           & 85.6                                                                                            & 70.8          & 81.9          & 89.1          & 97.9          & 81.9          & 89.1          & 91.9          & 86.0          \\ 
			TIES~\cite{yadav2023ties}              & 72.2                                                                                            & 61.2          & 64.7          & 66.6          & 90.0          & 70.0          & 86.1          & 73.7          & 73.1          \\
			DARE-TIES~\cite{yu2024language}        & 73.0                                                                                            & 62.7          & 56.1          & 63.6          & 85.7          & 69.6          & 86.2          & 71.5          & 71.0          \\
			AdaMerging~\cite{yang2023adamerging}   & \textbf{87.7}                                                                                   & 71.2          & 96.3          & 94.4          & 98.7          & 87.3          & 86.8          & 91.9          & 89.3          \\ \hline
			\addlinespace[2pt]
			\textit{\textbf{LoRA‑aware Merging}}   & \multicolumn{9}{l}{}                                                                                                                                                                                                            \\
			SVD~\cite{tang2025lora}                & 76.0                                                                                            & 60.9          & 74.8          & 90.4          & 98.1          & 74.2          & 85.0          & 93.6          & 81.6          \\
			Linear~\cite{peft}                     & 74.9                                                                                            & 67.3          & 68.4          & 67.6          & 91.7          & 72.9          & 86.7          & 75.3          & 75.6          \\
			KnOTS-TIES~\cite{stoica2025knots}      & 85.4                                                                                            & 69.3          & 77.8          & 76.6          & 93.3          & 80.1          & 89.4          & 82.0          & 81.7          \\
			KnOTS-DARE-TIES~\cite{stoica2025knots} & 84.1                                                                                            & 68.7          & 77.9          & 78.6          & 94.4          & 80.1          & \textbf{89.8} & 83.8          & 82.2          \\
			LoRA-LEGO~\cite{zhao2025loralego}      & 84.9                                                                                            & 71.4          & 82.7          & 91.1          & 98.5          & 81.2          & 87.4          & 93.4          & 86.3          \\ \hline
			\addlinespace[2pt]
			\method-Variant A                      & 87.2                                                                                            & 72.5          & 96.5          & \textbf{95.7} & 99.0          & 87.0          & 88.0          & 93.7          & 89.9          \\
			\method-Variant B                      & 85.8                                                                                            & \textbf{72.0} & \textbf{96.8} & 96.4          & \textbf{99.1} & \textbf{88.0} & 87.4          & \textbf{95.2} & \textbf{90.1} \\
			\bottomrule
		\end{tabular}
	}
	\label{tab:8vis_ours}
	\vspace{-10pt}
\end{table*}

\begin{table*}[t]
	\caption{Comparison between merging methods between our checkpoints and KnOTS checkpoints.}
	\vspace{-5pt}
	\centering
	\renewcommand{\arraystretch}{0.80}
	\resizebox{0.80\linewidth}{!}{
		\begin{tabular}{lcccccccc>{\columncolor[gray]{0.9}}c}
			\toprule
			\multirow{2}{*}{Method}              & \multicolumn{9}{c}{Dataset}                                                                                                                                        \\
			\cmidrule(lr){2-10}
			                                     & Cars                                                                                            & DTD  & EuroSAT & GTSRB & MNIST & RESISC45 & SUN397 & SVHN & Avg  \\
			\midrule
			                                     & \multicolumn{9}{c}{\textit{Per‑task absolute accuracies (\%)}}                                                                                                     \\ \cmidrule(lr){2-10}
			Finetuned (KnOTS)                    & 74.0                                                                                            & 58.3 & 99.0    & 92.7  & 99.3  & 88.4     & 64.5   & 96.2 & 84.1 \\
			Finetuned (Ours)                     & 76.2                                                                                            & 72.5 & 98.6    & 98.3  & 99.2  & 93.1     & 73.7   & 96.7 & 88.5 \\
			\midrule
			                                     & \multicolumn{9}{c}{\textit{Per‑task accuracies of merged models, normalized to finetuned (\%)}}                                                                    \\ \cmidrule(lr){2-10}
			\addlinespace[2pt]
			TA~\cite{ilharco2022editing} (KnOTS) & 82.0                                                                                            & 73.6 & 48.8    & 42.1  & 53.1  & 71.5     & 97.5   & 41.2 & 63.7 \\
			TA~\cite{ilharco2022editing} (Ours)  & 83.1                                                                                            & 68.7 & 79.7    & 93.0  & 98.8  & 79.4     & 85.2   & 93.8 & 85.2 \\ \hline
			\addlinespace[2pt]
			\method-Variant A (KnOTS)            & 84.5                                                                                            & 76.2 & 68.9    & 39.4  & 82.2  & 72.8     & 97.5   & 70.0 & 73.9 \\
			\method-Variant A (Ours)             & 86.5                                                                                            & 72.0 & 94.4    & 95.0  & 98.8  & 85.3     & 87.4   & 93.5 & 89.1 \\ \hline
			\addlinespace[2pt]
			\method-Variant B (KnOTS)            & 86.2                                                                                            & 78.4 & 76.8    & 42.9  & 82.7  & 75.4     & 98.6   & 69.7 & 76.3 \\
			\method-Variant B (Ours)             & 86.1                                                                                            & 72.7 & 95.7    & 95.6  & 99.0  & 86.4     & 88.0   & 94.2 & 89.7 \\
			\bottomrule
		\end{tabular}
	}
	\label{tab:pertask_eight_comp}
\end{table*}

\begin{table*}[t]
	\centering
	\caption{
		Using our checkpoints.
		Generalization results on two unseen tasks when merging ViT-B/32 models trained on six tasks.
	}
	\vspace{-5pt}
	\renewcommand{\arraystretch}{0.90}
	\resizebox{0.90\linewidth}{!}{
		\begin{tabular}{l|ccccccc|ccc|c}
			\toprule
			                                          & \multicolumn{7}{c|}{{Seen Tasks}}                                                                & \multicolumn{3}{c|}{{Unseen Tasks}} & \multicolumn{1}{c}{{All Tasks}}                                                                                                                                          \\
			\midrule \midrule
			{Method}                                  & {Cars}                                                                                           & {DTD}                               & {GTSRB}                         & {RESISC45}    & {SUN397}      & {SVHN}        & \textbf{Avg Acc} & {EuroSAT}     & {MNIST}       & \textbf{Avg Acc} & \textbf{Avg Acc} \\
			\midrule
			                                          & \multicolumn{11}{c}{\textit{Per-task accuracies of merged models, normalized to finetuned (\%)}}                                                                                                                                                                                                                  \\ \cmidrule(lr){2-12}
			{TA}~\cite{ilharco2022editing}            & 85.6                                                                                             & 73.0                                & 94.7                            & 82.4          & 88.4          & 95.0          & 86.5             & 47.3          & 78.1          & 62.7             & 80.6             \\
			{TIES}~\cite{yadav2023ties}               & 76.5                                                                                             & 62.5                                & 71.4                            & 73.3          & 87.5          & 75.0          & 74.4             & 49.0          & 64.5          & 56.8             & 69.9             \\
			{KnOTS-TIES}~\cite{stoica2025knots}       & 85.0                                                                                             & 69.4                                & 83.9                            & 81.1          & 89.4          & 84.8          & 82.3             & 56.6          & 73.4          & 65.0             & 78.0             \\
			{KnOTS-DARE-TIES}~\cite{stoica2025knots}  & 85.7                                                                                             & 69.3                                & 85.0                            & 81.4          & \textbf{90.3} & 86.4          & 83.0             & 56.7          & 74.0          & 65.4             & 78.6             \\
			{LoRA-LEGO}~\cite{zhao2025loralego}       & 85.2                                                                                             & 72.9                                & 94.9                            & 82.3          & 88.3          & 94.8          & 86.4             & 46.9          & 77.8          & 62.4             & 80.4             \\
			{AdaMerging~\cite{yang2023adamerging}  }  & 87.1                                                                                             & 73.8                                & 93.4                            & \textbf{91.2} & 88.4          & 96.2          & 88.4             & 52.2          & 83.1          & 67.7             & 83.2             \\
			\rowcolor[gray]{0.9} {\method-Variant A } & \textbf{87.4}                                                                                    & \textbf{74.9}                       & 94.9                            & 89.6          & 89.2          & \textbf{96.7} & \textbf{88.8}    & 57.4          & 86.6          & 72.0             & 84.6             \\
			\rowcolor[gray]{0.9} {\method-Variant B } & 87.2                                                                                             & 74.7                                & \textbf{95.4}                   & 89.4          & 89.3          & \textbf{96.7} & \textbf{88.8}    & \textbf{63.4} & \textbf{87.7} & \textbf{75.5}    & \textbf{85.5}    \\

			\midrule
			\midrule
			{Method}                                  & {Cars}                                                                                           & {DTD}                               & {EuroSAT}                       & {GTSRB}       & {MNIST}       & {SUN397}      & \textbf{Avg Acc} & {RESISC45}    & {SVHN}        & \textbf{Avg Acc} & \textbf{Avg Acc} \\
			\midrule
			                                          & \multicolumn{11}{c}{\textit{Per-task accuracies of merged models, normalized to finetuned (\%)}}                                                                                                                                                                                                                  \\ \cmidrule(lr){2-12}
			{TA}~\cite{ilharco2022editing}            & 87.2                                                                                             & 75.6                                & 86.4                            & 95.6          & 98.4          & 90.6          & 89.0             & 60.4          & 61.6          & 61.0             & 82.0             \\
			{TIES}~\cite{yadav2023ties}               & 74.7                                                                                             & 62.8                                & 65.3                            & 74.6          & 88.3          & 88.0          & 75.6             & 65.1          & 48.9          & 57.0             & 71.0             \\
			{KnOTS-TIES}~\cite{stoica2025knots}       & 86.7                                                                                             & 70.8                                & 80.8                            & 83.3          & 94.1          & 90.4          & 84.3             & \textbf{67.5} & 57.2          & 62.4             & 78.8             \\
			{KnOTS-DARE-TIES}~\cite{stoica2025knots}  & 87.6                                                                                             & 71.0                                & 80.4                            & 85.3          & 95.1          & 90.7          & 85.0             & 66.7          & 56.2          & 61.4             & 79.1             \\
			{LoRA-LEGO}~\cite{zhao2025loralego}       & 87.0                                                                                             & 75.3                                & 86.0                            & 95.9          & 98.5          & 90.2          & 88.8             & 60.0          & 61.5          & 60.8             & 81.8             \\
			{AdaMerging~\cite{yang2023adamerging}  }  & 89.7                                                                                             & 74.7                                & 96.7                            & 95.0          & \textbf{99.4} & 90.4          & 91.0             & 60.1          & \textbf{63.7} & 61.9             & 83.7             \\
			\rowcolor[gray]{0.9} {\method-Variant A } & \textbf{89.9}                                                                                    & 75.0                                & 97.1                            & 96.2          & 99.3          & 90.9          & \textbf{91.4}    & 61.2          & 63.5          & \textbf{62.4}    & \textbf{84.1}    \\
			\rowcolor[gray]{0.9} {\method-Variant B } & 89.7                                                                                             & 74.9                                & \textbf{97.4}                   & \textbf{96.4} & 99.2          & \textbf{91.0} & \textbf{91.4}    & 60.7          & 58.6          & 59.7             & 83.5             \\
			\bottomrule
		\end{tabular}
	}
	\label{tab:general_ours}
\end{table*}

\begin{table*}[t]
	\caption{Using our checkpoints with LoRA (Rank-4). Per-task accuracy on eight image-classification benchmarks.
		We merge eight ViT-B/32 checkpoints, each finetuned with LoRA.
		The upper panel shows the per-task absolute accuracy of the finetuned baselines; the lower panel reports accuracy of merged models, normalized by their corresponding finetuned baseline (\%).}
	\vspace{-5pt}
	\centering
	\renewcommand{\arraystretch}{0.85}
	\resizebox{0.85\linewidth}{!}{
		\begin{tabular}{lcccccccc>{\columncolor[gray]{0.9}}c}
			\toprule
			\multirow{2}{*}{Method}                & \multicolumn{9}{c}{Dataset}                                                                                                                                                                                                     \\
			\cmidrule(lr){2-10}
			                                       & Cars                                                                                            & DTD           & EuroSAT       & GTSRB         & MNIST         & RESISC45      & SUN397        & SVHN          & Avg           \\
			\midrule
			                                       & \multicolumn{9}{c}{\textit{Per-task absolute accuracies (\%)}}                                                                                                                                                                  \\ \cmidrule(lr){2-10}
			Finetuned                              & 72.1                                                                                            & 70.2          & 98.4          & 98.1          & 99.2          & 93.3          & 73.4          & 96.7          & 87.7          \\
			\midrule
			                                       & \multicolumn{9}{c}{\textit{Per-task accuracies of merged models, normalized to finetuned (\%)}}                                                                                                                                 \\ \cmidrule(lr){2-10}
			\addlinespace[2pt]
			\textit{\textbf{Vanilla Merging}}      & \multicolumn{9}{l}{}                                                                                                                                                                                                            \\
			TA~\cite{ilharco2022editing}           & 87.4                                                                                            & 76.2          & 74.0          & 78.1          & 98.1          & 83.7          & 89.7          & 92.3          & 84.9          \\
			TIE~\cite{yadav2023ties}               & 81.3                                                                                            & 64.4          & 46.5          & 51.9          & 77.4          & 69.1          & 86.8          & 64.1          & 67.7          \\
			DARE-TIES~\cite{yu2024language}        & 81.2                                                                                            & 65.4          & 50.9          & 48.4          & 76.3          & 69.8          & 86.7          & 63.5          & 67.8          \\
			AdaMerging~\cite{yang2023adamerging}   & 87.9                                                                                            & 75.1          & 94.7          & 92.0          & 98.2          & 86.9          & 89.5          & 91.1          & 89.4          \\ \hline
			\addlinespace[2pt]
			\textit{\textbf{LoRA-aware Merging}}   & \multicolumn{9}{l}{}                                                                                                                                                                                                            \\
			SVD~\cite{tang2025lora}                & 80.5                                                                                            & 62.7          & 40.6          & 73.7          & 94.4          & 69.3          & 87.3          & 92.4          & 75.1          \\
			Linear~\cite{peft}                     & 44.0                                                                                            & 25.6          & 12.4          & 37.8          & 71.5          & 23.3          & 44.2          & 51.7          & 38.8          \\
			KnOTS-TIES~\cite{stoica2025knots}      & 88.2                                                                                            & 72.2          & 72.5          & 62.3          & 89.7          & 79.3          & 89.5          & 77.0          & 78.8          \\
			KnOTS-DARE-TIES~\cite{stoica2025knots} & 87.7                                                                                            & 73.0          & 73.4          & 65.0          & 90.4          & 79.9          & 89.5          & 78.8          & 79.7          \\
			LoRA-LEGO~\cite{zhao2025loralego}      & 88.2                                                                                            & 75.8          & 74.8          & 76.9          & 97.5          & 83.8          & \textbf{90.4} & 91.5          & 84.8          \\ \hline
			\method-Variant A                      & 88.7                                                                                            & 76.4          & \textbf{94.9} & 92.1          & 98.3          & 86.8          & 90.0          & 92.0          & 89.9          \\
			\method-Variant B                      & \textbf{89.8}                                                                                   & \textbf{76.7} & 94.6          & \textbf{92.8} & \textbf{98.4} & \textbf{88.0} & 90.1          & \textbf{93.2} & \textbf{90.5} \\
			\bottomrule
		\end{tabular}
	}
	\label{tab:8vis_rank4}
\end{table*}


\paragraph{Merging Models with LoRA Rank 4.}
LoRA is known to exhibit weaker cross-task alignment than full-rank fine-tuning, so LoRA-aware merging becomes especially important at low ranks. Using our CLIP ViT-B/32 checkpoints fine-tuned with rank-4 adapters, we evaluate vanilla and LoRA-aware merging in \Cref{tab:8vis_rank4}. Vanilla methods (e.g., TIES) degrade substantially, LoRA-aware baselines (KnOTS-TIES, LoRA-LEGO) recover more accuracy, and \method\ achieves the strongest results overall: Variant B (500 iters) attains the best average normalized accuracy at \textbf{90.5}\% and leads on most datasets. These rank-4 outcomes mirror the trends at higher ranks and underscore the benefit of modeling subspace coverage and anisotropy in low-rank LoRA merging.

\clearpage
\clearpage

\paragraph{Ablation Study on LoRA Rank.}
We evaluate the impact of varying LoRA ranks $r$ on the merging performance using the ViT-B/32 backbone. \Cref{tab:rank_ablation} presents the absolute and normalized accuracy across different rank configurations. TARA achieves higher accuracy than the baselines across all evaluated ranks. Notably, in the `Full' setting, where the LoRA rank matches the ViT embedding dimension to emulate full-parameter fine-tuning, TARA does not simply converge to the FFT merging solution, rather, it maintains a distinct performance margin.

\begin{table}[h!]
	\centering
    \vspace{-5pt}
	\caption{Absolute (Normalized) Accuracy across LoRA Ranks.}
	\label{tab:rank_ablation}
	\vspace{-5pt}
	\setlength{\tabcolsep}{4pt}
	\renewcommand{\arraystretch}{1.0}
	\resizebox{0.99\linewidth}{!}{
		\begin{tabular}{l|ccccc}
			\toprule
			\textbf{Method} & \textbf{r=4} & \textbf{r=16} & \textbf{r=64} & \textbf{r=256} & \textbf{Full} \\
			\midrule
			TA~\cite{ilharco2022editing}              & 71.9(82.9)   & 76.2(86.2)    & 77.9(86.2)    & 77.1(85.3)     & 72.0(80.8)    \\
			Adamerging~\cite{yang2023adamerging}      & 76.3(87.2)   & 78.7(88.5)    & 80.2(88.7)    & 79.8(88.6)     & 79.1(88.2)    \\
			TARA            & 77.5(88.8)   & 80.5(90.6)    & 81.9(90.5)    & 80.9(89.7)     & 80.5(89.5)    \\
			\bottomrule
		\end{tabular}}
    \vspace{-5pt}
\end{table}

\paragraph{Efficiency Analysis with Recent Baselines} 
To assess computational efficiency, we compare TARA against recent gradient-based merging methods, CALM~\cite{yan2025calm} and FW-Merging~\cite{chen2025fw}, using the CLIP-ViT-B/32 backbone. \Cref{tab:baselines} reports the average accuracy, runtime, and peak memory (VRAM) usage. The results indicate that TARA requires significantly less runtime and memory overhead while maintaining competitive accuracy compared to other gradient-based approaches.

\begin{table}[h!]
	\centering
	\caption{Comparison with Recent Baselines.}
	\vspace{-5pt}
	\label{tab:baselines}
	\renewcommand{\arraystretch}{1.0}
	\resizebox{0.99\linewidth}{!}{
		\setlength{\tabcolsep}{3pt}
		\begin{tabular}{c|cccc}
			\toprule
			\textbf{Method} & \textbf{Adamerging}~\cite{yang2023adamerging} & \textbf{CALM}~\cite{yan2025calm} & \textbf{FW-Merging}~\cite{chen2025fw} & \textbf{TARA-variant B} \\
			\midrule
			Avg. Acc.       & 78.7                & 80.0          & 77.9                & 80.5                    \\
			Time (min)      & 4.1                 & 20.0          & 13.4                & 5.1                     \\
			VRAM (MiB)      & 4344                & 8066          & 7542                & 5922                    \\
			\bottomrule
		\end{tabular}}
    \vspace{-5pt}
\end{table}

\paragraph{Extended Evaluation on Vision Benchmarks.}
To verify the scalability of our approach, we extend the evaluation to different model scales (ViT-B/32 and ViT-L/14) and varying numbers of tasks (8, 14, and 20) following settings similar to those in FW-Merging~\cite{chen2025fw}. As detailed in \cref{tab:vision_bench}, TARA consistently demonstrates comparable or superior absolute and normalized accuracy across these configurations relative to recent baselines. This performance gap is particularly evident in large-scale settings, such as merging 20 tasks with the ViT-L/14 backbone.

\begin{table}[h!]
	\centering
	\caption{Absolute (Normalized) accuracy on vision benchmark.}
	\label{tab:vision_bench}
	\vspace{-5pt}
	\renewcommand{\arraystretch}{1.0}
	\resizebox{0.99\linewidth}{!}{
		\setlength{\tabcolsep}{3pt}
		\begin{tabular}{l|cccc|c}
			\toprule
			\textbf{Backbone (\# Tasks)} & \textbf{TA}~\cite{ilharco2022editing} & \textbf{AdaMer.}~\cite{yang2023adamerging} & \textbf{CALM}~\cite{yan2025calm}  & \textbf{FW-Mer.}~\cite{chen2025fw} & \textbf{TARA} \\
			\midrule
			ViT-B/32 (8)                 & 76.2(86.2)  & 78.7(88.5)       & 80.0(90.6)    & 77.9(87.6)       & 80.5(90.6)    \\
			ViT-B/32 (14)                & 73.6(83.6)  & 75.0(84.8)       & 74.7(84.6)    & 75.4(85.2)       & 75.4(85.2)    \\
			ViT-B/32 (20)                & 64.6(72.7)  & 66.6(74.8)       & 66.9(75.2)    & 68.2(77.2)       & 68.7(77.2)    \\
			\midrule
			ViT-L/14 (8)                 & 88.0(95.8)  & 88.9(96.7)       & 83.0(90.3)    & 87.0(94.7)       & 88.9(96.7)    \\
			ViT-L/14 (14)                & 85.0(92.7)  & 85.2(93.0)       & 80.5(87.7)    & 85.7(93.6)       & 85.9(93.7)    \\
			ViT-L/14 (20)                & 80.7(87.2)  & 81.3(87.8)       & 78.5(84.5)    & 77.4(84.1)       & 81.5(88.0)    \\
			\bottomrule
		\end{tabular}
	}
    \vspace{-5pt}
\end{table}

\end{document}